\documentclass{article}
\usepackage[utf8]{inputenc}
\usepackage{fullpage}
\usepackage[utf8]{inputenc}
\usepackage{latexsym,amsfonts,amssymb,amsthm,amsmath}
\usepackage{color}
\usepackage{thm-restate}
\usepackage{tikz}
\usetikzlibrary{calc}
\usetikzlibrary{decorations.pathreplacing}
\usepackage{tikz}
\usepackage[backend=biber,maxbibnames=9,style=alphabetic]{biblatex}
\usepackage{thmtools}
\usepackage{algorithm2e}
\PassOptionsToPackage{algo2e,ruled}{algorithm2e}
\usepackage{booktabs}
\usepackage[colorlinks,linkcolor=blue,citecolor=blue,urlcolor=blue]{hyperref}
\usepackage[noabbrev,capitalize]{cleveref}

\usepackage[mathscr]{eucal}

\theoremstyle{plain}
\newtheorem{theorem}{Theorem}[section]
\newtheorem{lemma}[theorem]{Lemma}

\newtheorem{corollary}[theorem]{Corollary}
\newtheorem{proposition}[theorem]{Proposition}
\theoremstyle{definition}
\newtheorem{definition}[theorem]{Definition}
\newtheorem{remark}[theorem]{Remark}

\newcommand{\ftime}{\text{F-TiME}}

\newcommand{\soul}{\text{SOUL}}

\newcommand{\crf}{\text{CRF}}
\newcommand{\cs}{\text{CS}}

\newcommand{\solar}{\text{SOLAR}}
\newcommand{\solaru}{\text{SOLAR-U}}

\newcommand{\fs}{\text{FS}}

\newcommand{\smv}{{\text{SMV}}}

\newcommand{\Acal}{\mathcal{A}}
\newcommand{\Bcal}{\mathcal{B}}
\newcommand{\Ccal}{\mathcal{C}}
\newcommand{\Dcal}{\mathcal{D}}
\newcommand{\Ecal}{\mathcal{E}}
\newcommand{\Fcal}{\mathcal{F}}
\newcommand{\Gcal}{\mathcal{G}}
\newcommand{\Hcal}{\mathcal{H}}

\newcommand{\Lcal}{\mathcal{L}}
\newcommand{\Mcal}{\mathcal{M}}
\newcommand{\Ncal}{\mathcal{N}}

\newcommand{\Rcal}{\mathcal{R}}

\newcommand{\Tcal}{\mathcal{T}}
\newcommand{\Ucal}{\mathcal{U}}

\newcommand{\Xcal}{\mathcal{X}}
\newcommand{\Ycal}{\mathcal{Y}}
\newcommand{\Ocal}{\mathcal{O}}
\newcommand{\Qcal}{\mathcal{Q}}

\newcommand{\Ebb}{\mathbb{E}}
\newcommand{\Nbb}{\mathbb{N}}
\newcommand{\Pbb}{\mathbb{P}}

\newcommand{\Rbb}{\mathbb{R}}

\newcommand{\Xbb}{\mathbb{X}}
\newcommand{\Ybb}{\mathbb{Y}}

\usepackage{bbm}

\newcommand{\1}{\mathbbm{1}}

\newcommand{\comment}[1]{}

\newcommand{\mb}[1]{\ensuremath{\boldsymbol{#1}}}

\title{Universal Regression with Adversarial Responses \footnote{Accepted, Annals of Statistics, June 2023}}
\author{
  Mo\"ise Blanchard\\
  MIT\\
  \small{\texttt{moiseb@mit.edu}}
  \and 
  Patrick Jaillet\\
  MIT\\
  \small{\texttt{jaillet@mit.edu}}
}

\date{}

\begin{document}

\maketitle

\begin{abstract}
    We provide algorithms for regression with adversarial responses under large classes of non-i.i.d. instance sequences, on general separable metric spaces, with \emph{provably minimal} assumptions. We also give characterizations of learnability in this regression context. We consider \emph{universal consistency} which asks for strong consistency of a learner without restrictions on the value responses. Our analysis shows that such an objective is achievable for a significantly larger class of instance sequences than stationary processes, and unveils a fundamental dichotomy between value spaces: whether finite-horizon mean estimation is achievable or not. We further provide \emph{optimistically universal} learning rules, i.e., such that if they fail to achieve universal consistency, any other algorithms will fail as well. For unbounded losses, we propose a mild integrability condition under which there exist algorithms for adversarial regression under large classes of non-i.i.d. instance sequences. In addition, our analysis also provides a learning rule for mean estimation in general metric spaces that is consistent under adversarial responses without any moment conditions on the sequence, a result of independent interest.
\end{abstract}


\section{Introduction}

\subsection{Motivation and background} We study the classical statistical problem of metric-valued regression. Given an instance metric space $(\Xcal,\rho_\Xcal)$ and a value metric space $(\Ycal,\rho_\Ycal)$ with a loss $\ell$, one observes instances in $\Xcal$ and aims to predict the corresponding values in $\Ycal$. The learning procedure follows an iterative process where successively, the learner is given an instance $X_t$ and predicts the value $Y_t$ based on the historical samples and the new instance. The learner's goal is to minimize the loss of its predictions $\hat Y_t$ compared to the true value $Y_t$. In particular, $\Ycal=\{0,1\}$ (resp. $\Ycal=\{0,\ldots,k\}$) with 0-1 loss corresponds to binary (resp. multiclass) classification while $\Ycal=\Rbb$ corresponds to the classical regression setting. Motivated by the increase of new types of data in numerous data analysis applications--- e.g., data lying on spherical spaces \cite{chang1989spherical,mardia2000directional}, manifolds \cite{shi2009intrinsic,davis2010population,thomas2013geodesic}, Hilbert spaces \cite{zaichyk2019efficient}, Hadamard spaces \cite{lin2021total}---we will study the case where both instances and value spaces are general separable metric spaces. This general setting adopted in the recent literature on universal learning \cite{hanneke2021open,tsir2022metric,blanchard2022universal} includes and extends the specific classification and regression settings mentioned above. In this context, we model the stream of data as a general stochastic process $(\Xbb,\Ybb):=(X_t,Y_t)_{t\geq 1}$, and are interested in \emph{consistent} predictions that have vanishing average \emph{excess} loss compared to any fixed measurable predictor functions $f:\Xcal\to\Ycal$, i.e., $\frac{1}{T}\sum_{t=1}^T \ell(\hat Y_t,Y_t)-\ell(f(X_t),Y_t)\to 0\;(a.s.)$. Naturally, one would hope that the algorithm converges for a large class of value functions. Thus, we are interested in \emph{universally consistent} learning rules that are consistent irrespective of the value process $\Ybb$.

The i.i.d. version of this problem where one assumes that the sequence $(\Xbb,\Ybb)$ is i.i.d. has been extensively studied. A classical result is that for binary classification in Euclidean spaces, $k-$nearest neighbor (kNN) with $k/\ln T\to\infty$ and $k/T\to 0$ is universally consistent under mild assumptions on the distribution of $(X_1,Y_1)$ \cite{stone1977consistent,devroye1994strong,devroye2013probabilistic}. These results were then extended to a broader class of spaces \cite{devroye2013probabilistic,gyorfi:02} and more recently, \cite{hanneke2021bayes,gyorfi2021universal,tsir2022metric} provided universally consistent algorithms for any essentially separable metric space $\Xcal$ which are precisely those for which universal consistency is achievable for i.i.d. pairs $(X_t,Y_t)_{t\geq 1}$ of instances and responses. In parallel, a significant line of work aimed to obtain such results in non-i.i.d. settings, notably relaxations of the i.i.d. assumptions such as stationary ergodic processes \cite{morvai1996nonparametric,gyorfi1999simple,gyorfi:02} or processes satisfying the law of large numbers \cite{morvai1999regression,gray2009probability,steinwart2009learning}.

\subsection{Optimistic universal learning} In this work, we aim to understand which are the minimal assumptions on the data sequences for which universal consistency is still achievable. As such, we follow the \emph{optimistic decision theory} \cite{hanneke2021learning} which formalizes the paradigm of ``learning whenever learning is possible". Precisely, the \emph{provably minimal} assumption for a given objective is that this task is achievable, or in other words that learning is possible. The goal then becomes to 1. characterize for which settings this objective is achievable and 2. if possible, provide learning rules that achieve this objective whenever it is achievable. These are called \emph{optimistically universal} learning rules and enjoy the convenient property that if they failed the objective, any other algorithms would fail as well.

 \subsection{Related works in universal learning} This paradigm was recently used to study minimal assumptions for the noiseless (realizable) case where there exists an unknown underlying function $f^*:\Xcal\to\Ycal$ such that $Y_t=f^*(X_t)$ \cite{hanneke2021learning}. In this setting, the two questions described above were recently settled. For bounded losses, a simple variant of the nearest neighbor algorithm is optimistically universal \cite{blanchard2021universal,blanchard2022universal} and learnable processes are significantly larger than stationary processes. On the other hand, for unbounded losses, universal regression is extremely restrictive since the only learnable processes are those which visit a finite number of points almost surely \cite{blanchard2022optimistic}. Yet, the general non-realizable setting was not characterized. As an initial result, for bounded losses, \cite{hanneke2022bayes} proposed an algorithm that achieves universal consistency for a large class of processes $\Xbb$, which intuitively asks that the sub-measure induced by empirical visits of the input sequence be continuous. There is however a significant gap between the proposed condition and the learnable processes in the bounded noiseless setting. \cite{hanneke2022bayes} then left open the question of identifying the precise provably-minimal conditions to achieve consistency, and whether there exists an optimistically universal learning rule.

 \subsection{Adversarial responses and related works in learning with experts} The consistency results in \cite{hanneke2022bayes} hold for arbitrary value processes $\Ybb$, arbitrarily correlated to the instance process $\Xbb$. We consider the slightly more general \emph{adversarial} responses and show that we can obtain the same results as for adversarial processes, without any generalizability cost. Formally, adversarial responses can not only arbitrarily depend on the instance sequence $\Xbb$, but may also depend on past predictions and past randomness used by the learner. This is a non-trivial generalization for randomized algorithms---note that randomization is necessary to obtain guarantees for general online learning problems \cite{bubeck2012regret,slivkins2019introduction}. There is a rich theory for arbitrary or adversarial responses $\Ycal$ when the reference functions $f^*:\Xcal\to\Ycal$ are restricted to specific function classes $\Fcal$. As a classical example, for the noiseless binary classification setting, there exist learning rules which guarantee a finite number of mistakes for arbitrary sequences $\Xbb$, if and only if the class $\Fcal$ has finite Littlestone dimension \cite{littlestone1988learning}. Other restrictions on the function class have been considered \cite{cesa2006prediction,ben2009agnostic,rakhlin2015online}. Universal learning diverges from this line of work by imposing no restrictions on function classes---namely \emph{all} measurable functions---but instead restricting instance processes $\Xbb$ to the optimistic set where universal consistency is achievable. Nevertheless, the algorithms we introduce for adversarial responses use as subroutine the traditional exponentially weighted forecaster for learning with expert advice from the online learning literature, also known as the Hedge algorithm \cite{littlestone1994weighted,cesa1997use,freund1997decision}.

\subsection{Contributions}

In this paper, we provide answers to two fundamental questions in universal regression. First, we exactly characterize the set of processes we call \emph{learnable}. These are instance processes $\Xbb$ for which universal learning is possible, i.e., consistency is achieved for every process $(X_t,Y_t)_{t\geq 1}$ with covariate sequence $\Xbb$. Second, we provide optimistically universal learning rules, i.e., a unique algorithm that achieves universal consistency for all processes $\Xbb$ for which this is achievable by some learning rule. The specific answers to these questions depend on the value space and loss $(\Ycal,\ell)$ as detailed below.

\subsubsection{Universal learning with empirically integrable responses} We introduce a mild moment-type assumption on the responses $\Ybb$, namely \emph{empirical integrability}, that roughly asks that one can bound the tails of the empirical first moment of $\Ybb$. We then proceed to analyze the processes for which learning adversarial responses guaranteed to satisfy this assumption, is achievable. The answer depends on a property of the value space and loss $(\Ycal,\ell)$ which we denote $\ftime$.
\begin{itemize}
    \item If every ball $B_\ell(y,r)$ of $(\Ycal,\ell)$ satisfies the $\ftime$ property, the class of processes $\Xbb$ for which universal consistency under adversarial empirically integrable responses may be achieved is the so-called Sublinear Measurable Visits ($\smv$) class. This coincides with the class of processes that admits universal learning for bounded losses in the realizable setting (noiseless responses) \cite{blanchard2022universal}. In particular, this shows that for value spaces with bounded losses satisfying $\ftime$, one can extend consistency results from the realizable setting to the adversarial one at no generalizability cost.
    \item Otherwise, the classes of processes $\Xbb$ for which one can achieve universal consistency for empirically integrable responses is a smaller class called Continuous Submeasure ($\cs$). This is a condition that was already considered by \cite{hanneke2022bayes}, which showed that for bounded metric losses, one can achieve universal learning under $\cs$ processes. Our results show that whenever the $\ftime$ condition is not satisfied for bounded losses, $\cs$ is also a necessary condition for universal learning.
\end{itemize}
Also, in both cases, we give an optimistically universal learning rule, that is implicit for the first case---it uses as subroutine the learning rule for mean-estimation---and explicit for the second. These results resolve an open question from \cite{hanneke2022bayes}.

 Intuitively, the property $\ftime$ asks that, for any fixed tolerance $\epsilon>0$, there is a learning rule that solves the analogous prediction problem without covariates $\Xbb$---\emph{mean-estimation}---in finite time within the tolerance $\epsilon$. This property is satisfied for ``reasonable'' value spaces, e.g., totally-bounded spaces or countably-many-classes classification $(\Nbb,\ell_{01})$, but we also provide an explicit example of bounded metric space that does not satisfy this condition.

To motivate the introduction of the empirical integrability condition we show that a weaker moment-type assumption on responses---that $\limsup_{T\to\infty} \frac{1}{T}\sum_{t=1}^T \ell(y_0,Y_t)<\infty\; (a.s.)$ for some $y_0\in\Ycal$---is not sufficient to extend the results from the bounded loss case to unbounded losses, resolving an open question from \cite{blanchard2022optimistic}. Further, empirical integrability is essentially necessary to obtain consistency results: it is automatically satisfied if the loss is bounded and for the i.i.d. setting it exactly asks that responses $Y$ have finite first moment.

As a direct implication of this work, finite second moment $\Ebb[Y^2]$ is sufficient to achieve consistency for stationary ergodic processes. This result relaxes the conditions of all past works to the best of our knowledge, which required finite fourth moment $\Ebb[Y^4]$ \cite{gyorfi:07}.

\subsubsection{Universal learning with unrestricted responses} For completeness, we also characterize the set of learnable processes without assuming empirical integrability on responses. Since the two notions coincide for bounded losses, we focus on unbounded losses. While there always exists an optimistically universal learning rule, the precise class of universally learnable processes depends on an alternative involving the mean-estimation problem. Either mean-estimation on $(\Ycal,\ell)$ is impossible and universal learning is never achievable, or universal learning is achievable for processes that only visit a finite number of distinct points, a property called Finite Support ($\fs$). Along the way, we show that mean-estimation with adversarial responses is always possible for metric losses, a result of independent interest.

\subsection{Organization of the paper} After presenting the learning framework and definitions in \cref{sec:formal_setup}, we describe in \cref{sec:main_results} our main results. Although these are stated for general value spaces under the empirical integrability constraint, the proofs build upon the bounded loss case. We follow this proof structure: in \cref{sec:totally_bounded_value_spaces} we consider totally-bounded value spaces for which we can give explicit optimistically universal learning rules, in \cref{sec:alternative} we consider general bounded loss spaces. We then turn to unbounded and mean estimation in \cref{sec:mean_estimation}. Last, in \cref{sec:unbounded_loss_moment_constraint} we introduce the empirical integrability and prove our general results for unbounded losses. We discuss open directions in \cref{sec:conclusion}.

\section{Formal setup}
\label{sec:formal_setup}
We provide the necessary definitions, concepts and conditions.

\subsection{Instance and value spaces}Consider a separable metric \emph{instance space} $(\Xcal,\rho_\Xcal)$ equipped with its Borel $\sigma-$algebra $\Bcal$, and a separable metric  \emph{value space} $(\Ycal,\rho_\Ycal)$ given with a loss $\ell$. We recall that a metric space is \emph{separable} if it contains a dense countable set. Unless mentioned otherwise, we suppose that the loss is a power of a metric, i.e., there exists $\alpha\geq 1$ such that the loss is $\ell=(\rho_\Ycal)^\alpha$. As a remark, all of the results in this work can be generalized to \emph{essentially separable} metric instance spaces, a condition introduced by \cite{hanneke2021bayes} which was shown to be the largest class of metric spaces for which learning possible. However, for the sake of exposition, we restrict ourselves to separable metric spaces. We denote $\bar\ell := \sup_{y_1,y_2\in\Ycal}\ell(y_1,y_2)$. In the first \cref{sec:totally_bounded_value_spaces,sec:alternative} of this work, we suppose that the loss $\ell$ is \emph{bounded}, i.e., $\bar \ell<\infty$. The case of \emph{unbounded} losses is addressed in the next \cref{sec:mean_estimation,sec:unbounded_loss_moment_constraint}. We also introduce the notion of near-metrics for which we will provide some results. We say that $\ell$ is a near-metric on $\Ycal$ if it is symmetric, satisfies $\ell(y,y)=0$ for all $y\in\Ycal$, for any $y'\neq y\in\Ycal$ we have $\ell(y,y')>0$, and it satisfies a relaxed triangle inequality $\ell(y_1,y_2)\leq c_\ell( \ell(y_1,y_3) + \ell(y_2,y_3))$ where $c_\ell$ is a finite constant.

\subsection{Online learning on adversarial responses}We consider the \emph{online learning} framework where at step $t\geq 1$, one observes a new instance $X_t\in\Xcal$ and predicts a value $\hat Y_t\in\Ycal$ based on the past history $(X_u,Y_u)_{u\leq t-1}$ and the new instance $X_t$ only. The learning rule may be randomized, where the private randomness used at each iteration $t$ is drawn from a fixed probability space $\Rcal$ and independent of the data generation process used to generate $Y_t$.

\begin{definition}
An \emph{online learning rule} is a sequence $f_\cdot:=\{f_t,R_t\}_{t\geq 1}$ of measurable functions $f_t: \Rcal \times \Xcal^{t-1}\times \Ycal^{t-1} \times \Xcal \to\Ycal$ together with a distribution $R_t$ on $\Rcal$.
\end{definition}

The prediction at time $t$ of $f_\cdot$ is $f_t(r_t; (X_u)_{\leq t-1},(Y_u)_{\leq t-1},X_t)$ where $r_t\sim R_t$ is independent of the new value $X_t$ and the past history $(X_u,Y_u)_{\leq t}$. For simplicity, we may omit the internal randomness $r_t$ and write directly $f_t:\Xcal^{t-1}\times \Ycal^{t-1}\times \Xcal\to\Ycal$. We are interested in general data-generating processes. To this means, a possible very general choice of instances and values are general stochastic processes $(\Xbb,\Ybb):=\{(X_t,Y_t)\}_{t\geq 1}$ on the product space $\Xcal\times \Ycal$. This corresponds to the arbitrarily dependent responses under instance processes $\Xbb$ \cite{hanneke2022bayes}. In this work, we consider the slightly more general \emph{adversarial responses} where the value $Y_t$ is also allowed to depend on the past private randomness $(r_u)_{u\leq t-1}$ used by the learning rule $f_\cdot$.

\begin{definition}
Let $\Xbb=(X_t)_{t\geq 1}$ be a stochastic process on $\Xcal$. An \emph{adversarial response mechanism} on $\Xbb$ is a stochastic process $\{(\tilde X_t,\mb Y_t)\}_{t\geq 1}$ where $\tilde X_t\in\Xcal$, $\mb Y_t = \mb Y_t(\cdot\mid\cdot)$ is a Markov kernel from $\Rcal^{t-1}$ to $\Ycal$, and $(\tilde X_t)_{t\geq 1}$ has same distribution as $\Xbb$.
\end{definition}

For a given learning rule $f_\cdot$, having observed the sampled randomness $r_1,\ldots,r_{t-1}\in\Rcal$ used by the learning rule before time $t$, the target value at time $t$ is $Y_t = \mb Y_t(r_1,\ldots,r_{t-1})$. Again, for simplicity, we will refer to the adversarial response mechanism as $\Ybb$, which allows us to view the data generating process as a usual stochastic process on $\Xcal\times\Ycal$. Of course, if the learning rule is \emph{deterministic}, adversarial responses are equivalent to arbitrary dependent responses as in \cite{hanneke2022bayes}, but this is not necessarily the case for general \emph{randomized} algorithms. 

\subsection{Empirically integrable responses}
We introduce a novel assumption on the responses, namely \emph{empirical integrability}.

\begin{definition}
    A process $(Y_t)_{t\geq 1}$ is \emph{empirically integrable} if there exists $y_0\in\Ycal$ such that for any $\epsilon>0$, almost surely there exists $M\geq 0$ for which
\begin{equation*}
    \limsup_{T\to\infty}\frac{1}{T}\sum_{t=1}^T\ell(y_0,Y_t)\1_{\ell(y_0,Y_t)\geq M}\leq \epsilon.
\end{equation*}  
\end{definition}
Unless mentioned otherwise, we will focus on the case where responses satisfy this property. This is a mild assumption on the responses. Indeed, it is worth noting that this condition is always satisfied if the loss $\ell$ is bounded. Further, if for some $y_0\in\Ycal$, $\ell(y_0,Y_t)$ admits moments of order $p>1$, the empirical integrability condition is also satisfied.

\subsection{Universal consistency}In this general setting, we are interested in online learning rules which achieve low long-run average loss compared to any fixed prediction function for general adversarial mechanisms. Given a learning rule $f_\cdot$ and an adversarial process $(\Xbb,\Ybb)$, for any measurable function $f^*:\Xcal\to\Ycal$, we denote the long-run average excess loss as
\begin{equation*}
    \Lcal_{(\Xbb,\Ybb)}(f_\cdot, f^*):= \limsup_{T\to\infty} \frac{1}{T} \sum_{t=1}^T \left(\ell(f_t(\Xbb_{\leq t-1},\Ybb_{\leq t-1},X_t),Y_t) - \ell(f^*(X_t),Y_t)  \right).
\end{equation*}
We can then define the notion of consistency which asks that the excess loss compared to any measurable function vanishes to zero.

\begin{definition}
    Let $(\Xbb,\Ybb)$ be an adversarial process and $f_\cdot$ a learning rule. $f_\cdot$ is consistent under $(\Xbb,\Ybb)$ if for any measurable function $f^*:\Xcal\to\Ycal$, we have $\Lcal_{(\Xbb,\Ybb)}(f_\cdot,f^*)\leq 0,\quad (a.s.)$.
\end{definition}

For example, if $(\Xbb,\Ybb)$ is an i.i.d. process on $\Xcal\times\Ycal$ following a distribution $\mu$ where $\mu$ has a finite first-order moment, achieving consistency is equivalent to reaching the optimal risk $R^*:= \inf_{f^*} \Ebb_{(X,Y)\sim \mu} \left[\ell(f^*(X),Y)\right],$ where the infimum is taken over all measurable functions $f^*:\Xcal\to\Ycal$. As introduced in \cite{hanneke2021learning,hanneke2022bayes}, consistency against all measurable function is the natural extension of consistency for i.i.d. processes $(\Xbb,\Ybb)$ to non-i.i.d. settings. The goal of universal learning is to design learning rules that are consistent for any adversarial process $\Ybb$ that is empirically integrable.

\begin{definition}
Let $\Xbb$ be a stochastic process on $\Xcal$ and $f_\cdot$ a learning rule. $f_\cdot$ is \emph{universally consistent} under $\Xbb$ for empirically integrable adversarial responses if for any adversarial process $(\tilde\Xbb,\Ybb)$ with $\tilde\Xbb\sim\Xbb$ and such that $\Ybb$ is empirically integrable, $f_\cdot$ is consistent.
\end{definition}

\subsection{Optimistic universal learning}
Given this regression setup, we define $\solar$ (Strong universal Online Learning with Adversarial Responses) as the set of processes $\Xbb$ for which universal consistency with adversarial responses is \emph{achievable},
\begin{multline*}
    \solar = \{\Xbb:\exists f_\cdot \text{ universally consistent learning rule under $\Xbb$}
    \\ \text{for empirically integrable adversarial responses}\}.
\end{multline*}

Note that this learning rule is allowed to depend on the process $\Xbb$. Similarly, in the realizable (noiseless) setting, one can define the set $\soul$ (Strong Online Universal Learning) of processes for which there exists a learning rule that is universally consistent for realizable responses when the loss is bounded (and hence, the empirical integrability condition is always satisfied). Of course, $\solar\subset\soul$. We are then interested in learning rules that would achieve universal consistency whenever possible. 

\begin{definition}
A learning rule $f_\cdot$ is \emph{optimistically universal} for adversarial regression with empirically integrable responses if it is universally consistent under all $\Xbb\in\solar$ for adversarial empirically integrable responses.
\end{definition}

Similarly, we say that a learning rule is optimistically universal for noiseless regression if it is universally consistent under all $\Xbb\in\soul$ for noiseless responses when the loss is bounded. In this general framework, the main interests of optimistic learning are 1. identifying the set of learnable processes with adversarial responses $\solar$, 2. determining whether there exists an optimistically universal learning rule, and 3. constructing one if it exists.

\begin{remark}
Except for \cref{subsec:metric_mean_estimation} in which we assume that the loss is a metric $\alpha=1$, one can generalize our results to any symmetric and discernible losses $\ell$ such that for any $0<\epsilon\leq 1$, there exists a constant $c_\epsilon$ such that for all $y_1,y_2,y_3\in\Ycal$, $ \ell(y_1,y_2) \leq (1+\epsilon)\ell(y_1,y_3) + c_\epsilon \ell(y_2,y_3).$ Without loss of generality, we can further assume that $c_\epsilon$ is non-increasing in $\epsilon$. This is a stronger assumption than having a near-metric $\ell$, for which we also give some results in \cref{sec:totally_bounded_value_spaces,sec:unbounded_loss_moment_constraint}.
\end{remark}

\section{Main results}
\label{sec:main_results}

We introduce some conditions on stochastic processes. For any process $\Xbb$ on $\Xcal$, given any measurable set $A\in\Bcal$ of $\Xcal$, let $\hat\mu_\Xbb(A):=\limsup_{T\to\infty}\frac{1}{T}\sum_{t=1}^T \1_A(X_t)$. We consider the condition $\cs$ (Continuous Sub-measure) defined as follows.\\

\noindent \textbf{Condition CS:}
\textit{For every decreasing sequence $\{A_k\}_{k=1}^\infty$ of measurable sets in $\Xcal$ with $A_k\downarrow \emptyset$, $ \Ebb[\hat \mu_{\Xbb }(A_k)] \underset{k\to\infty}{\longrightarrow} 0.$}

It is known that this condition is equivalent to $\Ebb[\hat\mu_\Xbb(\cdot)]$ being a continuous sub-measure \cite{hanneke2021learning}, hence the adopted name $\cs$. Importantly, $\cs$ processes contain in particular i.i.d., stationary ergodic or stationary processes. We now introduce a second condition $\smv$ (Sublinear Measurable Visits) which asks that for any partition, the process $\Xbb$ visits a sublinear number of sets of the partition.\\

\noindent \textbf{Condition $\smv$:} \textit{For every disjoint sequence $\{A_k\}_{k=1}^\infty$ of measurable sets of $\Xcal$ with $\bigcup_{k=1}^\infty A_k=\Xcal$,  (every countable measurable partition),}
\begin{equation*}
    |\{k\geq 1: A_k\cap\Xbb_{\leq T}\neq\emptyset \}| =o(T),\quad (a.s.).
\end{equation*}

This condition is significantly weaker and allows to consider a larger family of processes $\cs\subset\smv$, with $\cs\subsetneq\smv$ whenever $\Xcal$ is infinite \cite{hanneke2021learning}. Note that these sets depend on the instance space $(\Xcal,\rho_\Xcal)$. This dependence is omitted for simplicity. We first consider bounded losses. In the \emph{noiseless} case, where there exists some unknown measurable function $f^*:\Xcal\to\Ycal$ such that the stochastic process $\Ybb$ is given as $Y_t=f^*(X_t)$ for all $t\geq 1$, \cite{blanchard2022universal} showed that learnable processes are exactly $\soul=\smv$ for bounded losses. \cite{blanchard2022universal} also introduced a learning rule 2-Capped-1-Nearest-Neighbor (2C1NN), variant of the classical 1NN algorithm, which is \emph{optimistically universal} in the noiseless case for bounded losses. Interestingly, we show that this same learning rule is universally consistent for unbounded losses in the noiseless setting with empirically integrable responses.

\begin{restatable}{theorem}{ThmNoiselessUnbounded}
\label{thm:noiseless_unbounded}
    Let $(\Ycal,\ell)$ be a separable near-metric space. Then, 2C1NN is optimistically universal in the noiseless setting with empirically integrable responses, i.e., for all processes $\Xbb\in\smv$ and for all measurable target functions $f^*:\Xcal\to\Ycal$ such that $(f^*(X_t))_{t\geq 1}$ is empirically integrable, $\Lcal_{(\Xbb,(f^*(X_t))_{t\geq 1})}(2C1NN, f^*)=0\;(a.s.)$.
\end{restatable}

In general, one has $\solar\subset\smv$. It was posed as a question whether we could recover the complete set $\smv$ for learning under adversarial---or arbitrary---processes \cite{hanneke2022bayes}.\\

\noindent\textbf{Question \cite{hanneke2022bayes}:} For bounded losses, does there exist an online learning rule that is universally consistent for arbitrary responses under all processes $\Xbb\in\smv(=\soul)$?\\

We answer this question with an alternative. Depending on the bounded value space $(\Ycal,\ell)$, either $\solar=\smv$ or $\solar=\cs$, but in both cases there exists an optimistically universal learning rule. We now introduce the property $\ftime$ (Finite-Time Mean Estimation) on the value space $(\Ycal,\ell)$ which characterizes this alternative.\\

\noindent\textbf{Property $\ftime$:} \textit{For any $\eta>0$, there exists a horizon time $T_\eta \geq 1$, an online learning rule $g_{\leq T_\eta}$ such that for any $\mb y:=(y_t)_{t=1}^{T_\eta}$ of values in $\Ycal$ and any value $y\in\Ycal$, we have
\begin{equation*}
    \frac{1}{T_\eta}\Ebb\left[\sum_{t=1}^{T_\eta} \ell(g_t({\mb y}_{\leq t-1}), y_t) - \ell(y,y_t) \right] \leq \eta.
\end{equation*}
}

We are now ready to state our main results for bounded value spaces. The first result shows that if the value space satisfies the above property locally, we can universally learn all the processes in $\soul$ even under adversarial responses.

\begin{restatable}{theorem}{ThmSOULRegressionUnbounded}
\label{thm:SOUL_regression_unbounded}
Suppose that any ball of $(\Ycal,\ell)$, $B_\ell(y,r)$ satisfies $\ftime$. Then, $\solar=\smv$ and there exists an optimistically universal learning rule $f_\cdot$ for adversarial regression with empirically integrable responses., i.e., such that for any stochastic process $(\Xbb,\Ybb)$ on $\Xcal\times\Ycal$ with $\Xbb\in\smv$ and $\Ybb$ empirically integrable, for any measurable function $f:\Xcal\to\Ycal$ we have $\Lcal_{(\Xbb,\Ybb)}(f_\cdot,f^*)\leq 0,\quad (a.s.)$.
\end{restatable}

$\ftime$ defines a non-trivial alternative, and an explicit construction of a non-$\ftime$ bounded metric space $(\Ycal,\rho_\Ycal)$ is given in \cref{subsec:bad_non_totally_bounded_ex} with $\Ycal=\Nbb$. Nevertheless, $\ftime$ is satisfied by a large class of spaces, e.g., any totally-bounded metric space and countable classification $(\Ycal,\ell)=(\Nbb,\ell_{01})$ satisfy $\ftime$. Hence, we can universally learn all $\soul$ processes with adversarial responses, for countable classification (the empirical integrability condition is automatically satisfied because the loss is bounded). If $\ftime$ is not satisfied locally, we have the following result which shows that learning under $\cs$ is still possible but universal learning beyond $\cs$ processes cannot be achieved.

\begin{restatable}{theorem}{ThmCSRegressionUnbounded}
\label{thm:CS_regression_unbounded}
Suppose that there exists a ball $B_\ell(y,r)$ of $(\Ycal,\ell)$ that does not satisfy $\ftime$. Then, $\solar=\cs$ and there exists an optimistically universal learning rule $f_\cdot$ for adversarial regression with empirically integrable responses., i.e., such that for any stochastic process $(\Xbb,\Ybb)$ on $(\Xcal,\Ycal)$ with $\Xbb\in\cs$ and $\Ybb$ empirically integrable, then, for any measurable function $f:\Xcal\to\Ycal$ we have $\Lcal_{(\Xbb,\Ybb)}(f_\cdot,f^*)\leq 0,\quad (a.s.)$.
\end{restatable}

For metric losses $\ell=\rho_\Ycal$, it was already known \cite{hanneke2022bayes} that universal learning under adversarial responses under all processes in $\cs$ is achievable by some learning rule. Hence, \cref{thm:CS_regression_unbounded} implies that this learning rule is automatically optimistically universal for adversarial regression for all metric value spaces with bounded loss which do not satisfy $\ftime$. However, our result is stronger in that consistency holds for any power of a metric loss $\ell=\rho_\Ycal^\alpha,\alpha\geq 1$ and unbounded value spaces.

\begin{remark}
As a direct consequence of \cref{thm:SOUL_regression_unbounded,thm:CS_regression_unbounded}, for stationary ergodic processes, finite second moment of the values $\Ebb[Y^2]<\infty$ suffices for consistency, in agreement with the known results for the i.i.d. setting. This relaxes the fourth-moment conditions $\Ebb[Y^4]<\infty$ proposed in the literature \cite{gyorfi:07}.
\end{remark}

We now consider removing the empirical integrability assumption. As mentioned above, for bounded losses this assumption is automatically satisfied, hence \cref{thm:SOUL_regression_unbounded,thm:CS_regression_unbounded} apply directly, with a simplified alternative: whether $(\Ycal,\ell)$ satisfies $\ftime$.

\begin{corollary}
\label{cor:good_value_spaces}
Suppose that $\ell$ is bounded.
\begin{itemize}
    \item If $(\Ycal,\ell)$ satisfies $\ftime$. Then, $\solar = \smv (=\soul)$.
    \item If $(\Ycal,\ell)$ does not satisfy $\ftime$. Then, $\solar = \cs$.
\end{itemize}
Further, an optimistically universal learning rule for adversarial regression always exists, i.e., achieving universal consistency with adversarial responses under any $\Xbb\in\solar$.
\end{corollary}

It remains to analyze the case of unbounded losses without empirical integrability assumption on the responses. To avoid confusions, we denote by $\solaru$ the set of processes that admit universal learning with adversarial (unrestricted) responses. Unfortunately, even in the noiseless setting, universal learning is extremely restrictive in that case. Specifically, the set of universally learnable processes $\soul$ for noiseless responses is reduced to the set FS (Finite Support) of processes that visit a finite number of different points almost surely \cite{blanchard2022optimistic}.\\

\noindent\textbf{Condition $\fs$:} The process $\Xbb$ satisfies $|\{x\in \Xcal: \{x\}\cap \Xbb \neq \emptyset\}|<
\infty\quad (a.s.)$.\\

We show that in the adversarial setting we still have $\solaru=\fs$ when $\ell$ is a metric: we can solve the fundamental problem of mean estimation where one sequentially makes predictions of a sequence $\Ybb$ of values in $(\Ycal,\ell)$ and aims to have a better long-run average loss than any fixed value. If responses $\Ybb$ are i.i.d. this is the Fr\'echet means estimation problem \cite{evans2020strong,schotz2022strong,jaffe2022strong}. Our main result on mean estimation holds in general spaces and is of independent interest.

\begin{restatable}{theorem}{ThmMeanEstimation}
\label{thm:mean_estimation}
Let $(\Ycal,\ell)$ be a separable metric space. There exists an online learning rule $f_\cdot$ that is universally consistent for adversarial mean estimation, i.e., for any adversarial process $\Ybb$ on $\Ycal$, almost surely, for all $y\in \Ycal$,
\begin{equation*}
    \limsup_{T\to\infty}\frac{1}{T} \sum_{t=1}^T \left(\ell( f_t(\Ybb_{\leq t-1}),Y_t) - \ell(y,Y_t)\right)\leq 0.
\end{equation*}
\end{restatable}

Further, we show that for powers of metric we may have $\solaru=\emptyset$. Specifically, for real-valued regression with Euclidean norm and loss $|\cdot|^\alpha$ and $\alpha>1$, adversarial regression or mean estimation are not achievable. We then show that we have an alternative: either mean estimation with adversarial responses is achievable, $\solaru=\fs$ and we have an optimistically universal learning rule; or mean estimation is not achievable and $\solaru=\emptyset$. Thus, even in the best case scenario for unbounded losses, $\solaru=\fs$, which is already extremely restrictive. \cite{blanchard2022optimistic} asked whether imposing moment conditions on the responses would allow recovering the large set $\smv$ as learnable processes instead. Specifically, they formulated the following question.\\

\noindent\textbf{Question \cite{blanchard2022optimistic}:} For unbounded losses $\ell$, 
does there exist an online learning rule $f_\cdot$ which is consistent under every $\Xbb \in \smv$, for every measurable function $f^*:\Xcal\to\Ycal$ such that there exists $y_0\in\Ycal$ with $\limsup_{T \to \infty} \frac{1}{T} \sum_{t=1}^{T} \ell(y_0, f^*(X_t)) < \infty~~\text{(a.s.)}$, i.e., such that we have $\mathcal{L}_{\Xbb}(f_{\cdot},f^*) = 0~~\text{(a.s.)}$?\\

We answer negatively to this question. Under this first-moment condition, universal learning under all $\smv$ processes is not achievable even in this noiseless case. We show the stronger statement that noiseless universal learning under all processes having pointwise convergent relative frequencies---which are included in $\cs$---is not achievable. However, under the empirical integrability condition introduced above we are able to recover all positive results from bounded losses.

\begin{table}[h]
\caption{Characterization of learnable instance processes in universal consistency (ME = Mean Estimation).}
\label{table:summary_of_results}

\resizebox{\columnwidth}{!}{
\begin{tabular}{|c |c |c |c|} 
 \hline
 $\begin{array}{c}\text{Learning}\\
 \text{setting}
 \end{array}$& Bounded loss & Unbounded loss & $\begin{array}{c}
    \text{Unbounded loss with} \\
     \text{empirically integrable}\\
     \text{responses}
 \end{array}$\\ [0.5ex] 
 \hline\hline
 $\begin{array}{c}\text{Noiseless}\\
 \text{responses}
 \end{array}$  & $
        \soul=\smv $ \cite{blanchard2022universal} & $\soul=\fs$ \cite{blanchard2022optimistic} & $\begin{array}{c}
            \text{Identical to}  \\
            \text{bounded loss}
        \end{array} \mb{[\text{This paper}]} $\\ 
 \hline
 $\begin{array}{c}
 \text{Adversarial}\\
 \text{(or arbitrary)}\\
 \text{responses}
 \end{array}$  & $\begin{array}{c}
    \solar\supset\cs  \text{ (metric loss) \cite{hanneke2022bayes}} \\[0.5ex] 
    \hline
    \text{Does }(\Ycal,\ell) \text{ satisfy }\ftime ?\\
   \begin{cases}
        \text{Yes} & \solar =\smv\\
        \text{No} & \solar =\cs 
    \end{cases}\mb{[\text{This paper}]} 
 \end{array}$ & $\begin{array}{c}
      \text{Is ME achievable?} \\
      \begin{cases}
        \text{Yes} & \solaru=\fs \\
        \text{No} & \solaru =\emptyset
        \end{cases} \mb{[\text{This paper}]} 
 \end{array}$ &  $\begin{array}{c}
            \text{Identical to}  \\
            \text{bounded loss}
        \end{array} \mb{[\text{This paper}]} $\\
 \hline
\end{tabular}
}

\end{table}

\begin{table}[h]
\caption{Proposed learning rules for universal consistency (ME = Mean Estimation and EI = Empirical Integrability).\protect\footnotemark }
\label{table:summary_of_learning_rules}

\resizebox{\columnwidth}{!}{
\begin{tabular}{|c |l |l |c|c|c|} 
 \hline 
 $\begin{array}{c}\text{Learning}\\
 \text{setting}
 \end{array}$& Loss (and response/setting constraints) & Learning rule & $\begin{array}{c}
     \text{Guarantees} \\
     \text{for which}\\
     \text{processes }\Xbb? 
 \end{array}$ &$\begin{array}{c}
      \text{Optimist.}\\
      \text{universal?} 
 \end{array}$ & Reference\\ [0.5ex] 
 \hline\hline
    I.i.d.  & Finite or countable class., 01-loss & OptiNet & i.i.d. & No & \cite{hanneke2021bayes} \\
    responses & Real-valued regression + integrable & Proto-NN & i.i.d. & No & \cite{gyorfi2021universal}\\
       & Metric loss + integrable & MedNet & i.i.d. & No &\cite{tsir2022metric}\\
 \hline
 Noiseless & Bounded loss & 2C1NN &$\smv$ & Yes &\cite{blanchard2022universal}\\
 responses & Unbounded loss & Memorization &$\fs$ & Yes &\cite{blanchard2022optimistic}\\ 
  (realizable) & Unbounded + EI & 2C1NN &$\smv$ & Yes & [This paper]\\
 \hline
 & Bounded loss + metric loss & Hedge-variant &$\cs$ &Not always &\cite{hanneke2022bayes}\\
 & Bounded loss + $\ftime$ & $(1+\delta)$C1NN-hedged & $\smv$ & Yes & [This paper]\\
Adversarial & Bounded loss + not $\ftime$ & Hedge-variant 2 & $\cs$ & Yes & [This paper]\\
(or arbitrary) & Unbounded loss + ME & ME-algorithm & $\fs$ & Yes &[This paper]\\
responses  & Unbounded loss + not ME & N/A & $\emptyset $ & N/A &[This paper]\\
 & Unbounded + EI + local $\ftime$&  EI-$(1+\delta)$C1NN-hedged &$\smv$ & Yes & [This paper]\\
 & Unbounded + EI + not local $\ftime$ & EI-Hedge-variant & $\cs$ & Yes & [This paper]\\
 \hline
\end{tabular}
}

\end{table}

\footnotetext{In our paper, an algorithm is optimistically universal if it is universally consistent for all processes under which universal learning is possible in the considered setting. OptiNet, Proto-NN, and MedNet are optimistically universal in another sense, their guarantees hold in all metric spaces for which universal learning with i.i.d. pairs of instances and responses is achievable: \emph{essentially separable} spaces $(\Xcal,\rho_\Xcal)$ \cite{hanneke2021bayes}. Our learning rules also enjoy this second optimistic property.}

\cref{table:summary_of_results,table:summary_of_learning_rules} summarize known results in the literature and our contributions. As a reminder, $\fs\subset\cs\subset\smv$ in general, and $\fs\subsetneq \cs\subsetneq \smv$ whenever $\Xcal$ is infinite \cite{hanneke2021learning}.

\section{An optimistically universal learning rule for totally-bounded value spaces}
\label{sec:totally_bounded_value_spaces}

We start our analysis of universal learning under adversarial responses with \emph{totally-bounded} value spaces, for which we can give simple and explicit algorithms. Hence, we suppose in this section that the value space $(\Ycal,\ell)$ is totally-bounded, i.e., for any $\epsilon>0$ there exists a finite $\epsilon-$net $\Ycal_\epsilon$ of $\Ycal$ such that for any $y\in\Ycal$, there exists $y'\in\Ycal_\epsilon$ with $\ell(y,y')<\epsilon$. In particular, a totally-bounded space is necessarily bounded and separable.  The goal of this section is to show that for such value spaces, adversarial universal regression is achievable for all processes in $\smv$ as in the noiseless setting (the empirical integrability assumption is automatically satisfied in this context). Further, we explicitly construct an optimistically universal learning rule for adversarial responses.\\

We recall that in the noiseless setting, the 2C1NN learning rule achieves universal consistency for all $\smv$ processes \cite{blanchard2022universal}. At each iteration $t$, This rule performs the nearest neighbor rule over a restricted dataset instead of the complete history $\Xbb_{\leq t-1}$. The dataset is updated by keeping track of the number of times each point $X_u$ was used as nearest neighbor. This number is then capped at $2$ by deleting from the current dataset any point which has been used twice as representative. Unfortunately, this learning rule is not optimistically universal for adversarial responses. More generally, \cite{tsir2022metric} noted that any learning rule which only outputs observed historical values cannot be consistent, even in the simplest case of $\Xcal=\{0\}$ and i.i.d. responses $\Ybb$. For instance, take $\Ycal=\bar B(0,1)$ the closed ball of radius $1$ in the plane $\Rbb^2$ with the euclidean loss, consider the points $A,B,C\in\Ycal$ representing the equilateral triangle $e^{2ik\pi/3}$ for $k=0,1,2$, and let $\Ybb$ be an i.i.d. process following the distribution which visits $A$, $B$ or $C$ with probability $\frac{1}{3}$. Predictions within observed values, i.e., $A,B$ or $C$, incur an average loss of $\frac{2}{3}\sqrt 3 >1$ where $1$ is the loss obtained with the fixed value $(0,0)$.

To construct an optimistically universal learning rule for adversarial responses, we first generalize a result from \cite{blanchard2022universal}. Instead of the 2C1NN learning rule, we use $(1+\delta)$C1NN rules for $\delta>0$ arbitrarily small. Similarly as in 2C1NN, each new input $X_t$ is associated to a representative $\phi(t)$ used for the prediction $\hat Y_t =Y_{\phi(t)}$. In the $(1+\delta)$C1NN rule, each point is used as a representative at most twice with probability $\delta$ and at most once with probability $1-\delta$. In order to have this behavior irrespective of the process $\Xbb$, which can be thought of been chosen by a (limited) adversary within the $\soul$ processes, the information of whether a point can allow for 1 or 2 children is only revealed when necessary. Specifically, at any step $t\geq 1$, the algorithm initiates a search for a representative $\phi(t)$. It successively tries to use the nearest neighbor of $X_t$ within the current dataset and uses it as a representative if allowed by the maximum number of children that this point can have. However, the information whether a potential representative $u$ can have at most 1 or 2 children is revealed only when $u$ already has one child.
\begin{itemize}
    \item If $u$ allows for 2 children, it will be used as final representative $\phi(t)$.
    \item Otherwise, $u$ is deleted from the dataset and the search for a representative continues.
\end{itemize}
The rule is formally described in \cref{alg:1+deltaC1NN}, where $\bar y\in\Ycal$ is an arbitrary value, and the maximum number of children that a point $X_t$ can have is represented by $1+U_t$. In this formulation, all Bernouilli $\Bcal(\delta)$ samples are drawn independently of the past history. Note that if $\delta=1$, the $(1+\delta)$C1NN learning rule coincides with the 2C1NN rule of \cite{blanchard2022universal}.

\begin{algorithm}[tb]
\caption{The $(1+\delta)$C1NN learning rule}\label{alg:1+deltaC1NN}
\hrule height\algoheightrule\kern3pt\relax
\KwIn{Historical samples $(X_t,Y_t)_{t<T}$ and new input point $X_T$}
\KwOut{Predictions $\hat Y_t = (1+\delta)C1NN_t({\mb X}_{<t},{\mb Y}_{<t},X_t)$ for $t\leq T$}
$\hat Y_1:= \bar y$ \tcp*[f]{Arbitrary prediction at $t=1$}\\
$\Dcal_2\gets \{1\}$;
$n_1 \gets 0$;
\tcp*[f]{Initialisation}\\
\For{$t=2,\ldots, T$}{
    \eIf{exists $u<t$ such that $X_u=X_t$}{
        $\hat Y_t := Y_u$
    }
    {
        $continue\gets True$ \tcp*[f]{Begin search for available representative $\phi(t)$}\\
        \While{continue}{
            $\phi(t)\gets \min \left\{l\in \arg\min_{u\in \Dcal_t}  \rho_\Xcal(X_t,X_u)  \right\}$\\
          
            \uIf(\tcp*[f]{Candidate representative has no children}){$n_{\phi(t)}=0$}{
                $\Dcal_{t+1}\gets \Dcal_t\cup\{t\}$\\
                $continue\gets False$
            }
            \uElse(\tcp*[f]{Candidate representative has one child}){
                $U_{\phi(t)}\sim \Bcal(\delta)$\\
                \uIf{$U_{\phi(t)}=0$}{
                    $\Dcal_t \gets \Dcal_t\setminus\{\phi(t)\}$
                }
                \uElse{
                    $\Dcal_{t+1}\gets (\Dcal_t\setminus\{\phi(t)\})\cup\{t\}$\\
                    $continue\gets False$
                }
            }
      
        }
    }

    $\hat Y_t:=Y_{\phi(t)}$\\
      
    $n_{\phi(t)}\gets n_{\phi(t)}+1$\\
    $n_t\gets 0$
}
\hrule height\algoheightrule\kern3pt\relax
\end{algorithm}

\begin{theorem}
\label{thm:1+deltaC1NN_optimistic}
Fix $\delta>0$. For any separable Borel space $(\Xcal,\Bcal)$ and any separable near-metric output setting $(\Ycal,\ell)$ with bounded loss, in the noiseless setting, $(1+\delta)$C1NN is optimistically universal.
\end{theorem}

We now construct our algorithm. This learning rule uses a collection of algorithms $f^\epsilon_\cdot$ which each yield an asymptotic error at most a constant factor from $\epsilon^{\frac{1}{\alpha+1}}$. Now fix $\epsilon>0$ and let $\Ycal_\epsilon$ be a finite $\epsilon-$net of $\Ycal$ for $\ell$. Recall that we denote by $\bar\ell$ the supremum loss. We pose
\begin{equation*}
    T_\epsilon := \left\lceil \frac{\bar \ell^2 \ln |\Ycal_\epsilon|}{2\epsilon^2}\right\rceil \quad \text{and} \quad \delta_\epsilon:= \frac{\epsilon}{2T_\epsilon}.
\end{equation*}
The quantity $T_\epsilon$ will be the horizon window used by our learning rule to make its prediction using the $(1+\delta_\epsilon)$C1NN learning rule. Precisely, let $\phi$ be the representative function from the $(1+\delta_\epsilon)$C1NN learning rule. Note that this representative function $\phi(t)$ is defined only for times $t$ where a new instance $X_t$ is revealed, otherwise $(1+\delta_\epsilon)$C1NN uses simple memorization $\hat Y_t = Y_u$. For simplicity, we will denote by $\Ncal=\{t:\forall u<t,X_u\neq X_t\}$ these times of new instances. For $t\in \Ncal$, we denote by $d(t)$ the depth of time $t$ within the graph constructed by $(1+\delta_\epsilon)$C1NN, and define the horizon $L_t=d(t)\mod T_\epsilon$. Intuitively, the learning rule $f^\epsilon_\cdot$ performs the classical Hedge algorithm \cite{cesa2006prediction} on clusters of times that are close within the graph $\phi$. Precisely, we define the equivalence relation between times as follows:
\begin{equation*}
    t_1 \stackrel \phi \sim  t_2 \quad \Longleftrightarrow \quad \begin{cases}
        \phi^{L_{u_1}}(u_1) = \phi^{L_{u_2}}(u_2) &\text{ and } |\{u<t_i:X_u = X_{t_i}\}|\leq \frac{T_\epsilon}{\epsilon},\; i=1,2 \\
        \text{or}&\\
        X_{t_1} = X_{t_2} &\text{ and } |\{u<t_i:X_u = X_{t_1}\}|> \frac{T_\epsilon}{\epsilon},\; i=1,2,
    \end{cases}
\end{equation*}
where $u_i = \min\{u: X_u = X_{t_i}\}$ is the first occurrence of the considered instance point $X_{t_i}$. Hence, multiple occurrences of the same instance value fall in the same cluster and for new instance points times $t\in\Ncal$, all times of a given cluster share the same ancestor up to generation at most $T_\epsilon-1$. Additionally, a cluster is dedicated to instance points that have a significant number of duplicates. To make its prediction at time $t$, $f^\epsilon_\cdot$ performs the Hedge algorithm based on values observed on its current cluster $\{u\leq t: u \stackrel \phi \sim  t\}$. Let $\eta_\epsilon:=\sqrt{\frac{8\ln |\Ycal_\epsilon|}{\bar\ell^2 T_\epsilon}}$ and define the losses $L_y^t=\sum_{u<t:u \stackrel \phi \sim  t} \ell(Y_u,y)$. The learning rule $f^\epsilon_t(\Xbb_{\leq t-1},\Ybb_{\leq t-1},X_t)$ outputs a random value in $\Ycal_\epsilon$ independently from the past history with
\begin{equation*}
    \Pbb(\hat Y_t(\epsilon)=y) = \frac{e^{-\eta_\epsilon L_y^t}}{\sum_{z\in \Ycal_\epsilon} e^{-\eta_\epsilon L_z^t}},\quad y\in \Ycal_\epsilon,
\end{equation*}
where, for simplicity, we denoted $\hat Y_t(\epsilon)$ the prediction given by the learning rule $f^\epsilon_\cdot$ at time $t$. 

\begin{algorithm}[tb]
\caption{The $f_\cdot^\epsilon$ learning rule}\label{alg:f_epsilon}
\hrule height\algoheightrule\kern3pt\relax
\KwIn{Historical samples $(X_t,Y_t)_{t<T}$ and new input point $X_T$,\\
\quad \quad \quad \, Representatives $\phi_\epsilon(\cdot)$ and depths $d_\epsilon(\cdot)$ constructed iteratively within $(1+\delta_\epsilon)$C1NN.}
\KwOut{Predictions $\hat Y_t(\epsilon) = f_t^\epsilon({\mb X}_{<t},{\mb Y}_{<t},X_t)$ for $t\leq T$}
$\Ycal_\epsilon$ an $\epsilon-$net of $\Ycal$\\
$ T_\epsilon := \left\lceil \frac{\bar \ell^2 \ln |\Ycal_\epsilon|}{2\epsilon^2}\right\rceil, \quad \eta_\epsilon:=\sqrt{\frac{8\ln |\Ycal_\epsilon|}{\bar\ell^2 T_\epsilon}}$\\

\For{$t=1,\ldots,T$}{
    $L_y^t = \sum_{u<t:u \stackrel {\phi_\epsilon} \sim  t} \ell(Y_u,y), \quad y\in\Ycal_\epsilon$ \tcp*[f]{Losses on the cluster given by $\phi_\epsilon$}\\
    
    $p^t(y) = \displaystyle \frac{\exp(-\eta_\epsilon L_y^t)}{\sum_{z\in \Ycal_\epsilon} \exp(-\eta_\epsilon L_z^t)},\quad y\in \Ycal_\epsilon$\\
    
    $\hat Y_t\sim p^t$
}
\hrule height\algoheightrule\kern3pt\relax
\end{algorithm}

Having constructed the learning rules $f^\epsilon_\cdot$, we are now ready to define our final learning rule $f_\cdot$. Let $\epsilon_i=2^{-i}$ for all $i\geq 0$. Intuitively, it aims to select the best prediction within the rules $f_\cdot^{\epsilon_i}$. If there were a finite number of such predictors, we could directly use the algorithms for learning with experts from the literature \cite{cesa2006prediction}. Instead, we introduce these predictors one at a time: at step $t\geq 1$ we only consider the indices $I_t:=\{i\leq \ln t\}$. We then compute an estimate $\hat L_{t-1,i}$ of the loss incurred by each predictor $f_\cdot^{\epsilon_i}$ for $i\in I_t$ and select a random index $\hat i_t$ independent from the past history from an exponentially-weighted distribution based on the estimates $\hat L_{t-1,i}$. The final output of our learning rule is $\hat Y_t:= \hat Y_t(\epsilon_{\hat i})$. The complete algorithm is formally described in \cref{alg:optim_rule}. The following lemma quantifies the loss of the rule $f_\cdot$ compared to the best rule $f^{\epsilon_i}_\cdot$.

\begin{algorithm}[tb]
\caption{An optimistically universal learning rule for totally bounded spaces}\label{alg:optim_rule}
\hrule height\algoheightrule\kern3pt\relax
\KwIn{Historical samples $(X_t,Y_t)_{t<T}$ and new input point $X_T$,\\
\quad \quad \quad \, Predictions $\hat Y_\cdot(\epsilon_i)$ from the learning rules $f^{\epsilon_i}_\cdot$.}
\KwOut{Predictions $\hat Y_t$ for $t\leq T$}
$w_{0,0}=1, t_i:=\lceil e^i\rceil,\quad i\geq 0$ \\
$I_t = \{i\leq \ln t\}, \; \eta_t = \sqrt{\frac{\ln t}{t}}, \quad t\geq 1$\\

\For{$t=1,\ldots,T$}{
    
    $L_{t-1,i}:=\sum_{s=t_i}^{t-1}\ell(\hat Y_s(\epsilon_i),Y_s),\; \hat L_{t-1,i}:=\sum_{s=t_i}^{t-1}\hat\ell_s, \quad i\in I_t$\\
    $w_{t-1,i}=e^{\eta_t(\hat L_{t-1,i}-L_{t-1,i})}$\\
    $p_t(i) = \displaystyle \frac{w_{t-1,i}}{\sum_{j\in I_t}w_{t-1,j}}$\\
    $\hat i_t\sim p_t(\cdot)$  \tcp*[f]{model selection}\\
    $\hat Y_t = \hat Y_t(\epsilon_i)$\\
    
    $ \hat\ell_t:= \frac{\sum_{i\in I_t} w_{t-1,i}\ell(\hat Y_t(\epsilon_i),Y_t)}{\sum_{i\in I_t} w_{t-1,i}}.$
}
\hrule height\algoheightrule\kern3pt\relax
\end{algorithm}

\begin{lemma}\label{lemma:concatenation_predictors}
Almost surely, there exists $\hat t\geq 0$ such that
\begin{equation*}
    \forall t\geq \hat t,\forall i\in I_t,\quad \sum_{s=t_i}^t\ell(\hat Y_t,Y_t) \leq \sum_{s=t_i}^t \ell(\hat Y_t(\epsilon_i),Y_t) + (2+\bar\ell+\bar\ell^2)\sqrt {t\ln t}.
\end{equation*}
\end{lemma}

We are now ready to show that \cref{alg:optim_rule} is universally consistent under $\smv$ processes.

\begin{theorem}
\label{thm:optimistic_regression_totally_bounded}
Suppose that $(\Ycal,\ell)$ is totally-bounded. There exists an online learning rule $f_\cdot$ which is universally consistent for adversarial responses under any process $\Xbb\in\smv(=\soul)$, i.e., for any  process $(\Xbb,\Ybb)$ on $(\Xcal,\Ycal)$ with adversarial response, such that $\Xbb\in\smv$, then for any measurable function $f:\Xcal\to\Ycal$, we have $\Lcal_{(\Xbb,\Ybb)}(f_\cdot,f)\leq 0,\quad (a.s.)$.
\end{theorem}

\noindent \textbf{Proof sketch.} First observe that \cref{lemma:concatenation_predictors} allows us to combine predictors $f_\cdot^\epsilon$: if individually they perform well, \cref{alg:optim_rule} achieves the best long-term average excess loss among them. We then proceed to show that $f_\cdot^\epsilon$ has low average error in the long run. First, $(1+\delta_\epsilon)$C1NN is universally consistent on $\smv$ processes in the noiseless setting by \cref{thm:1+deltaC1NN_optimistic}. This intuitively shows that for noiseless functions, the value at time $\phi_\epsilon(t)$ provides a good representative for the value at time $t$. Extrapolating this argument, we show that if two times are close (for the graph metric) within the graph formed by $\phi_\epsilon$, they will have close values for any fixed function in the long run. As a result, times in the same cluster defined by $\overset{\phi_\epsilon}{\sim}$ share similar values in the long run. The $f_\cdot^\epsilon$ rule precisely aims to learn the best predictor by cluster using the classical Hedge algorithm. Because it can only ensure low regret compared to a finite number of options, we use $\epsilon$-nets of the value space $\Ycal$. The reason why we need to have $(1+\delta_\epsilon)$C1NN instead of the known 2C1NN algorithm is that for a given time $T$, we need to ensure low excess error even though some clusters might not be completed. Because the tree formed by $\phi_\epsilon$ resembles a $(1+\delta_\epsilon)$-branching process, the fraction of times which belong to unfinished clusters is only a small fraction $\epsilon T$ of the $T$ times, hence does not affect the average long-term excess error significantly. Altogether, we show that $f^\epsilon_\cdot$ has $\Ocal(\epsilon^{\frac{1}{\alpha+1}})$ long-term average excess error compared to any fixed function for any $\smv$ process, which ends the proof.\\

As a result, $\smv\subset\solar$ for totally-bounded value spaces. Recalling that for bounded values $\smv=\soul$ \cite{blanchard2022universal}, i.e., processes $\Xbb\notin\smv$ are not universally learnable even in the noiseless setting, we have $\solar\subset\smv$. Thus we obtain a complete characterization of the processes which admit universal learning with adversarial responses: $\solar=\smv$. Further, the proposed learning rule is optimistically universal for adversarial regression.

\begin{corollary}
\label{cor:optimistic_totally_bounded}
Suppose that $(\Ycal,\ell)$ is totally-bounded. Then, $\solar = \smv$, and there exists an optimistically universal learning rule for adversarial regression, i.e., which achieves universal consistency with adversarial responses under any process $\Xbb\in\solar$.
\end{corollary}

This is a first step towards the more general \cref{thm:good_value_spaces}. Indeed, one can note that $\ftime$ is satisfied by any totally-bounded value space: given a fixed error tolerance $\eta>0$, consider a finite $\frac{\eta}{2}-$net $\Ycal_{\eta/2}$ of $\Ycal$. Because this is a finite set, we can perform the classical Hedge algorithm \cite{cesa2006prediction} to have $\Theta(\sqrt {T\ln |\Ycal_{\eta/2}|})$ regret compared to the best fixed value of $\Ycal_{\eta/2}$. For example, if $\alpha=1$, posing $T_\eta=\Theta(\frac{4}{\eta^2}\ln |\Ycal_{\eta/2}|)$ enables to have a regret at most $\frac{\eta}{2} T_\eta$ compared to any fixed value of $\Ycal_{\eta/2}$, hence regret at most $\eta T_\eta$ compared to any value of $\Ycal$. This achieves $\ftime$, taking a deterministic time $\tau_\eta:=T_\eta$.

\section{Characterization of learnable processes for bounded losses}\label{sec:alternative}

While \cref{sec:totally_bounded_value_spaces} focused on totally-bounded value spaces, the goal of this section is to give a full characterization of the set $\solar$ of processes for which adversarial regression is achievable and provide optimistically universal algorithms, for any \emph{bounded} value space.

\subsection{Negative result for non-totally-bounded spaces}
\label{subsec:bad_non_totally_bounded_ex}

Although for all bounded value spaces $(\Ycal,\ell)$, noiseless universal learning is achievable on all $\smv(=\soul)$ processes, this is not the case for adversarial regression in non-totally-bounded spaces. We show in this section that extending \cref{cor:optimistic_totally_bounded} to any bounded value space is impossible: the set of learnable processes for adversarial regression may be reduced to $\cs$ only, instead of $\smv$.

\begin{theorem}
\label{thm:negative_optimistic}
Let $(\Xcal,\Bcal)$ a separable Borel metrizable space. There exists a separable metric value space $(\Ycal,\ell)$ with bounded loss such that the following holds: for any process $\Xbb\notin\cs$, universal learning under $\Xbb$ for arbitrary responses is not achievable. Precisely, for any learning rule $f_\cdot$, there exists a process $\Ybb$ on $\Ycal$, a measurable function $f^*:\Xcal\to\Ycal$ and $\epsilon>0$ such that with non-zero probability $    \Lcal_{(\Xbb,\Ybb)}(f_\cdot,f^*) \geq  \epsilon.$
\end{theorem}

In the proof, we explicitly construct a bounded metric space that does not satisfy $\ftime$. More precisely, we choose $\Ycal=\Nbb=\{i\geq 0\}$ and a specific metric loss $\ell$ with values in $\{0,\frac{1}{2},1\}$. For any $k\geq 1$, we pose $n_k:=2k(k-1)+2^k-1$ and define the sets
\begin{equation*}
    I_k:=\{n_k,n_k+1,\ldots,n_k+4k-1\} \quad \text{and} \quad J_k:=\{n_k+4k,n_k+4k+1,\ldots, n_{k+1}-1\}.
\end{equation*}
These sets are constructed so that $|I_k|=4k$, $|J_k|=2^k$ for all $k\geq 1$, and together with $\{0\}$, they form a partition of $\Nbb$. We now construct the loss $\ell$. We pose $\ell(i,j)=\1_{i=j}$ for all $i,j\in \Nbb$ unless there is $k\geq 1$ such that $(i,j)\in I_k\times J_k$ or $(j,i)\in I_k\times J_k$. It now remains to define the loss $\ell(i,j)$ for all $i\in I_k$ and $j\in J_k$. Note that for any $j\in J_k$, we have that $j-n_k-4k\in\{0,\ldots,2^k-1\}$. Hence we will use their binary representation which we write as $j-n_k-4k = \{b_j^{k-1}\ldots b_j^1 b_j^0\}_2 = \sum_{u=0}^{k-1} b_j^u 2^u$ where $b_j^0,b_j^1,\ldots,b_j^{k-1}\in\{0,1\}$ are binary digits. Finally, we pose
\begin{align*}
    \ell(n_k + 4u,j)= \ell(n_k + 4u+1,j) &= \frac{1+b^u_j}{2}, \\
    \ell(n_k + 4u+2,j)= \ell(n_k + 4u+3,j) &= \frac{2-b^u_j}{2},
\end{align*}
for all $u\in\{0,1,\ldots,k-1\}$ and $j\in J_k.$\\

\noindent \textbf{Proof sketch.} This value space does not belong to $\ftime$ because for any algorithm and horizon time $k$, there is a sequence of length $k$ of elements in $I_k$ with $y_u = n_k+4(u-1)+2b_u+c_u$ for $1\leq u\leq k$ and $b_u,c_u\in\{0,1\}$, such that the algorithm incurs an average excess loss $\frac{1}{4}$ per iteration compared to some fixed element of $J_k$. To find such a sequence, we sample randomly and independently Bernoulli variables $b_u,c_u\sim\Bcal(\frac{1}{2})$. In hindsight, the best predictor of the sequence is $n_k+4k+j$, where $j=b_1\cdots b_k$ in binary representation. However, the algorithm only observes these bits in an online fashion: at time $t$ it incurs an excess loss cost if it guesses an element of $I_k$ because it has probability at most $\frac{1}{4}$ of finding $y_t$. And if it predicts an element of $J_k$, it cannot know in advance the correct $t$-th bit to choose in their binary representation.

We then proceed to show that for this space $\solar=\cs \subsetneq\soul$. To do so, we show that for processes $\Xbb\notin\cs$ there exists a sequence of disjoint measurable sets $\{B_p\}_{p\geq 1}$ and increasing times $(t_p)_{p\geq 1}$ and $\epsilon>0$ such that with non-zero probability, 
\begin{equation*}
    \forall p\geq 1,\quad \Xbb_{\leq t_{p-1}}\cap B_p=\emptyset \text{ and } \exists t_{p-1}<t\leq t_p: \frac{1}{t}\sum_{t'=1}^t \1_{B_p}(X_{t'})\geq \epsilon.
\end{equation*}
On this event, an online algorithm does not receive any information for instances in $B_p$ before time $t_{p-1}$. We then construct responses by $(t_{p-1},t_p]$. During this period and for contexts in $B_p$, we choose the same difficult-to-predict sequence of values as above for $k=t_p-t_{p-1}$. On the other hand, because the sets $B_p$ are disjoint, there exists a measurable function $f^*$ that selects the best action in hindsight for each set $B_p$. Intuitively, within horizon $t_p$, the algorithm cannot gather enough information to achieve lower average excess error than $\frac{\epsilon}{4}$ compared to $f^*$, which shows that it is not universally consistent.\\

Although learning beyond $\cs$ is impossible in this case, there still exists an optimistically universal learning rule for adversarial responses. Indeed, the main result of \cite{hanneke2022bayes} shows that for any bounded value space, there exists a learning rule which is consistent under all $\cs$ processes for arbitrary responses (when $\ell$ is a metric, i.e., $\alpha=1$). 

\begin{theorem}[\cite{hanneke2022bayes}]
\label{thm:hanneke_2022}
Suppose that $(\Ycal,\ell)$ is metric and $\ell$ is bounded. Then, there exists an online learning rule $f_\cdot$ which is universally consistent for arbitrary responses under any process $\Xbb\in\cs$, i.e., such that for any stochastic process $(\Xbb,\Ybb)$ on $(\Xcal,\Ycal)$ with $\Xbb\in\cs$, then for any measurable function $f:\Xcal\to\Ycal$, we have $\Lcal_{(\Xbb,\Ybb)}(f_\cdot,f)\leq 0,\quad (a.s.)$.
\end{theorem}

The proof of this theorem given in \cite{hanneke2022bayes} extends to adversarial responses. However, we defer the argument because we will later prove \cref{thm:CS_regression_unbounded} which also holds for any loss $\ell=\rho_\Ycal^\alpha$ for $\alpha\geq 1$ and unbounded losses in \cref{sec:unbounded_loss_moment_constraint}. This shows that for any separable metric space $(\Xcal,\rho_\Xcal)$, there exists a metric value space for which the learning rule proposed in \cite{hanneke2022bayes} was already optimistically universal.

\subsection{Adversarial regression for classification with a countable number of classes}
\label{subsec:countable_classification}

Although we showed in the last section that adversarial regression under all $\smv$ processes is not achievable for some non-totally-bounded spaces, we will show that there exist non-totally-bounded value spaces for which we can recover $\solar=\smv$. Precisely, we consider the case of classification with countable number of classes $(\Nbb,\ell_{01})$, with $0-1$ loss $\ell_{01}(i,j)=\1_{i\neq j}$. The goal of this section is to prove that in this case, we can learn arbitrary responses under any $\soul$ process. The main difficulty with non-totally-bounded classification is that we cannot apply traditional online learning tools because $\epsilon-$nets may be infinite. Hence, we first show a result that allows us to perform online learning with an infinite number of experts in the context of countable classification.

\begin{lemma}
\label{lemma:bandits_Nbb}
Let $t_0\geq 1$. There exists an online learning rule $f_\cdot$ such that for any sequence $\mb y:=(y_i)_{i\geq 1}^T$ of values in $\Nbb$, we have that for $T\geq t_0$
\begin{equation*}
    \sum_{t=1}^T \Ebb[\ell_{01}(f_t({\mb y}_{\leq t-1}), y_t)] \leq \min_{y\in \Nbb} \sum_{t=1}^T \ell_{01}(y, y_t) +  1 +\ln 2\sqrt{\frac{t_0}{2\ln t_0}} + \sqrt{\frac{\ln t_0}{2t_0}} (t_0 + T),
\end{equation*}
and with probability $1-\delta$,
\begin{equation*}
    \sum_{t=1}^T \Ebb[\1_{f_t({\mb y}_{\leq t-1})= y_t}] \geq \max_{y\in \Nbb} \sum_{t=1}^T \1_{y= y_t} -  1- \ln 2\sqrt{\frac{t_0}{2\ln t_0}} - \sqrt{\frac{\ln t_0}{2t_0}} (t_0 + T) - \sqrt{2T\ln \frac{1}{\delta}}.
\end{equation*}
\end{lemma}
\noindent \textbf{Proof sketch.} We adapt the classical Hedge algorithm, which in its standard form can only ensure sublinear regret compared to a fixed set of values. Instead, we only consider a small subset of candidate values that is enlarged occasionally with previously observed values $y\in \Ybb_{\leq t}$. This formalizes the intuition that even though there are a priori an infinite number of candidate values ($\Nbb$), it is reasonable to only focus on values with high frequency in the observed sequence $\Ybb_{\leq t}$: if the next value $y_{t+1}$ is not in this set, the algorithm incurs a loss $1$, which would also be incurred by the best fixed predictor until time $t+1$ in hindsight.\\

We can therefore adapt the learning rules $f^\epsilon_\cdot$ from \cref{sec:totally_bounded_value_spaces} by replacing the Hedge algorithm with the algorithm from \cref{lemma:bandits_Nbb}. Further adapting parameters, we obtain our main result for countable classification.

\begin{theorem}
\label{thm:countable_classification}
Let $(\Xcal,\Bcal)$ be a separable Borel metrizable space. There exists an online learning rule $f_\cdot$ which is universally consistent for adversarial responses under any process $\Xbb\in\smv$ for countable classification, i.e., such that for any adversarial process $(\Xbb,\Ybb)$ on $(\Xcal,\Nbb)$ with $\Xbb\in\smv$, for any measurable function $f^*:\Xcal\to\Nbb$, we have that $\Lcal_{(\Xbb,\Ybb)}(f_\cdot,f^*)\leq 0,\quad (a.s.).$
\end{theorem}

\subsection{A complete characterization of universal regression on bounded spaces}

The last two Sections \ref{subsec:bad_non_totally_bounded_ex} and \ref{subsec:countable_classification} gave examples of non-totally-bounded value spaces for which we obtain respectively $\solar=\cs$ or $\solar=\smv$. In this section, we prove that there is an underlying alternative, defined by $\ftime$, which enables us to precisely characterize the set $\solar$ of learnable processes for adversarial regression.

When $\ftime$ is satisfied by the value space, similarly to the case of countable classification, we recover $\solar=\smv$ and there exists an optimistically universal rule. The corresponding algorithm follows the same general structure as the learning rule provided in \cref{sec:totally_bounded_value_spaces} for totally-bounded-spaces, however, the learning rules $f^\epsilon_\cdot$ need to be significantly modified. First, the Hedge algorithm should be replaced by the learning rule $g_{\leq t_\epsilon}$ provided by the $\ftime$ property. Second, as the horizon time $t_\epsilon$ of this learning rule is bounded, the clusters of points on which it is applied have to be adapted: we cannot simply use clusters by distance in the graph defined by the $(1+\delta_\epsilon)C1NN$ algorithm. Instead, we construct clusters of smaller size $t_\epsilon$ among these larger graph-based clusters.

More precisely, we take the horizon time $t_\epsilon$ and the learning rule $g^\epsilon_{\leq t_\epsilon}$ satisfying the condition imposed by the assumption on $(\Ycal,\ell)$. Then, let $T_\epsilon = \lceil\frac{t_\epsilon}{\epsilon}\rceil$. Similarly as before, we then define $\delta_\epsilon:=\frac{\epsilon}{2T_\epsilon}$ and let $\phi$ be the representative function from the $(1+\delta_\epsilon)$C1NN learning rule. Then, we introduce the same equivalence relation between times $\stackrel \phi\sim$, which induces clusters of times. We define a sequence of i.i.d. copies $g^{\epsilon,t}_\cdot$ of the learning rule $g^\epsilon_\cdot$ for all $t\geq 1$. This means that the randomness used within these learning rules is i.i.d, and the copy $g^{\epsilon,t}_\cdot$ should be sampled only at time $t$, independently of the past history. Predictions are then made by blocks of size $t_\epsilon$ within the same cluster: at time $t$, let $u_1<\ldots<u_{L_t}<t$ be the elements of the current block. If the block does not contain $t_\epsilon$ elements yet, we use $g^{\epsilon,u_1}_{L_t+1}$ for the prediction at time $t$. Otherwise, we start a new block and use $g^{\epsilon,t}_1$. Hence, letting $\psi(t) = \max \Ccal(t)$ be the last time in the same cluster as $t$ (as defined by $\phi_\epsilon$) and $L_t$ the size of the current block of $t$ without counting $t$, we now define the learning rule $f^\epsilon_\cdot$ such that for any sequence $\mb x$, $\mb y$,
\begin{equation*}
    f_t^\epsilon(\mb x_{\leq t-1},\mb y_{\leq t-1}, x_t) := g^{\epsilon,\psi^{L_t}(t)}_{L_t+1}\left(\{y_{\psi^{L_t+1-u}(t)}\}_{u=1}^{L_t}\right).
\end{equation*}
The complete learning rule is given in Algorithm \ref{alg:modified_f_epsilon}. The learning rules $f^\epsilon_\cdot$ are then combined into a single learning rule as in the original algorithm for totally-bounded spaces, following the same procedure given in Algorithm \ref{alg:optim_rule}. We then show that it is universally consistent under $\smv$ processes using same arguments as for \cref{thm:optimistic_regression_totally_bounded}.

\begin{algorithm}[tb]
\caption{The modified $f_\cdot^\epsilon$ learning rule for value spaces $(\Ycal,\ell)$ satisfying $\ftime$. When initializing a learner $g_\cdot^{\epsilon,t}$ for finite-time mean estimation, its internal randomness is sampled independently from the past.}\label{alg:modified_f_epsilon}
\hrule height\algoheightrule\kern3pt\relax
\KwIn{Historical samples $(X_t,Y_t)_{t<T}$ and new input point $X_T$,\\
\quad \quad \quad \, Learning rule for finite-time mean estimation $g^\epsilon_{\leq t_\epsilon}$, $T_\epsilon=\lceil\frac{t_\epsilon}{\epsilon}\rceil$, $\delta_\epsilon:=\frac{\epsilon}{2T_\epsilon}$.
\quad\quad\quad \, Representatives $\phi_\epsilon(\cdot)$ constructed iteratively within $(1+\delta_\epsilon)$C1NN.}
\KwOut{Predictions $\hat Y_t(\epsilon) = f_t^\epsilon({\mb X}_{<t},{\mb Y}_{<t},X_t)$ for $t\leq T$}

\For{$t=1,\ldots,T$}{
    $\Ccal(t) = \{u<t:u \stackrel {\phi_\epsilon} \sim  t\}$\\
    \lIf{$\Ccal(t)=\emptyset$}{
        $L_t=0$ and initialize learner $g_\cdot^{\epsilon,t}$
    }
    \Else{
        $\psi(t) = \max \Ccal(t)$\\
        \lIf{$L_{\psi(t)}< t_\epsilon-1$}{
            $L_t=L_{\psi(t)}+1$
        }
        \lElse{
            $L_t=0$ and initialize learner $g_\cdot^{\epsilon,t}$
        }
    }
    $\hat Y_t = g^{\epsilon,\psi^{L_t}(t)}_{L_t+1}\left(\{y_{\psi^{L_t+1-u}(t)}\}_{u=1}^{L_t}\right)$
}
\hrule height\algoheightrule\kern3pt\relax
\end{algorithm}

\begin{theorem}
\label{thm:good_value_spaces}
Suppose that $\ell$ is bounded and $(\Ycal,\ell)$ satisfies $\ftime$. Then, $\solar = \smv (=\soul)$ and there exists an optimistically universal learning rule for adversarial regression, i.e., which achieves universal consistency with adversarial responses under any process $\Xbb\in\smv$.
\end{theorem}

We are now interested in value spaces $(\Ycal,\ell)$ which do not satisfy $\ftime$. We will show that in this case, $\solar$ is reduced to the processes $\cs$. We first introduce a second property on value spaces as follows.\\

\noindent\textbf{Property 2:} \textit{For any $\eta>0$, there exists a horizon time $T_\eta \geq 1$ and an online learning rule $g_{\leq \tau}$ where $\tau$ is a random time with $1\leq \tau\leq T_\eta$ such that for any $\mb y:=(y_t)_{t=1}^{T_\eta}$ of values in $\Ycal$ and any value $y\in\Ycal$, we have
\begin{equation*}
    \Ebb\left[\frac{1}{\tau}\sum_{t=1}^{\tau} \left(\ell(g_t({\mb y}_{\leq t-1}), y_t) - \ell(y,y_t)\right)\right] \leq \eta.
\end{equation*}
}

\begin{remark}
The random time $\tau$ may depend on the possible randomness of the learning rule $g_\cdot$, but it does not depend on any of the values $y_1,y_2,\ldots$ on which the learning rule $g_\cdot$ may be tested. Intuitively, the learning rule uses some randomness which is first privately sampled and may be used by $\tau$. This randomness is never explicitly revealed to the adversary choosing the values $\mb y$, but only implicitly through the realizations of the predictions.
\end{remark}

\begin{lemma}
\label{lemma:equivalent_conditions}
Property $\ftime$ is equivalent to Property 2.
\end{lemma}

Using this second property, we can then show that when $\ftime$ is not satisfied, universal consistency outside $\cs$ under adversarial responses is not achievable. In the proof, we only use stochastic processes $(\Xbb,\Ybb)$, hence the same result holds if we only considered universal consistency under arbitrary responses.

\begin{theorem}
\label{thm:bad_value_spaces}
Suppose that $\ell$ is bounded and $(\Ycal,\ell)$ does not satisfy $\ftime$. Then, $\solar = \cs$ and there exists an optimistically universal learning rule for adversarial regression, i.e., which achieves universal consistency with adversarial responses under any process $\Xbb\in\cs$.
\end{theorem}

\noindent \textbf{Proof sketch.} First, from \cref{thm:hanneke_2022} we already have $\cs\subset\solar$. The main difficulty is to prove that one cannot universally learn any process $\Xbb\notin\cs$. To do so, we re-use the property derived in the proof of \cref{thm:negative_optimistic} that for non-$\cs$ processes, one can find a disjoint sequence of sets $\{B_p\}_{p\geq 1}$, an increasing times $(t_p)_{p\geq 1}$ and $\epsilon>0$ such that with non-zero probability for all $p\geq 1$, the process $\Xbb$ never visits $B_p$ before time $t_{p-1}$ and at some point between times $t_{p-1}+1$ and $t_p$, the set $B_p$ has been visited a proportion $\epsilon$ of times. Now $(\Ycal,\ell)$ does not satisfy $\ftime$, hence does not satisfy Property 2 by \cref{lemma:equivalent_conditions} for some constant $\eta>0$. Then, for $p\geq 1$, during period $(t_{p-1},t_p]$, we define the values $\Ybb_{t_{p-1}<\cdot\leq t_p}$ when the instance process visits $B_p$ as a sequence $\mb y_{t_{p-1}<\cdot\leq t_p}$ such that the algorithm has average excess loss at least $\eta$ whenever $\Xbb$ visits $B_p$, compared to a fixed value $y_p^*\in\Ycal$. We note that the randomized version of $\ftime$ given by \cref{lemma:equivalent_conditions} is important because we do not know in advance when, between $t_{p-1}$ and $t_p$, $B_p$ has been visited a fraction $\epsilon$ of times: potentially, this time is random and there is a huge gap (exponential or more) between $t_{p-1}$ and $t_p$. On the constructed stochastic process $\Ybb$, the algorithm does not have vanishing average excess loss compared to the function equal to $y_p^*$ on $B_p$. This proves that no algorithm is universally consistent on $\Xbb$.\\

This completes the proof of \cref{cor:good_value_spaces} and closes our study of universal learning with adversarial responses for bounded value spaces. Notably, there always exists an optimistically universal learning rule, however, this rule highly depends on the value space. \begin{itemize}
    \item If $(\Ycal,\ell)$ satisfies $\ftime$, we can learn all $\smv=\soul$ processes. The proposed learning rule of \cref{thm:good_value_spaces} is \emph{implicit} in general. Indeed, to construct it one first needs to find an online learning rule for mean estimation with finite horizon as described by property $\ftime$, which is then used as a subroutine in the optimistically universal learning rule for adversarial regression. We showed however that for totally-bounded value spaces, this learning rule can be \emph{explicited} using $\epsilon-$nets.
    \item If the value space does not satisfy $\ftime$, we can only learn $\cs$ processes and there is an inherent gap between noiseless online learning and regression. We propose a learning rule in \cref{sec:unbounded_loss_moment_constraint} which is optimistically universal---see \cref{thm:CS_regression_unbounded}. This rule is inspired by the proposed algorithm of \cite{hanneke2022bayes} which is optimistically universal for metric losses $\alpha=1$.
\end{itemize}
These two classes of learning rules use very different techniques. Specifically, under processes $\Xbb\in\cs$, \cite{hanneke2021learning} showed that there exists a countable set $\Fcal$ of measurable functions $f:\Xcal\to\Ycal$ which is ``dense'' within the space of all measurable functions along the realizations $f(X_t)$. We refer to \cref{sec:unbounded_loss_moment_constraint} for a precise description of this density notion. Hence, under process $\Xbb$, we can approximate $f^*$ by functions in $\Fcal$ with arbitrary long-run average precision. However, such property is impossible to obtain for any process $\Xbb\in\smv\setminus \cs$: no process $\Xbb\notin \cs$ admits a ``dense'' countable sequence of measurable functions. Thus, to learn processes $\smv$ for value spaces satisfying $\ftime$, a fundamentally different learning rule than that proposed by \cite{hanneke2021learning} or \cite{hanneke2022bayes} was needed.

\section{Adversarial universal learning for unbounded losses}
\label{sec:mean_estimation}

We now turn to the case of unbounded losses, i.e., value spaces $(\Ycal,\ell)$ with $\bar \ell=\infty$. In this section, we consider universal learning without empirical integrability constraints, for which we introduced the notation $\solaru$ as the set of processes that admit universal learning (we recall that for bounded losses such distinction was unnecessary). In this case, and for more general near-metrics, \cite{blanchard2022optimistic} showed that $\soul=\fs$. In other terms, for unbounded losses, the learnable processes in the noiseless setting necessarily visit a finite number of distinct instance points of $\Xcal$ almost surely. Thus, universal learning on unbounded value spaces is very restrictive and in particular, $\solaru\subset\fs$. We will show that either $\solaru=\fs$ or $\solaru=\emptyset$.

\subsection{Adversarial regression for metric losses}
\label{subsec:metric_mean_estimation}

In this section, we focus on metric losses $\ell$, i.e., $\alpha=1$. In this case, we show that we always have the equality $\solaru=\fs$ and that we can provide an optimistically universal learning rule. To do so, we first consider the fundamental estimation problem where one observes values $\Ybb$ from a general separable metric value space and aims to sequentially predict a value $\hat Y_t$ in order to minimize the long-run average loss. We refer to this problem as the mean estimation problem, which is equivalent to regression for the instance space $\Xcal=\{0\}$. For instance, in the specific case of i.i.d. processes $\Ybb$, mean estimation is exactly the problem of Fr\'echet mean estimation for distributions on $\Ycal$. We show that even for adversarial processes $\Ybb$, we can achieve sublinear regret compared to the best single value prediction, even for unbounded value spaces $(\Ycal,\ell)$.

If the space were finite, then we could use traditional Hedge algorithms \cite{cesa2006prediction}. Instead, given a separable value space, we have access to a dense countable sequence of values. We then select the best prediction among this dense sequence by introducing the values of the sequence one at a time, similarly to the argument we used in \cref{lemma:concatenation_predictors}. The learning rule for mean estimation is described in Algorithm \ref{alg:mean_estimation}.

\begin{algorithm}[tb]
\caption{The mean estimation algorithm.}\label{alg:mean_estimation}
\hrule height\algoheightrule\kern3pt\relax
\KwIn{Historical samples $(Y_t)_{t<T}$}
\KwOut{Predictions $\hat Y_t$ for $t\leq T$}
$(y^i)_{i\geq 0}$ dense sequence in $\Ycal$\\
$I_t:=\{i\leq \ln t:\ell(y^0,y^i) \leq \ln t \},\eta_t:=\frac{1}{4\sqrt t},t\geq 1;\quad t_i=\lceil\max(e^i,e^{\ell(y^0,y^i)})\rceil,i\geq 0$\\
$w_{0,0}:=1,\quad \hat Y_1=y^0$ \tcp*[f]{Initialisation}\\
\For{$t=2,\ldots, T$}{
    $L_{t-1,i} = \sum_{s=t_i}^{t-1} \ell(y^i,Y_s),\quad \hat L_{t-1,i} = \sum_{s=t_i}^{t-1}\hat \ell_s,\quad i\in I_t$\\
    $w_{t-1,i} := \exp(\eta_t(\hat L_{t-1,i}-L_{t-1,i})),\quad i\in I_t$\\
    $p_t(i) = \frac{w_{t-1,i}}{\sum_{j\in I_t} w_{t-1,j}},\quad i\in I_t$\\
    $\hat Y_t \sim p_t(\cdot)$ \tcp*[f]{Prediction}\\
    $\hat \ell_t:=\frac{\sum_{j\in I_t} w_{t-1,j}\ell(y^j,Y_t)}{\sum_{j\in I_t} w_{t-1,j}}$
}
\hrule height\algoheightrule\kern3pt\relax
\end{algorithm}

\ThmMeanEstimation*

\begin{remark}
    The above result guarantees that on the same event of probability one, the proposed learning rule achieves sublinear regret compared to any fixed value prediction. This was not the case for universal regression where, instead, for every fixed measurable function $f:\Xcal\to\Ycal$, with probability one our learning rules achieved sublinear regret. This stems essentially from the fact that there exists a dense countable set of values $\Ycal$, but in general, there does not exist a countable set of measurable functions which are dense within all measurable functions in infinity norm.
\end{remark}

We now return to the general regression problem on unbounded spaces. A simple learning rule would be to run in parallel the learning rule $g_x$ for mean estimation on each distinct observed $x\in\Xcal$, i.e., on the sub-process $\Ybb_{\{t:X_t=x\}}$. As a consequence of \cref{thm:mean_estimation} we can show that this learning rule is universally consistent on $\fs$ processes.

\begin{corollary}
\label{cor:universal_regression_unbounded}
Suppose that $(\Ycal,\ell)$ is an unbounded metric space. Then, $\solaru = \fs(=\soul)$ and there exists an optimistically universal learning rule for adversarial regression, i.e., which achieves universal consistency with adversarial responses under any process $\Xbb\in\fs$.
\end{corollary}

\subsection{Negative result for real-valued adversarial regression with loss $\ell=|\cdot|^\alpha$ with $\alpha>1$}

Unfortunately, one cannot extend \cref{cor:universal_regression_unbounded} to losses that are powers of metrics in general. Even in the classical setting of real-valued regression $\Ycal=\Rbb$ with Euclidean norm, we show that adversarial regression with any loss $\ell=|\cdot|^\alpha$ for $\alpha>1$ is not achievable, i.e., $\solaru=\emptyset$.

\begin{theorem}
\label{thm:empty_solar}
Let $\alpha>1$. For the Euclidean value space $(\Rbb,|\cdot|)$ and loss $\ell=|\cdot|^\alpha$ we obtain $\solaru=\emptyset$. In particular, there does not exist a consistent learning rule for mean estimation on $\Rbb$ with squared loss for adversarial responses.
\end{theorem}

\noindent \textbf{Proof sketch.} The reason why mean estimation with adversarial responses is impossible for $\alpha>1$ but possible for $\alpha=1$ is that for $\alpha>1$, predicting a value off by $1$ unit of the best value in hindsight can yield unbounded excess loss for that specific prediction. In particular, we consider a sequence of values of the form $Y^{\mb b}_t = M_t b_t$ where $(M_t)_{t\geq 1}$ is a fixed sequence growing super-exponentially in $t$, and $\mb b =(b_t)$ is an i.i.d. Rademacher random variables in $\{\pm 1\}$. The sequence $(M_t)_{t\geq 1}$ is constructed so that if the prediction $\hat Y_t$ and true value $Y_t$ have different signs $\hat Y_t\cdot Y_t\leq 0$, the excess loss of the algorithm compared to the value $sign(Y^{\mb b}_t)=sign(b_t)$ is (super-)linear in $t$. Because the algorithm cannot know in advance the sign of $b_t$, there is a realization in which it makes an infinite number of mistakes and as a result has non-zero long-term excess loss compared to the value $1$ or $-1$.\\

The above of this result also shows that the same negative result holds more generally for unbounded metric value spaces which have some ``symmetry''. The main ingredients for this negative result were having a point from which there exist arbitrary far values from symmetric directions. In particular, this holds for a discretized value space $(\Nbb,|\cdot|)$ with Euclidean metric, and any Euclidean space $\Rbb^d$ with $d\geq 1$.

\subsection{An alternative for adversarial regression with unbounded losses}

In the two previous sections, we gave examples of losses for which $\solaru=\emptyset$ or $\solaru=\fs$. The following simple result is that this is the only alternative and that $\solaru=\fs$ is equivalent to achieving consistency for mean estimation with adversarial responses.

\begin{proposition}
\label{prop:alternative_unbounded}
Let $(\Ycal,\rho_\Ycal)$ be a separable metric value space. Suppose that there exists an online learning rule $g_\cdot$ which is consistent for mean estimation with adversarial responses for the loss $\ell=\rho_\Ycal^\alpha$, where $\alpha\geq 1$, i.e., for any adversarial process $\Ybb$ on $(\Ycal,\ell)$, we have for any $y^*\in\Ycal$,
\begin{equation*}
    \limsup \frac{1}{T}\sum_{t=1}^T \left(\ell(f_t(\Ybb_{\leq t-1}),Y_t) - \ell(y^*,Y_t)\right) \leq 0,\quad (a.s),
\end{equation*}
then $\solaru=\fs$ and there exists an optimistically universal learning rule for adversarial regression. Otherwise, $\solaru=\emptyset$.
\end{proposition}

\begin{remark}
    There exists separable metric value spaces $(\Ycal,\rho_\Ycal)$ for which powers of metrics losses still yield $\solaru=\fs$. For instance, consider $(\Ycal,\rho_\Ycal)=(\Rbb,\sqrt{|\cdot|_2})$, where $|\cdot|_2$ denotes the Euclidean metric. One can check that this defines a metric on $\Ycal$ and for any loss $\ell=\rho_\Ycal^\alpha$ with $\alpha\leq 2$, we have $\solaru=\fs$. However, for $\alpha>2$, $\solaru=\emptyset$.
\end{remark}

\section{Adversarial universal learning with moment constraint}
\label{sec:unbounded_loss_moment_constraint}

In the previous section, we showed that learnable processes for adversarial regression are only in $\fs$, i.e., visit a finite number of instance points. This shows that universal learning without restrictions on the adversarial responses $\Ybb$ is extremely restrictive. For instance, it does not contain i.i.d. processes. A natural question is whether adding mild constraints on the process $\Ybb$ would allow recovering the same results for unbounded losses as for bounded losses from \cref{sec:totally_bounded_value_spaces,sec:alternative}. This question also arises in noiseless regression since the set of learnable processes is reduced from $\soul=\smv$ for bounded losses to $\soul=\fs$ for unbounded losses. Hence, \cite{blanchard2022optimistic} posed as question whether having finite long-run empirical first-order moments would be sufficient to recover learnability in $\smv$. Precisely, they introduced the following constraint on noiseless processes $\Ybb=f^*(\Xbb)$: there exists $y_0\in\Ycal$ with
\begin{equation*}
    \limsup_{T\to\infty} \frac{1}{T} \sum_{t=1}^T \ell(y_0,f^*(X_t)) <\infty\quad (a.s.).
\end{equation*}
The question now becomes whether there exists an online learning rule which would be consistent under all $\Xbb\in\smv$ processes for any noiseless responses $\Ybb=f^*(\Xbb)$ with $f^*$ satisfying the above first-moment condition. We show that such an objective is not achievable whenever $\Xcal$ is infinite---if $\Xcal$ is finite, any process $\Xbb$ on $\Xcal$ is automatically $\fs$ and hence learnable in a noiseless or adversarial setting. In fact, under this first-order moment condition, we show the stronger statement that learning under all processes $\Xbb$ which admit pointwise convergent relative frequencies ($\crf$) is impossible even in this noiseless setting.\\

\noindent\textbf{Condition $\crf$:} For any measurable set $A\in\Bcal$, $    \lim_{T\to\infty} \frac{1}{T}\sum_{t=1}^T\1_A(X_t)$ exists almost surely.\\

\cite{hanneke2021learning} showed that $\crf\subset\cs$. In particular, $\crf\subset \smv$. We show the following negative result on learning under $\crf$ processes for noiseless regression under first-order moment constraint, which holds for unbounded near-metric spaces $(\Ycal,\ell)$.

\begin{theorem}
\label{thm:negative_first_order_moment}
Suppose that $\Xcal$ is infinite and that $(\Ycal,\ell)$ is an unbounded separable near-metric space. There does not exist an online learning rule which would be consistent under all processes $\Xbb\in\crf$ for all measurable target functions $f^*:\Xcal\to\Ycal$ such that there exists $y_0\in \Ycal$ with
\begin{equation*}
    \limsup_{T\to\infty}\frac{1}{T}\sum_{t=1}^T\ell(y_0,f^*(X_t))<\infty \quad (a.s.).
\end{equation*}
\end{theorem}

\noindent\textbf{Proof sketch.} We consider a sequence of values $(y_k)_{k\geq 0}$ such that $\ell(y_0,y_k)$ diverges as $k\to\infty$, then let $(t_k)_{k\geq 1}$ be a sequence of times such that $t_k\approx \sum_{k'\leq k}\ell(y_0,y_k)$. Next, let $(x_k)_{k\geq 0}$ be a sequence of distinct points. We construct a process $\Xbb$ such that $X_t=x_0$ except at sparse times $(t_k)_{k\geq 1}$ for which $X_{t_k}=x_k$. Because $t_k$ has a super-linear growth, $\Xbb$ visits a sublinear number of distinct points and we can show that it satisfies the $\crf$ property. Now for a random binary sequence $\mb b=(b_k)_{kk\geq 1}$ we consider the function $f_{\mb b}^*$ which is equal to $y_0$ except at points $x_k$ for $k\geq 1$ where $f_{\mb b}^*(x_k)=y_0\1[b_k=0] + y_k \1[b_k=1]$. With these classes of functions, the algorithm cannot know in advance at time $t_k$ whether to predict $y_0$ or $y_k$ and incurs a loss $\Ocal( \ell(y_0,y_k) )$ in average as a result. Therefore, at time $t_k$, a total loss $\Ocal( \sum_{k'\leq k}\ell(y_0,y_k))= \Ocal(t_k)$ is incurred compared to $f_{\mb b}^*$. On the other hand, by the construction of the sequence $(t_k)_{k\geq 1}$, $\frac{1}{T}\sum_{t=1}^T\ell(y_0,f_{\mb b}^*(X_t)) \leq  \frac{1}{T}\sum_{t_k\leq T}\ell(y_0,y_k)$ stays bounded. Thus the learning rule is not consistent under all target functions satisfying the specified moment constraint.\\

\cref{thm:negative_first_order_moment} answers negatively to the question posed in \cite{blanchard2022optimistic}. A natural question is whether another meaningful constraint on responses can be applied to obtain positive results under large classes of processes on $\Xcal$. To this means, we introduced the slightly stronger \emph{empirical integrability} condition. We recall that an (adversarial) process $\Ybb$ is \emph{empirically integrable} if and only if there exists $y_0\in\Ycal$ such that for any $\epsilon>0$, almost surely there exists $M\geq 0$ with
\begin{equation*}
    \limsup_{T\to\infty}\frac{1}{T}\sum_{t=1}^T\ell(y_0,Y_t)\1_{\ell(y_0,Y_t)\geq M}\leq \epsilon.
\end{equation*}
Note that the threshold $M$ may be \emph{dependent} on the adversarial process $\Ybb$, but the guarantee should hold for any choice of predictions (in the case of adaptive adversaries). This is essentially the mildest condition on the sequence $\Ybb$ for which we can still obtain results. For example, if the loss is bounded, this constraint is automatically satisfied using $M>\bar\ell$. More importantly, note that any process $\Ybb$ which has bounded higher-than-first moments, i.e., such that there exists $p>1$ and $y_0\in\Ycal$ such that $    \limsup_{T\to\infty} \frac{1}{T}\sum_{t=1}^T \ell^p(y_0,Y_t) <\infty,\quad (a.s.)$, is empirically integrable. Further, for stationary processes $\Ybb$, having bounded first moment $\Ebb[\ell(y_0,Y_1)]<\infty$ is exactly being empirically integrable. Indeed, by the strong law of large numbers, almost surely $\limsup_{T\to\infty}\frac{1}{T}\sum_{t=1}^T\ell(y_0,Y_t)\1_{\ell(y_0,Y_t)\geq M} = \Ebb[\ell(y_0,Y_1)\1_{\ell(y_0,Y_1)\geq M}]$. Therefore, empirical integrability is a direct consequence of the dominated convergence theorem.

\begin{lemma}
\label{lemma:link_to_first_moment}
Let $\Ybb$ an stationary process on $\Ycal$ which has bounded first moment, i.e., there exists $y_0\in\Ycal$ such that $\Ebb[\ell(y_0,Y_1)]<\infty$. Then, $\Ybb$ is empirically integrable.
\end{lemma}
\begin{proof}
Let $\Ybb$ an stationary process and $y_0\in\Ycal$ with $\Ebb[\ell(y_0,Y_1)]<\infty$. Then, by the dominated convergence theorem we have $\Ebb[\ell(y_0,Y_1)\1_{\ell(y_0,Y_1)\geq M}]\to 0$ as $M\to\infty$. Hence, for $\epsilon>0$, there exists $M_\epsilon$ such that $\Ebb[\ell(y_0,Y_1)\1_{\ell(y_0,Y_1)\geq M}]\leq \epsilon$. Then, the sequence $(\ell(y_0,Y_t)\1_{\ell(y_0,Y_t)\geq M})_t$ is still stationary. hence, by the law of large numbers, almost surely,
\begin{equation*}
    \lim_{T\to\infty}\frac{1}{T}\sum_{t=1}^T  \ell(y_0,Y_t)\1_{\ell(y_0,Y_t)\geq M_\epsilon} = \Ebb[\ell(y_0,Y_1)\1_{\ell(y_0,Y_1)\geq M_\epsilon}] \leq \epsilon.
\end{equation*}
This ends the proof that $\Ybb$ is empirically integrable.
\end{proof}

The goal of this section is to show that under this moment constraint, we can recover all results from \cite{blanchard2022universal}, \cite{hanneke2022bayes} and this work in Sections \ref{sec:totally_bounded_value_spaces} and \ref{sec:alternative}, even for unbounded value spaces, leading up to Theorems \ref{thm:SOUL_regression_unbounded} and \ref{thm:CS_regression_unbounded}. We will use the following simple equivalent formulation for empirical integrability.

\begin{lemma}
\label{lemma:empirically_integrable}
A process $\Ybb$ is empirically integrable if and only if there exists $y_0\in\Ycal$ such that almost surely, for any $\epsilon>0$ there exists $M>0$ with
\begin{equation*}
    \limsup_{T\to\infty}\frac{1}{T}\sum_{t=1}^T\ell(y_0,Y_t)\1_{\ell(y_0,Y_t)\geq M}\leq \epsilon.
\end{equation*}
\end{lemma}

\noindent\textbf{General strategy.} First, the empirical integrability condition holds for some $y_0\in\Ycal$ if and only if it holds for all $y_0\in\Ycal$. Thus, we can fix $y_0\in\Ycal$ independently of the instance or value process. Next, we define the restriction function $\phi_M:\Ycal\to\Ycal$ such that $\phi_M(y)=y$ if $\ell(y_0,y)<M$ and $\phi_M(y)=y_0$ otherwise. This function has values in the bounded set $B_\ell(y_0,M)$. Thus, we can apply our learning rules for the bounded loss case to learn the restricted values $\Ybb^M=(\phi_M(Y_t))_{t\geq 1}$. If we use these predictions to learn $\Ybb$, the excess loss compared to a fixed function mostly results from the restriction $\limsup_{T\to\infty}\frac{1}{T}\sum_{t=1}^T\ell(Y_t,\phi_M(Y_t)) = \limsup_{T\to\infty} \frac{1}{T} \sum_{t=1}^T \ell(y_0,Y_t)\1_{\ell(y_0,Y_t)\geq M}$. This excess can then be bounded with the empirical integrability condition at $y_0$. We then combine the resulting predictors for $M\geq 1$ using \cref{lemma:concatenation_predictors}. While this general strategy allows to use learning rules for the bounded loss case as subroutine to solve the unbounded loss case with empirical integrability constraint, we can adapt it to each case to simplify the algorithms.

\subsection{Noiseless universal learning with moment condition}

We first apply this strategy to the noiseless case. The main result from \cite{blanchard2022universal} showed that the 2C1NN learning rule achieves universal consistency on all $\smv$ processes for bounded value spaces. Instead of using the 2C1NN learning rule as subroutine as described in the strategy above, we show that we can readily use 2C1NN for empirically integrable noiseless responses in unbounded value spaces, as stated in \cref{thm:noiseless_unbounded}.

To prove this result, we first observe that 2C1NN trained on the responses $\Ybb=(f^*(X_t))_{t\geq 1}$ or the restricted responses $(\phi_M\circ f^*(X_t))_{t\geq 1}$ gives the same prediction at time $t$ provided that the representative $\phi(t)$ satisfied $\ell(y_0,Y_{\phi(t)})<M$. By construction of the 2C1NN learning rule, points can be used as representatives at most twice. Hence, up to a factor 2, times when the predictions on unrestricted and restricted responses differ, can be associated with times when $\ell(y_0,Y_t)\geq M$. As a result, we show that the empirical integrability condition can be applied to bound the excess loss resulting from the difference between unrestricted and restricted responses.

\subsection{Adversarial regression with moment condition under $\cs$ processes}

We now turn to adversarial regression under $\cs$ processes. \cite{hanneke2022bayes} showed that regression for arbitrary responses under all $\cs$ processes is achievable in bounded value spaces. We generalize this result to unbounded losses and to adversarial responses with empirical integrability constraint using the general strategy. In particular, our learning rule is also optimistically universal for adversarial regression for all bounded value spaces which do not satisfy $\ftime$. Now consider the general case and suppose that there exists a ball $B_\ell(y,r)$ which does not satisfy $\ftime$, \cref{thm:bad_value_spaces} shows that universal learning for values falling in $B_\ell(y,r)$ cannot be achieved for processes $\Xbb\notin \cs$. Now because $B_\ell(y,r)$ is bounded, responses restricted to this set satisfy the empirical integrability constraint. In particular, this shows that the condition $\cs$ is also necessary for universal learning with adversarial responses with empirical integrability. Altogether, this proves \cref{thm:CS_regression_unbounded}.

This generalizes the main results from \cite{hanneke2022bayes} to unbounded non-metric losses and from \cite{tsir2022metric} to non-metric losses, arbitrary responses and $\cs$ instance processes $\Xbb$. Indeed, they consider bounded first moment conditions on i.i.d. responses, which are empirically integrable by \cref{lemma:link_to_first_moment}. Further, as a direct consequence of \cref{thm:CS_regression_unbounded} and \cref{lemma:link_to_first_moment}, we can significantly relax the conditions for universal consistency on stationary ergodic processes found in the literature. Precisely, \cite{gyorfi:07} showed that for regression with squared loss, under the assumption $\Ebb[Y_1^4]<\infty$, consistency on stationary ergodic processes is possible. We can relax this result to bounded second moments, matching the standard results for i.i.d. processes.

\begin{corollary}
    Let $(\Ycal,\ell)=(\Rbb,|\cdot|^2)$. The learning rule of \cref{thm:CS_regression_unbounded} is  consistent on any stationary ergodic process $(X_t,Y_t)_{t\geq 1}$ with $\Ebb[Y_1^2]<\infty$.
\end{corollary}

\subsection{Adversarial regression with moment condition under $\smv$ processes}

Last, we generalize our result \cref{thm:good_value_spaces} for value spaces satisfying $\ftime$, to unbounded value spaces, with the same moment condition on responses using the general strategy. In order to apply \cref{thm:good_value_spaces} to bounded balls of the value space, we now ask that all balls $B_\ell(y,r)$ in the value space $(\Ycal,\ell)$ satisfy $\ftime$. This proves \cref{thm:SOUL_regression_unbounded}.

Theorems \ref{thm:CS_regression_unbounded} and \ref{thm:SOUL_regression_unbounded} completely characterize learnability for adversarial regression with moment condition. Namely, if the value space $(\Ycal,\ell)$ is such that any bounded ball satisfies $\ftime$ (resp. there exists a ball $B_\ell(y,r)$ that disproves $\ftime$), \cref{thm:SOUL_regression_unbounded} (resp. \ref{thm:CS_regression_unbounded}) gives an optimistic learning rule which achieves consistency under all processes in $\smv$ (resp. $\cs$). This ends our analysis of adversarial regression for unbounded value spaces.

\section{Open research directions}
\label{sec:conclusion}

In this work, we provided a characterization of learnability for universal learning in the regression setting, for a class of losses satisfying specific relaxed triangle inequality identities, which contains powers of metrics $\ell=\rho_\Ycal^\alpha$ for $\alpha\geq 1$. A natural question would be whether one can generalize these results to larger classes of losses, e.g. non-symmetric losses which may appear in classical machine learning problems.

The present work could also have some implications for adversarial contextual bandits. Specifically, one may consider the case of a learner who receives partial information on the rewards/losses as opposed to the traditional regression setting where the response is completely revealed at each iteration. In the latter case, the learner can for instance compute the loss of \emph{all} values with respect to the response realization. On the other hand, in the contextual bandits framework, the reward/loss is revealed \emph{only} for the pulled arm---or equivalently the prediction of the learner. In these partial information settings, exploration then becomes necessary. The authors are investigating whether the results presented in this work could have consequences in these related domains.

\paragraph{Acknowledgements.}The authors are grateful to Prof. Steve Hanneke for enlightening discussions. This work is being partly funded by ONR grant N00014-18-1-2122.

\printbibliography

\newpage

\tableofcontents

\begin{appendix}

\section{Identities on the loss function}

We recall the following known identities, which we will use to analyze the loss $\ell=\rho_\Ycal^\alpha$.

\begin{lemma}
\label{lemma:loss_identity}
Let $\alpha\geq 1$. Then, $(a+b)^\alpha \leq 2^{\alpha-1}(a^\alpha+b^\alpha)$ for all $a,b\geq 0$. Let $0< \epsilon\leq 1$ and $\alpha\geq 1$. There exists some constant $c_\epsilon^\alpha>0$ such that $(a+b)^\alpha \leq (1+\epsilon)a^\alpha + c_\epsilon^\alpha b^\alpha$ for all $a,b\geq 0$, and $c_\epsilon^\alpha\leq \left(\frac{4\alpha}{\epsilon}\right)^\alpha$.
\end{lemma}

\begin{proof}
The first identity is classical. A proof of the second one can be found for example in \cite{evans2020strong} (Lemma 2.3) where they obtain $
    c_\epsilon^\alpha = \left(1+\frac{1}{(1+\epsilon)^{1/\alpha}-1}\right)^\alpha\leq \left(\frac{4\alpha}{\epsilon}\right)^\alpha.$
\end{proof}

\section{Proofs of \cref{sec:totally_bounded_value_spaces}}\label{appA}

\subsection{Proof of \cref{thm:1+deltaC1NN_optimistic}}
\label{subsec:1+deltaC1NN}

In this section, we prove that for any $\delta>0$, the $(1+\delta)$C1NN learning rule is optimistically universal for the noiseless setting. The proof follows the same structure as the proof of the main result in \cite{blanchard2022universal} which shows that 2C1NN is optimistically universal. We first focus on the binary classification setting and show that the learning rule $(1+\delta)$C1NN is consistent on functions representing open balls.

\begin{proposition}
\label{prop:consistent_ball_borel}
Fix $0<\delta\leq 1$. Let $(\Xcal,\Bcal)$ be a separable Borel space constructed from the metric $\rho_\Xcal$. We consider the binary classification setting $\Ycal =\{0,1\}$ and the $\ell_{01}$ binary loss. For any input process $\Xbb\in \smv$, for any $x\in \Xcal$, and $r>0$, the learning rule $(1+\delta)$C1NN is consistent for the target function $f^*= \1_{B_{\rho_\Xcal}(x,r)}$.
\end{proposition}

\begin{proof}
We fix $\bar x\in \Xcal$, $r>0$ and $f^* = \1_{ B(\bar x,r)}$. We reason by the contrapositive and suppose that $(1+\delta)$C1NN is not consistent on $f^*$. Then, $\eta:=\mathbb P(\Lcal_\Xbb ((1+\delta)C1NN,f^*)>0)>0$. Therefore, there exists $0<\epsilon \leq 1$ such that $\mathbb P(\Lcal_\Xbb ((1+\delta)C1NN,f^*)> \epsilon)>\frac{\eta}{2}$.
Denote by $\Acal:=\{\Lcal_\Xbb ((1+\delta)C1NN,f^*)>\epsilon\}.
$ this event of probability at least $\frac{\eta}{2}$.  Because $\Xcal$ is separable, let $(x^i)_{i\geq 1}$  a dense sequence of $\Xcal$. We consider the same partition $(P_i)_{i\geq 1}$ of $B(\bar x,r)$ and the partition $(A_i)_{i\geq 0}$ of $\Xcal$ as in the original proof of \cite{blanchard2022universal}, but with the constant $c_\epsilon:=\frac{1}{2\cdot 2^{2^8/(\epsilon\delta)}}$ and changing the construction of the sequence $(n_l)_{l\geq 1}$ so that for all $l\geq 1$
\begin{equation*}
    \Pbb\left[\forall n\geq n_l,\;|\{i,\; P_i(\tau_l)\cap \Xbb_{< n}\neq \emptyset \} | \leq \frac{\epsilon\delta}{2^{10}} n\right]\geq 1- \frac{\delta}{2\cdot 2^{l+2}}\quad \text{ and } \quad n_{l+1} \geq \frac{2^9}{\epsilon\delta}n_l.
\end{equation*}
Last, consider the product partition of $(P_i)_{i\geq 1}$ and $(A_i)_{i\geq 0}$ which we denote $\Qcal$. Similarly, we define the same events $\Ecal_l,\Fcal_l$ for $l\geq 1$. We aim to show that with nonzero probability, $\Xbb$ does not visit a sublinear number of sets of $\Qcal$.

We now denote by $(t_k)_{k\geq 1}$ the increasing sequence of all (random) times when $(1+\delta)$C1NN makes an error in the prediction of $f^*(X_t)$. Because the event $\Acal$ is satisfied, $\Lcal_{\mb x} ((1+\delta)C1NN,f^*)>\epsilon$, we can construct an increasing sequence of indices $(k_l)_{l\geq 1}$ such that $t_{k_l}<\frac{2k_l}{\epsilon}$. For any $t\geq 2$, we will denote by $\phi(t)$ the (random) index of the representative chosen by the $(1+\delta)$C1NN learning rule. Now let $l\geq 1$. Consider the tree $\Gcal$ where nodes are times $\Tcal:=\{t\leq t_{k_l}\}$ within horizon $t_{k_l}$, where the parent relations are given by $(t,\phi(t))$ for $t\in \Tcal\setminus\{1\}$. In other words, we construct the tree in which the parent of each new input is its representative. Note that by construction of the $(1+\delta)$C1NN learning rule, each node has at most $2$ children.

\subsubsection{Step 1}In this step, we consider the case when the majority of input points on which $(1+\delta)$C1NN made a mistake belong to $B(\bar x,r)$, i.e., $|\{k\leq k_l,\; X_{t_k}\in B(\bar x,r)\}|\geq \frac{k_l}{2}$. We denote $\Hcal_1$ this event. Let us now consider the subgraph $\tilde \Gcal$ given by restricting $\Gcal$ only to nodes in the ball $B(\bar x,r)$---which are mapped to the true value $1$---i.e., on times $\Tcal:=\{t\leq t_{k_l},\; X_t\in B(\bar x,r)\}$. In this subgraph, the only times with no parent are times $t_k$ with $k\leq k_l$ and $X_{t_k}\in B(\bar x,r)$, and possibly time $t=1$. Therefore, $\tilde \Gcal$ is a collection of disjoint trees with roots times $\{t_k, \; k\leq k_l, \; x_{t_k}\in B(\bar x,r)\}$, and possibly $t=1$ if $X_1\in B(\bar x,r)$. For a given time $t_k$ with $k\leq k_l$ and $X_{t_k}\in B(\bar x,r)$, we denote by $\Tcal_k$ the corresponding tree in $\tilde \Gcal$ with root $t_k$. We now introduce the notion of \emph{good} trees. We say that $\Tcal_k$ is a good tree if $\Tcal_k\cap \Dcal_{t_{k_l}+1}\neq \emptyset$, i.e., the tree survived until the last dataset. Conversely a tree is \emph{bad} if all its nodes were deleted before time $t_{k_l}+1$. We denote the set of good and bad trees by  $G=\{k:\Tcal_k\text{ good}\}$ and $B=\{k:\Tcal_k\text{ bad}\}$. In particular, we have $|G|+|B| = |\{k\leq k_l,X_{t_k}\in B(\bar x,r)\}|\geq k_l/2$. We aim to upper bound the number of bad trees. We now focus on trees $\Tcal_k$ which induced a future first mistake, i.e., such that $\{l\in\Tcal_k|\exists u\leq t_{k_l}:\phi(u)=l,\rho_\Xcal(X_l,\bar x)\geq r \text{ and } \forall v<u,\phi(v)\neq l \}\neq\emptyset$. We denote the corresponding minimum time $l_k=\min \{l\in\Tcal_k\mid \exists u\leq t_{k_l}:\phi(u)=l,\rho_\Xcal(X_l,\bar x)\geq r,\forall v<u,\phi(v)\neq l \}$. The terminology first mistake refers to the fact that the first time which used $l$ as representative corresponded to a mistake, as opposed to $l$ already having a children $X_u\in B(\bar x,r)$ which continues descendents of $l$ within the tree $\Tcal_k$. Note that bad trees necessarily induce a future first mistake---otherwise, this tree would survive. For each of these times $l_k$ two scenarios are possible.
\begin{enumerate}
    \item The value $U_{l_k}$ was never revealed within horizon $t_{k_l}$: as a result $l_k\in\Dcal_{t_{k_l}+1}$.
    \item The value $U_{l_k}$ was revealed within horizon $t_{k_l}$. Then, $U_{l_k}$ we revealed using a time $t$ for which $l_k$ was a potential representative. This scenario has two cases:
    \begin{enumerate}
        \item $\rho_\Xcal(X_t,\bar x)< r$. If used as representative $\phi(t)=l_k$, then $l_k$ would not have induced a mistake in the prediction of $Y_t$.
        \item $\rho_\Xcal(X_t,\bar x)\geq r$. If used as representative $\phi(t)=l_k$, then $l_k$ would have induced a mistake in the prediction of $Y_t$.
    \end{enumerate}
\end{enumerate}
In the case 2.a), if the point is used as representative $\phi(t)=l_k$ and if the corresponding tree $\Tcal_k$ was bad, at least another future mistake is induced by $\Tcal_k$---otherwise this tree would survive. We consider times $l_k$ for which the value was revealed, which corresponds to the only possible scenario for bad trees. We denote the corresponding set $K:=\{k:U_{l_k}\text{ revealed within horizon }t_{k_l}\}$. We now consider the sequence $k^a_1,\ldots k^a_\alpha$ containing all indices of $K$ for which scenario 2.a) was followed, ordered by chronological order for the reveal of $U_{l_{k^a_i}}$, i.e., $U_{l_{k^a_1}}$ was the first item of scenario 2.a) to be revealed, then $U_{l_{k^a_2}}$ etc. until $U_{l_{k^a_\alpha}}$. Similarly, we construct the sequence $k^b_1,\ldots k^b_\beta$ of indices in $K$ corresponding to scenario 2.b), ordered by order for the reveal of $U_{l_{k^b_i}}$. We now consider the events
\begin{align*}
    \Bcal:=\left\{\alpha + \beta\leq  \frac{k_l}{2}-\frac{k_l\delta}{32}\right\}&,\quad
    \Ccal:=\left\{\sum_{i=1}^{\min(\alpha,\lceil k_l/8\rceil)} U_{l_{k^a_i}}\geq  \frac{k_l\delta}{16}\right\},\\
    \Dcal:=&\left\{\sum_{i=1}^{\min(\beta,\lceil k_l/8 \rceil)} U_{l_{k^b_i}}\geq  \frac{k_l\delta}{16} \right\}.
\end{align*}
We now show that for $l>16$, under the event
\begin{equation*}
    \Mcal_{k_l}:=\Hcal_1\cap \left[\Bcal\cup
    (\{\alpha\geq \lceil k_l/8\rceil\}\cap \Ccal) \cup
    (\{\alpha< \lceil k_l/8\rceil\}\cap \Dcal)\right],
\end{equation*}
we have that $|G|\geq \frac{k_l\delta}{32}$. Suppose that $\Mcal_{k_l}$ is met. First note that because a bad tree can only fall into scenarios 2.a) or 2.b) we have $|B|\leq \alpha+\beta$. Hence $|G|\geq \frac{k_l}{2}-\alpha-\beta$ because of $\Hcal_1$. Thus, the result holds directly if $\Bcal$ is satisfied. We can now suppose that $\Bcal^c$ is satisfied, i.e., $\alpha+\beta > \frac{k_l}{2}-\frac{k_l\delta}{32}$. Now suppose that $\alpha\geq \lceil k_l/8\rceil$ and $\Ccal$ are also satisfied. For all indices such that $U_{l_{k^a_i}}=1$, i.e., we fall in case 2.a) and $l_{k_i^a}$ is used as representative, the corresponding tree $\Tcal_{k^a_i}$ would need to induce at least an additional mistake to be bad. Recall that in total at most $k_l/2$ mistakes are induced by points of $\Tcal$. Also, by definition of the set $K$, $\alpha+\beta$ mistakes are already induced by the times $t_k$ for $k\in K$. These corresponded to the future first mistakes for all times $\{l_k:k\in K\}$. Hence, we obtain
\begin{equation*}
    |G| \geq \sum_{i=1}^\alpha U_{l_{k^a_i}} - \left(\frac{k_l}{2} - \alpha-\beta\right) \geq \frac{k_l \delta}{16} - \frac{k_l\delta}{32} = \frac{k_l \delta}{32}.
\end{equation*}
Now consider the case where $\Hcal_1$, $\Bcal^c$, $\alpha < \lceil k_l/8\rceil $ and $\Dcal$ are met. In particular, because $l>16$ we have $k_l>16$ hence $\frac{k_l}{2}-\frac{k_l\delta}{32}\geq 2 \lceil k_l/8\rceil$. Thus, because of $\Bcal^c$ we have $\beta> \frac{k_l}{2}-\frac{k_l\delta}{32}-\alpha\geq \lceil k_l/8\rceil$. Now observe that for all indices such that $U_{l_{k^b_i}}=1$, the time $l_k$ induced two mistakes. Therefore, counting the total number of mistakes we obtain
\begin{equation*}
   \frac{k_l}{2}\geq  \alpha + \beta + \sum_{i=1}^{\beta} U_{l_{k^b_i}} \geq \frac{k_l}{2} - \frac{k_l\delta}{32} + \frac{k_l\delta}{16}
\end{equation*}
which is impossible. This ends the proof that under $\Mcal_{k_l}$ we have $|G|\geq \frac{k_l\delta}{32}$.

We now aim to lower bound the probability of this event. To do so, we first upper bound the probability of the event $\{\alpha\geq \lceil k_l/8\rceil\}\cap \Ccal^c$. We introduce a process $(Z_i)_{i=1}^{\lceil k_l/8 \rceil}$ such that for all $i\leq \max(\alpha,\lceil k_l/8\rceil)$, $Z_i=U_{l_{k^a_i}}-\delta$ and $Z_i=0$ for $\alpha<i\leq \lceil k_l/8 \rceil$. Because of the specific ordering chosen $k_1^a,\ldots,k_\alpha^a$, this process is a sequence of martingale differences, with values bounded by $1$ in absolute value. Therefore, for $l>16$ the Azuma-Hoeffing inequality yields
\begin{equation*}
    \Pbb\left[\sum_{i=1}^{\lceil k_l/8\rceil} Z_i\leq -\frac{k_l\delta}{16}\right] \leq e^{-\frac{k_l^2\delta^2}{2\cdot 16^2(k_l/8+1)}} \leq e^{-\frac{k_l\delta^2}{2^7}}.
\end{equation*}
But on the event $\{\alpha\geq \lceil k_l/8\rceil\}\cap \Ccal^c$ we have precisely
\begin{equation*}
    \sum_{i=1}^{\lceil k_l/8\rceil}Z_i = \sum_{i=1}^{\min(\alpha,\lceil k_l/8\rceil)} U_{l_{k_i^a}}-\lceil k_l/8 \rceil\delta \leq  \frac{k_l\delta}{16} -\lceil k_l/8 \rceil\delta \leq -\frac{k_l\delta}{16}.
\end{equation*}
Therefore $\Pbb[ \Ccal^c\cap \{\alpha\geq \lceil k_l/8\rceil\}]\leq  \Pbb\left[\sum_{i=1}^{\lceil k_l/8\rceil} Z_i\leq -\frac{k_l\delta}{16}\right] \leq e^{-k_l\delta^2/2^7}.$ Similarly we obtain $\Pbb[ D^c\cap \{\beta\geq \lceil k_l/8 \rceil\}] \leq e^{-k_l\delta^2/2^7}.$ Finally we write for any $l>16$,
\begin{align*}
    \Pbb[\Hcal_1\setminus\Mcal_{k_l}]&= \Pbb[\Hcal_1\cap\Bcal^c \cap(\{\alpha< \lceil k_l/8 \rceil\}\cup \Ccal^c) \cap(\{\alpha\geq \lceil k_l/8 \rceil\}\cup \Dcal^c)]\\
    &= \Pbb[\Hcal_1\cap\Bcal^c \cap[(\{\alpha< \lceil k_l/8 \rceil\}\cap \Dcal^c)\cup (\{\alpha \geq \lceil k_l/8 \rceil\}\cap \Ccal^c)]]\\
    &\leq \Pbb[\Ccal^c\cap\{\alpha\geq \lceil k_l/8\rceil\}] + \Pbb[\Dcal^c\cap\{\alpha< \lceil k_l/8\rceil\}\cap\Bcal^c]\\
    &\leq \Pbb[\Ccal^c\cap\{\alpha\geq \lceil k_l/8\rceil\}] + \Pbb[\Dcal^c\cap\{\beta\geq \lceil k_l/8\rceil\}]\\
    &\leq 2e^{-\frac{k_l\delta^2}{2^7}}.
\end{align*}
In particular, we obtain
\begin{equation*}
    \Pbb\left[\left\{|G|\geq \frac{k_l\delta}{32}\right\}\cap\Hcal_1\right]\geq \Pbb[\Mcal_{k_l}] \geq \Pbb[\Hcal_1]-2e^{-\frac{k_l\delta^2}{2^7}}.
\end{equation*}

\subsubsection{Step 2} We now consider the opposite case, when a majority of mistakes are made outside $B(\bar x,r)$, i.e., $|\{k\leq k_l,\; X_{t_k}\in B(\bar x,r)\}|< \frac{k_l}{2}$, which corresponds to the event $\Hcal_1^c$. Similarly, we consider the subgraph $\tilde \Gcal$ given by restricting $\Gcal$ only to nodes outside the ball $B(\bar x,r)$, i.e., on times $\Tcal:=\{t\leq t_{k_l},\; \rho_\Xcal(X_t,\bar x)\geq r)\}$. Again, $\tilde \Gcal$ is a collection of disjoint trees with roots times $\{t_k, \; k\leq k_l, \; \rho_\Xcal(X_{t_k},\bar x)\geq r)\}$---and possibly $t=1$. For a given time $t_k$ with $k\leq k_l$ and $\rho_\Xcal(X_{t_k},\bar x)\geq r$, we denote by $\Tcal_k$ the corresponding tree in $\tilde \Gcal$ with root $t_k$. Similarly to the previous case, $\Tcal_k$ is a \emph{good} tree if $\Tcal_k\cap \Dcal_{t_{k_l}+1}\neq \emptyset$ and \emph{bad} otherwise. We denote the set of good and bad trees by  $G=\{k:\Tcal_k\text{ good}\}$. We can again focus on trees $\Tcal_k$ which induced a future first mistake, i.e., such that $\{l\in\Tcal_k|\exists u\leq t_{k_l}:\phi(u)=l,\rho_\Xcal(X_l,\bar x)< r \text{ and } \forall v<u,\phi(v)\neq l \}\neq\emptyset$ and more specifically their minimum time $l_k=\min \{l\in\Tcal_k\mid \exists u\leq t_{k_l}:\phi(u)=l,\rho_\Xcal(X_l,\bar x)< r,\forall v<u,\phi(v)\neq l \}$. The same analysis as above shows that
\begin{equation*}
    \Pbb\left[\left\{|G|\geq \frac{k_l\delta}{32}\right\}\cap\Hcal_1^c\right] \geq \Pbb[\Hcal_1^c] -2e^{-\frac{k_l\delta^2}{2^7}}.
\end{equation*}
Therefore, if $G$ denotes more generally the set of good trees (where we follow the corresponding case 1 or 2) we finally obtain that for any $l>16$,
\begin{equation*}
    \Pbb\left[|G|\geq \frac{k_l\delta}{32}\right]\geq 1-4e^{-\frac{k_l\delta^2}{2^7}}.
\end{equation*}
We denote by $\tilde \Mcal_{k_l}$ this event. By Borel-Cantelli lemma, almost surely, there exists $\hat l$ such that for any $l\geq \hat l$, the event $\tilde \Mcal_{k_l}$ is satisfied.
We denote $\Mcal:=\bigcup_{l\geq 1} \bigcap_{l'\geq l}\tilde \Mcal_{k_l}$ this event of probability one. The aim is to show that on the event $\Acal\cap\Mcal\cap\bigcap_{l\geq 1}(\Ecal_l\cap\Fcal_l)$, which has probability at least $\frac{\eta}{4}$, $\Xbb$ disproves the $\smv$ condition. In the following, we consider a specific realization $\mb x$ of the process $\Xbb$ falling in the event $\Acal\cap\Mcal\cap\bigcap_{l\geq 1}(\Ecal_l\cap\Fcal_l)$---$\mb x$ is not random anymore. Let $\hat l$ be the index given by the event $\Mcal$ such that for any $l\geq \hat l$, $\Mcal_{k_l}$ holds. We consider $l\geq \hat l$ and successively consider different cases in which the realization $\mb x$ may fall.

\begin{itemize}
    \item In the first case, we suppose that a majority of mistakes were made in $B(\bar x,r)$, i.e., that we fell into event $\Hcal_1$ similarly to Step 1. Because the event $\tilde \Mcal_{k_l}$ is satisfied we have $|G|\geq \frac{k_l\delta}{2^5}$. Now note that trees are disjoint, therefore, $\sum_{k\in G} |\Tcal_k|\leq t_{k_l}<\frac{2k_l}{\epsilon}.$
Therefore,
\begin{equation*}
    \sum_{k\in G}\1_{|\Tcal_k|\leq \frac{2^7}{\epsilon\delta}} = |G| - \sum_{k\in G}\1_{|\Tcal_k|> \frac{2^7}{\epsilon\delta}}> |G|-\frac{\epsilon\delta}{2^7} \sum_{k\in G}|\Tcal_k|\geq \frac{k_l\delta}{2^5}- \frac{k_l\delta}{2^6} = \frac{k_l\delta}{2^6}.
\end{equation*}
We will say that a tree $|\Tcal_k|$ is \emph{sparse} if it is good and has at most $\frac{2^7}{\epsilon\delta}$ nodes. With $S := \{k\in G,\;|\Tcal_k|\leq \frac{2^7}{\epsilon\delta} \}$ the set of sparse trees, the above equation yields $|S|\geq \frac{k_l\delta}{2^6}$. The same arguments as in \cite{blanchard2022universal} give
\begin{equation*}
    |\{i,\; A_i\cap \mb{x}_{\leq t_{k_l}}\neq \emptyset \}|\geq |S|\geq \frac{k_l\delta}{2^6} \geq \frac{\epsilon\delta}{2^7}t_{k_l}.
\end{equation*}
The only difference is that we chose $c_\epsilon$ so that $2^{2\cdot \frac{2^7}{\epsilon\delta} -1} \leq \frac{1}{4c_\epsilon}$ as needed in the original proof.

\item We now turn to the case when the majority of input points on which $(1+\delta)$C1NN made a mistake are not in the ball $B(\bar x,r)$, similarly to Step 2. Using the same notion of sparse tree $S := \{k\in G,\;|\Tcal_k|\leq \frac{2^7}{\epsilon\delta} \}$, we have again $|S|\geq \frac{k_l\delta}{2^6}$. We use the same arguments as in the original proof. Suppose $|\{k\in S,\; \rho_\Xcal(x_{p^k_{d(k)}},\bar x)>r\}|\geq \frac{|S|}{2}$, then we have
\begin{equation*}
    |\{i,\; A_i\cap \mb{x}_{\leq t_{k_l}}\neq \emptyset \}|\geq |\{k\in S,\; \rho_\Xcal(x_{p^k_{d(k)}},\bar x)>r\}| \geq \frac{|S|}{2} \geq \frac{k_l\delta}{2^7}\geq \frac{\epsilon\delta}{2^8}t_{k_l}.
\end{equation*}
\end{itemize}

\subsubsection{Step 3}
In this last step, we suppose again that the majority of input points on which $(1+\delta)$C1NN made a mistake are not in the ball $B(\bar x,r)$ but that $|\{k\in S,\; \rho_\Xcal(x_{p^k_{d(k)}},\bar x)>r\}|< \frac{|S|}{2}$. Therefore, we obtain
\begin{equation*}
    |\{k\in S,\; \rho_\Xcal(x_{p^k_{d(k)}},\bar x)=r\}| = |S|-|\{k\in S,\; \rho_\Xcal(x_{p^k_{d(k)}},\bar x)>r\}| \geq \frac{|S|}{2} \geq \frac{k_l\delta}{2^7}\geq \frac{\epsilon\delta}{2^8}t_{k_l}.
\end{equation*}
We will now make use of the partition $(P_i)_{i\geq 1}$. Because $(n_u)_{u\geq 1}$ is an increasing sequence, let $u\geq 1$ such that $n_{ u+1}\leq t_{k_l}\leq n_{ u+2}$ (we can suppose without loss of generality that $t_{k_0}>n_2$). Note that we have $n_u\leq \frac{\epsilon\delta}{2^9}n_{u+1}\leq \frac{\epsilon\delta}{2^9}t_{k_l}$. Let us now analyze the process between times $n_u$ and $t_{k_l}$. In particular, we are interested in the indices $T=\{k\in S,\; \rho_\Xcal(x_{p^k_{d(k)}},\bar x)=r\}$ and times $\Ucal_u = \{p^k_{d(k)}:\; n_u< p^k_{d(k)}\leq k_l,\; k\in T\}$. In particular, we have
\begin{equation*}
    |\Ucal_u| \geq  |\{k\in S,\; \rho_\Xcal(x_{p^k_{d(k)}},\bar x)=r\}| - n_u  \geq \frac{\epsilon\delta}{2^8}t_{k_l} -\frac{\epsilon\delta}{2^9}t_{k_l} = \frac{\epsilon\delta}{2^9}t_{k_l}.
\end{equation*}
Defining $T' := \{k\in T,\; r-\frac{r}{2^{u+3}}\leq \rho_\Xcal(x_{\phi(t_k)},\bar x)<r\}$, the same arguments as in the original proof yield
\begin{equation*}
    |\{i,\; P_i\cap \mb x_{\leq t_{k_l}} \neq \emptyset\}| \geq |T'| \geq |\Ucal_u|-|\{i,\; P_i(\tau_u)\cap \mb x_{\Ucal_u} \neq \emptyset\}|\geq \frac{\epsilon\delta}{2^9}t_{k_l}-\frac{\epsilon\delta}{2^{10}}t_{k_l}=\frac{\epsilon\delta}{2^{10}}t_{k_l}.
\end{equation*}

\subsubsection{Step 4}
In conclusion, in all cases, we obtain 
\begin{equation*}
    |\{Q\in \Qcal,\; Q\cap \mb x_{\leq t_{k_l}} \neq \emptyset\}| \geq \max(|\{i,\; A_i\cap \mb{x}_{\leq t_{k_l}}\neq \emptyset \}|,|\{i,\; P_i\cap \mb x_{\leq t_{k_l}} \neq \emptyset\}|) \geq \frac{\epsilon\delta}{2^{10}}t_{k_l}.
\end{equation*}
Because this is true for all $l\geq \hat l$ and $t_{k_l}$ is an increasing sequence, we conclude that $\mb x$ disproves the $\smv$ condition for $\Qcal$. Recall that this holds whenever the event $\Acal\cap\Mcal\cap\bigcap_{l\geq 1}(\Ecal_l\cap \Fcal_l)$ is met. Thus,
\begin{equation*}
    \Pbb[|\{Q\in \Qcal,\; Q\cap \Xbb_{<T}\}|=o(T)]\leq 1-\Pbb\left[\Acal\cap\Mcal\cap\bigcap_{l\geq 1}(\Ecal_l\cap \Fcal_l)\right] \leq 1-\frac{\eta}{4}<1.
\end{equation*}
This shows that $\Xbb\notin \smv$ which is absurd. Therefore $(1+\delta)$C1NN is consistent on $f^*$. This ends the proof of the proposition.
\end{proof}

Using the fact that in the $(1+\delta)$C1NN learning rule, no time $t$ can have more than $2$ children, as the 2C1NN rule, we obtain with the same proof as in \cite{blanchard2022universal} the following proposition.

\begin{proposition}
\label{prop:opt1+delta_bin}
Fix $0<\delta\leq 1$. Let $(\Xcal,\Bcal)$ be a separable Borel space. For the binary classification setting, the learning rule $(1+\delta)$C1NN is universally consistent for all processes $\Xbb\in \smv$.
\end{proposition}

Finally, we use a result from \cite{blanchard2021universal} which gives a reduction from any near-metric bounded value space to binary classification.

\begin{theorem}[\cite{blanchard2021universal}]
\label{thm:invariance}
If $(1+\delta)$C1NN is universally consistent under a process $\Xbb$ for binary classification, it is also universally consistent under $\Xbb$ for any separable near-metric setting $(\Ycal,\ell)$ with bounded loss.
\end{theorem}

Together with \cref{prop:opt1+delta_bin}, \cref{thm:invariance} ends the proof of \cref{thm:1+deltaC1NN_optimistic}.

\subsection{Proof of \cref{thm:optimistic_regression_totally_bounded}}

Let $0<\epsilon\leq 1$. We first analyze the prediction of the learning rule $f^\epsilon_\cdot$. In the rest of the proof, we denote $\bar\ell(\hat Y_t(\epsilon),Y_t):=\sum_{y\in\Ycal_\epsilon} \Pbb(\hat Y_t(\epsilon)=y) \ell(y,Y_t)$ the immediate expected loss at each iteration.  The learning rule was constructed so that we perform exactly the classical Hedge / exponentially weighted average forecaster on each cluster of times $\Ccal(t) = \{u\leq t: u \stackrel \phi \sim t\}$. As a result \cite{cesa2006prediction} (Theorem 2.2), we have that for any $t\geq 1$,
\begin{align*}
    \frac{1}{\bar\ell} \sum_{u\in\Ccal(t)} \bar\ell(\hat Y_u(\epsilon),Y_u) &\leq \frac{1}{\bar\ell} \min_{y\in\Ycal_\epsilon} \sum_{u\in \Ccal(t)} \ell(y,Y_u) + \frac{\ln|\Ycal_\epsilon|}{\bar \ell \eta_\epsilon} + \frac{|\Ccal(t)|\bar \ell\eta_\epsilon}{8}\\
    & \leq  \frac{1}{\bar\ell} \min_{y\in\Ycal_\epsilon} \sum_{u\in \Ccal(t)} \ell(y,Y_u) + \sqrt{\frac{ \ln |\Ycal_\epsilon|}{8 T_\epsilon}}(T_\epsilon+ |\Ccal(t)|)\\
    & \leq  \frac{1}{\bar\ell} \min_{y\in\Ycal_\epsilon} \sum_{u\in \Ccal(t)} \ell(y,Y_u) + \frac{\epsilon}{\bar \ell} \max(T_\epsilon, |\Ccal(t)|)
\end{align*}
Now consider a horizon $T\geq 1$, and enumerate all the clusters $\Ccal_1(T),\ldots,\Ccal_{p(T)}(T)$ at horizon $T$, i.e. the classes of equivalence of $\phi$ among the times $\{t\leq T\}$. Note that if a cluster $i\leq p$ has $|\Ccal_i(T)| < T_\epsilon$, then either it must contain a time $t\in\Ncal$ which is a leaf of the tree formed by $\phi$ until time $T$, or it is a cluster of duplicates of an instance $X_u$ which has already had $\frac{T_\epsilon}{\epsilon}$ occurrences. As a result, the times falling into such clusters of duplicates with less than $T_\epsilon$ members form at most a proportion $\epsilon$ of the total $T$ times. Denote by $\Acal_i:=\{t\leq T: t\in\Ncal, |\{u\leq T:\phi(u)=t\}|=i\}$ times which have excactly $i$ children for $i\in\{0,1,2\}$. Note that no time can have more than $2$ children. In particular $\Acal_0$ is the set of leaves. Then, by summing the above equations we obtain
\begin{align*}
    \sum_{t=1}^{T} \bar\ell(\hat Y_t(\epsilon),Y_t)  &\leq \sum_{i=1}^{p(T)} \left(\min_{y\in\Ycal_\epsilon} \sum_{u\in\Ccal_i(T)} \ell(y,Y_u) + \epsilon \max(T_\epsilon, |\Ccal_i(T)|) \right)\\
    &\leq \sum_{i=1}^{p(T)} \min_{y\in\Ycal_\epsilon} \sum_{u\in\Ccal_i(T)} \ell(y,Y_u) + \epsilon T +  T_\epsilon |\{1\leq i\leq p: |\Ccal_i(T)|< T_\epsilon\}| \\
    &\leq \sum_{i=1}^{p(T)} \min_{y\in\Ycal_\epsilon} \sum_{u\in\Ccal_i(T)} \ell(y,Y_u) + \epsilon T +  T_\epsilon |\Acal_0| + \epsilon T_\epsilon,
\end{align*}
where in the last inequality we used the fact that all clusters with $|\Ccal_i(T)|<T_\epsilon$ contain a leaf from $\Acal_0$, which is therefore distinct for each such cluster. Now note that by counting the number of edges of the tree structure we obtain $\frac{1}{2}(3|\Acal_2| + 2|\Acal_1|+|\Acal_0|-1) = T-1 = |\Acal_0|+|\Acal_1|+|\Acal_2|-1$, where the $-1$ on the left-hand side accounts for the root of this tree which does not have a parent. Hence we obtain $|\Acal_0|= |\Acal_2|+1$. Further, $|\Acal_2|\leq |\{t\leq T:U_t=1\}|$ which follows a binomial distribution $\Bcal(T,\delta_\epsilon)$. Therefore, using the Chernoff bound, with probability $1-e^{-T\delta_\epsilon/3}$ we have
\begin{align*}
    \sum_{t=1}^{T} \bar\ell(\hat Y_t(\epsilon),Y_t)  &\leq \sum_{i=1}^{p(T)} \min_{y\in\Ycal_\epsilon} \sum_{u\in\Ccal_i(T)} \ell(y,Y_u) + 2\epsilon T +  T_\epsilon ( 1+ 2T \delta_\epsilon)\\
    &\leq \sum_{i=1}^{p(T)} \min_{y\in\Ycal_\epsilon} \sum_{u\in\Ccal_i(T)} \ell(y,Y_u) + T_\epsilon  +3\epsilon T.
\end{align*}
We now observe that the sequence $\{\ell(\hat Y_t(\epsilon),Y_t)-\bar\ell(\hat Y_t(\epsilon),Y_t)\}_{T\geq 1}$ is a sequence of martingale differences bounded by $\bar\ell$ in absolute value. Hence, the Hoeffding-Azuma inequality yields that for any $T\geq 1$, with probability $1-\frac{1}{T^2}-e^{-T\delta_\epsilon/3}$,
\begin{equation*}
    \sum_{t=1}^{T} \ell(\hat Y_t(\epsilon),Y_t)\leq \sum_{i=1}^{p(T)} \min_{y\in\Ycal_\epsilon} \sum_{u\in\Ccal_i(T)} \ell(y,Y_u) + T_\epsilon +3\epsilon T + 2\bar\ell \sqrt{T\ln T}.
\end{equation*}
Because $\sum_{T\geq 1}\frac{1}{T^2}+e^{-T\delta_\epsilon/3}<\infty$ the Borel-Cantelli lemma implies that with probability one, there exists a time $\hat T$ such that
\begin{equation*}
    \forall T\geq \hat T,\quad \sum_{t=1}^{T} \ell(\hat Y_t(\epsilon),Y_t)\leq \sum_{i=1}^{p(T)} \min_{y\in\Ycal_\epsilon} \sum_{u\in\Ccal_i(T)} \ell(y,Y_u) + T_\epsilon+ 2\bar\ell \sqrt{T\ln T} +3\epsilon T.
\end{equation*}
We denote by $\Ecal_\epsilon$ this event. We are now ready to analyze the risk of the learning rule $f^\epsilon_\cdot$. Let $f:\Xcal\to\Ycal$ a measurable function to which we compare the prediction of $f^\epsilon_\cdot$. By \cref{thm:1+deltaC1NN_optimistic}, the rule $(1+\delta_\epsilon)$C1NN is optimistically universal in the noiseless setting. Therefore, because $\Xbb\in\soul$ we have in particular
\begin{equation*}
    \frac{1}{T}\sum_{t=1}^T  \ell((1+\delta_\epsilon)C1NN_t(\Xbb_{\leq t-1},f(\Xbb_{\leq t-1}),X_t),f(X_t))\to 0\quad (a.s.),
\end{equation*}
i.e., almost surely, $\frac{1}{T}\sum_{t\leq T, t\in\Ncal} \ell(f(X_{\phi(t)}),f(X_t)) \to 0$ --- the times corresponding to duplicate instances incur a $0$ loss by memorization. We denote by $\Fcal_\epsilon$ this event of probability one. Using \cref{lemma:loss_identity}, we write for any $u=1,\ldots,T_\epsilon-1$,
\begin{align*}
    &\sum_{t\leq T, t\in\Ncal} \ell(f(X_{\phi^u(t)}),f(X_t))\\
    &\leq 2^{\alpha-1}\sum_{t\leq T, t\in\Ncal} \ell(f(X_{\phi^{u-1}(t)}),f(X_t)) + 2^{\alpha-1} \sum_{t\leq T, t\in\Ncal}\ell(f(X_{\phi^l(t)}),f(X_{\phi^{u-1}(t)}))\\
    &\leq 2^{\alpha-1}\sum_{t\leq T, t\in\Ncal} \ell(f(X_{\phi^{u-1}(t)}),f(X_t)) \\
    &\quad\quad\quad+ 2^{\alpha-1}\sum_{t\leq T, t\in\Ncal} \ell(f(X_{\phi(t)}),f(X_t)) \cdot |\{l\leq T:\phi^{u-1}(l)=t\}|\\
    &\leq 2^{\alpha-1}\sum_{t\leq T, t\in\Ncal} \ell(f(X_{\phi^{u-1}(t)}),f(X_t)) + 2^{\alpha+u-2}\sum_{t\leq T, t\in\Ncal} \ell(f(X_{\phi(t)}),f(X_t))
\end{align*}
where we used the fact that times have at most $2$ children. Therefore, iterating the above equations, we obtain that on $\Fcal_\epsilon$, for any $u=1,\ldots,T_\epsilon-1$
\begin{align*}
    \frac{1}{T}\sum_{t\leq T, t\in\Ncal} \ell(f(X_{\phi^u(t)}),f(X_t)) &\leq \left(\sum_{k=1}^u 2^{\alpha+k-2 + (\alpha-1)(u-k)}\right)\frac{1}{T}\sum_{t\leq T, t\in\Ncal} \ell(f(X_{\phi(t)}),f(X_t))\\
    &\leq \frac{2^{u\alpha}}{T}\sum_{t\leq T, t\in\Ncal} \ell(f(X_{\phi(t)}),f(X_t)) \to 0.
\end{align*}
In the rest of the proof, for any $y\in\Ycal$, we will denote by $y^\epsilon$ a value in the $\epsilon-$net $\Ycal_\epsilon$ such that $\ell(y,y^\epsilon)\leq \epsilon$. We now pose $\mu_\epsilon=\min\{0<\mu\leq 1:c_\mu^\alpha \leq \frac{1}{\sqrt \epsilon} \}$ if the corresponding set is non-empty and $\mu_\epsilon=1$ otherwise. Note that because $c_\mu^\alpha$ is non-increasing in $\mu$, we have $\mu_\epsilon\longrightarrow_{\epsilon\to 0} 0$. Now let $0<\mu\leq 1$. $\mu:=\epsilon^{\frac{1}{\alpha+1}}$. Finally, for any cluster $\Ccal_i(T)$, let $t_i = \min\{u\in \Ccal_i(T)\}$. Putting everything together, on the event $\Ecal_\epsilon\cap\Fcal_\epsilon$, for any $T\geq \hat T$, we have
\begin{align*}
    \sum_{t=1}^T \ell(\hat Y_t(\epsilon),Y_t)&\leq \sum_{i=1}^{p(T)} \min_{y\in\Ycal_\epsilon} \sum_{u\in\Ccal_i(T)} \ell(y,Y_u) + T_\epsilon  + 2\bar\ell \sqrt{T\ln T} +3\epsilon T\\
    &\leq  \sum_{i=1}^{p(T)} \sum_{u\in\Ccal_i(T)} \ell(f(X_{t_i})^\epsilon,Y_u) + T_\epsilon \bar\ell + 2\bar\ell \sqrt{T\ln T} +3\epsilon T\\
    &\leq \sum_{i=1}^{p(T)} \sum_{u\in\Ccal_i(T)}
    [c_{\mu_\epsilon}^\alpha \ell(f(X_{t_i})^\epsilon,f(X_{t_i})) +(c_{\mu_\epsilon}^\alpha)^2\ell(f(X_{t_i}),f(X_u))\\
    &\quad\quad\quad\quad+ (1+{\mu_\epsilon})^2 \ell(f(X_u),Y_u) ] + T_\epsilon \bar\ell + 2\bar\ell \sqrt{T\ln T} +3\epsilon T\\
    &\leq (1+{\mu_\epsilon})^2 \sum_{t=1}^T \ell(f(X_t),Y_t) + (c_{\mu_\epsilon}^\alpha)^2 \frac{T_\epsilon}{\epsilon}\sum_{u=1}^{T_\epsilon-1}\sum_{t\leq T, t\in\Ncal} \ell(f(X_t),f(X_{\phi^u(t)})) \\
    &\quad\quad\quad\quad+  T_\epsilon \bar\ell + 2\bar\ell \sqrt{T\ln T} +(3 +  c_{\mu_\epsilon}^\alpha)\epsilon T\\
    &\leq \sum_{t=1}^T \ell(f(X_t),Y_t) + \frac{(c_{\mu_\epsilon}^\alpha)^2 T_\epsilon}{\epsilon}\sum_{u=1}^{T_\epsilon-1}\sum_{t\leq T, t\in\Ncal} \ell(f(X_t),f(X_{\phi^u(t)}))\\
    &\quad\quad\quad\quad+  T_\epsilon \bar\ell + 2\bar\ell \sqrt{T\ln T} +(3\epsilon +  \epsilon c_{\mu_\epsilon}^\alpha + 3{\mu_\epsilon}) T,
\end{align*}
where in the third inequality we used \cref{lemma:loss_identity} twice, and in the fourth inequality we used the fact that clusters containing distinct instances have at most $\frac{T_\epsilon}{\epsilon}$ duplicates of each instance. Hence, for any $\epsilon<(c_1^\alpha)^{-2}$, on the event $\Ecal_\epsilon\cap\Fcal_\epsilon$, we obtain
\begin{equation*}
    \limsup_{T\to\infty} \frac{1}{T} \sum_{t=1}^T \ell(\hat Y_t(\epsilon),Y_t) - \ell(f(X_t),Y_t) \leq 3\epsilon +  \epsilon c_{\mu_\epsilon}^\alpha + 3{\mu_\epsilon} \leq  3\epsilon +  \sqrt \epsilon  + 3{\mu_\epsilon},
\end{equation*}
where $\mu_\epsilon\longrightarrow_{\epsilon\to 0} 0$. We now denote $\delta_\epsilon:= 2\epsilon +  \sqrt \epsilon  + 3{\mu_\epsilon}$ and $i_0=\lceil \frac{2\ln c_1^\alpha}{\ln 2} \rceil$. We now turn to the final learning rule and show that by using the predictions of the rules $f^{\epsilon_i}_\cdot$ for $i\geq 0$, it achieves zero risk. First, by the union bound, on the event $\bigcap_{i\geq 0} \Ecal_{\epsilon_i}\cap\Fcal_{\epsilon_i}$ of probability one,
\begin{equation*}
    \limsup_{T\to\infty} \frac{1}{T} \sum_{t=1}^T \ell(\hat Y_t(\epsilon_i),Y_t) - \ell(f(X_t),Y_t) \leq \delta_{\epsilon_i},\quad \forall i\geq i_0.
\end{equation*}
Now define $\Hcal$ the event probability one according to \cref{lemma:concatenation_predictors} such that there exists $\hat t$ for which
\begin{equation*}
    \forall t\geq \hat t,\forall i\in I_t,\quad \sum_{s=t_i}^t\ell(\hat Y_t,Y_t) \leq \sum_{s=t_i}^t \ell(\hat Y_t(\epsilon_i),Y_t) + 
    (2+\bar\ell+\bar\ell^2)\sqrt{t\ln t}.
\end{equation*}
In the rest of the proof we will suppose that the event $\Hcal\cap \bigcap_{i\geq 0} \Ecal_{\epsilon_i}\cap\Fcal_{\epsilon_i}$ is met. Let $i\geq i_0$. For any $T\geq \max(\hat t,t_i)$, we have
\begin{align*}
    \frac{1}{T}\sum_{t=1}^T \ell(\hat Y_t,Y_t)-& \ell(f(X_t),Y_t) \leq \frac{t_i}{T}\bar \ell + \frac{1}{T}\sum_{t=t_i}^T \ell(\hat Y_t,Y_t)- \ell(f(X_t),Y_t)\\
    &\leq \frac{t_i}{T}\bar \ell + \frac{1}{T}\sum_{t=t_i}^T \ell(\hat Y_t(\epsilon_i),Y_t)- \ell(f(X_t),Y_t) + (2+\bar\ell+\bar\ell^2)\sqrt{\frac{\ln T}{T}}\\
    &\leq \frac{1}{T}\sum_{t=1}^T \ell(\hat Y_t(\epsilon_i),Y_t)- \ell(f(X_t),Y_t) + \frac{2t_i}{T}\bar\ell  + (2+\bar\ell+\bar\ell^2)\sqrt{\frac{\ln T}{T}}.
\end{align*}
Therefore we obtain $\limsup_{T\to\infty} \frac{1}{T}\sum_{t=1}^T \ell(\hat Y_t,Y_t)- \ell(f(X_t),Y_t) \leq \delta_{\epsilon_i}$. Because this holds for any $i\geq i_0$ on the event $\Hcal\cap \bigcap_{i\geq 0} \Ecal_{\epsilon_i}\cap\Fcal_{\epsilon_i}$ of probability one, and $\delta_{\epsilon_i}\to 0$ for $i\to\infty$, we have
\begin{equation*}
    \limsup_{T\to\infty} \frac{1}{T}\sum_{t=1}^T \ell(\hat Y_t,Y_t)- \ell(f(X_t),Y_t) \leq 0.
\end{equation*}
This ends the proof of the theorem.

\subsection{Proof of \cref{lemma:concatenation_predictors}}
\label{subsection:proof_concatenation}

We first introduce the following helper lemma which can be found in \cite{cesa2006prediction}.

\begin{lemma}[\cite{cesa2006prediction}]
\label{lemma:cesa_bianchi_lugosi}
For all $N\geq 2$, for all $\beta\geq\alpha\geq 0$ and for all $d_1,\ldots,d_N\geq 0$ such that $\sum_{i=1}^N e^{-\alpha d_i}\geq 1$,
\begin{equation*}
    \ln \frac{\sum_{i=1}^N e^{-\alpha d_i}}{\sum_{i=1}^N e^{-\beta d_i}} \leq \frac{\beta-\alpha}{\alpha} \ln N.
\end{equation*}
\end{lemma}

We are now ready to compare the predictions of the learning rule $f_\cdot$ to the predictions of the rules $f^\epsilon_\cdot$.

For any $t\geq 0$, we define the instantaneous regret $r_{t,i} = \hat\ell_t - \ell(\hat Y_t(\epsilon_i),Y_t)$. We first note that $|r_{t,i}|\leq \bar \ell$. We now define $w'_{t-1,i}:=e^{\eta_{t-1}(\hat L_{t-1,i}-L_{t-1,i})}$. We also introduce $W_{t-1} = \sum_{i\in I_t}w_{t-1,i}$ and $W'_{t-1} = \sum_{i\in I_{t-1}} w'_{t-1,i}$. We denote the index $k_t\in I_t$ such that $\hat L_{t,k_t}- L_{t,k_t} = \max_{i\in I_t} \hat L_{t,i} - L_{t,i}$. Then we write
\begin{multline*}
    \frac{1}{\eta_t}\ln \frac{w_{t-1,k_{t-1}}}{W_{t-1}}- \frac{1}{\eta_{t+1}}\ln \frac{w_{t,k_t}}{W_t}=
    \left(\frac{1}{\eta_{t+1}}-\frac{1}{\eta_t}\right)\ln\frac{W_t}{w_{t,k_t}} + \frac{1}{\eta_t} \ln \frac{W_t/w_{t,k_t}}{W'_t/w'_{t,k_t}} \\
    +\frac{1}{\eta_t}\ln \frac{w_{t-1,k_{t-1}}}{w'_{t,k_t}} +  \frac{1}{\eta_t}\ln \frac{W'_t}{W_{t-1}}.
\end{multline*}
By construction, we have $\ln\frac{W_t}{w_{t,k_t}}\leq \ln |I_t| \leq \ln(1+\ln t)$. Further, we have that
\begin{align*}
    \frac{1}{\eta_t} \ln \frac{W_t/w_{t,k_t}}{W'_t/w'_{t,k_t}} &=\frac{1}{\eta_t}\ln \frac{\sum_{i\in I_{t+1}} e^{\eta_{t+1}(\hat L_{t,i}-L_{t,i}-\hat L_{t,k_t}+L_{t,k_t})}}{\sum_{i\in I_t} e^{\eta_t(\hat L_{t,i}-L_{t,i}-\hat L_{t,k_t}+L_{t,k_t})}}\\
    &=\frac{1}{\eta_t}\ln \frac{\sum_{i\in I_{t+1}} w_{t,i}}{\sum_{i\in I_t} w_{t,i}} + \frac{1}{\eta_t} \ln \frac{\sum_{i\in I_{t+1}} e^{\eta_{t+1}(\hat L_{t,i}-L_{t,i}-\hat L_{t,k_t}+L_{t,k_t})}}{\sum_{i\in I_{t+1}} e^{\eta_t(\hat L_{t,i}-L_{t,i}-\hat L_{t,k_t}+L_{t,k_t})}}\\
    &\leq \frac{1}{\eta_t}\ln \frac{\sum_{i\in I_{t+1}} w_{t,i}}{\sum_{i\in I_t} w_{t,i}} + \frac{1}{\eta_t}\left(\frac{\eta_t-\eta_{t+1}}{\eta_{t+1}}\right) \ln |I_{t+1}|\\
    &\leq  \frac{|I_{t+1}|-|I_t|}{\eta_t \sum_{i\in I_t} w_{t,i}} + \left(\frac{1}{\eta_{t+1}}-\frac{1}{\eta_t}\right) \ln (1+\ln (t+1)),
\end{align*}
where in the first inequality we applied \cref{lemma:cesa_bianchi_lugosi}. We also have
\begin{equation*}
    \frac{1}{\eta_t}\ln \frac{w_{t-1,k_{t-1}}}{w'_{t,k_t}} = (\hat L_{t-1,k_{t-1}}- L_{t-1,k_{t-1}}) - (\hat L_{t,k_t}, L_{t,k_t}).
\end{equation*}
Last, because $|r_{t,i}|\leq \bar \ell$ for all $i\in I_t$, we can use Hoeffding's lemma to obtain
\begin{equation*}
    \frac{1}{\eta_t}\ln \frac{W'_t}{W_{t-1}} = \frac{1}{\eta_t} \ln \sum_{i\in I_t} \frac{w_{t-1,i}}{W_{t-1}}e^{\eta_t r_{t,i}} \leq \frac{1}{\eta_t}\left( \eta_t\sum_{i\in I_t} r_{t,i} \frac{w_{t-1,i}}{W_{t-1}} + \frac{\eta_t^2 (2\bar\ell)^2}{8}\right) = \frac{1}{2}\eta_t \bar\ell^2.
\end{equation*}
Putting everything together gives
\begin{multline}
\label{eq:to_sum_combine_estimators}
     \frac{1}{\eta_t}\ln \frac{w_{t-1,k_{t-1}}}{W_{t-1}}- \frac{1}{\eta_{t+1}}\ln \frac{w_{t,k_t}}{W_t}
     \leq 2\left(\frac{1}{\eta_{t+1}}-\frac{1}{\eta_t}\right) \ln (1+\ln (t+1)) + \frac{|I_{t+1}|-|I_t|}{\eta_t \sum_{i\in I_t} w_{t,i}} \\
     + (\hat L_{t-1,k_{t-1}}- L_{t-1,k_{t-1}}) - (\hat L_{t,k_t}- L_{t,k_t}) + \frac{1}{2}\eta_t \bar\ell^2.
\end{multline}
First suppose that we have $\sum_{i\in I_t}w_{t,i}\leq 1$. Then either $k_t\in I_{t+1}\setminus I_t$ in which case $\hat L_{t,k_t}-L_{t,k_t}=0$, or we have directly
\begin{equation*}
    \hat L_{t,k_t}-L_{t,k_t} \leq \frac{1}{\eta_{t+1}}\ln\left[\sum_{i\in I_t}w_{t,i}\right] \leq 0.
\end{equation*}
Otherwise, let $t'=\min \{1\leq s\leq t:\forall s\leq s'\leq t,\sum_{i\in I_{s'}} w_{s',i}\geq 1\}$. We sum equation~\eqref{eq:to_sum_combine_estimators} for $s=t',\ldots, t$ which gives
\begin{multline*}
     \frac{1}{\eta_1}\ln \frac{w_{t'-1,k_{t'-1}}}{W_{t'-1}}- \frac{1}{\eta_{t+1}}\ln \frac{w_{t,k_t}}{W_t}  \leq \frac{2}{\eta_{t+1}} \ln (1+\ln (t+1))+ \frac{|I_{t+1}|}{\eta_t} \\
     + (\hat L_{t'-1,k_{t'-1}}- L_{t'-1,k_{t'-1}}) - (\hat L_{t,k_t}- L_{t,k_t}) + \frac{\bar\ell^2}{2}\sum_{s=t'}^t\eta_s.
\end{multline*}
Note that we have $\frac{w_{t,k_t}}{W_t}\leq 1$ and $\frac{w_{t'-1,k_{t'-1}}}{W_{t'-1}}\geq \frac{1}{|I_{t'-1}|}\geq \frac{1}{1+\ln t}$. Also, assuming $t'\geq 2$, since $\sum_{i\in I_{t'-1}} w_{t'-1,i}< 1$, we have for any $i\in I_{t'-1}$ that $\hat L_{t'-1,i}-L_{t'-1,i}\leq 0$, hence $\hat L_{t'-1,k_{t'-1}}- L_{t'-1,k_{t'-1}}\leq 0$. If $t'=1$ we have directly $\hat L_{0,k_0}-L_{0,k_0}=0$. Finally, using the fact that $\sum_{s=1}^t \frac{1}{\sqrt s}\leq 2\sqrt t$, we obtain
\begin{align*}
    \hat L_{t,k_t}- L_{t,k_t} &\leq  \ln(1+\ln (t+1))\left(1+2\sqrt{\frac{t+1}{\ln(t+1)}}\right) +(1+\ln (t+1))\sqrt {\frac{t}{\ln t}} + \bar\ell^2\sqrt {t\ln t}\\
    &\leq (3/2+\bar \ell^2)  \sqrt {t \ln t},
\end{align*}
for all $t\geq t_0$ where $t_0$ is a fixed constant. This in turn implies that for all $t\geq t_0$ and $i\in I_t$, we have $\hat L_{t,i} -L_{t,i} \leq (3/2+\bar \ell^2)  \sqrt {t \ln t}.$ Now note that $|\ell(\hat Y_t,Y_t)-\hat \ell_t|\leq \bar \ell$. Hence, we can use Hoeffding-Azuma inequality for the variables $\ell(\hat Y_t,Y_t)-\hat \ell_t$ that form a sequence of martingale differences to obtain $ \Pbb\left[\sum_{s=t_i}^t \ell(\hat Y_s,Y_s)>\hat L_{t,i} + u\right] \leq e^{ -\frac{2u^2}{t\bar\ell^2}}.$ Hence, for $t\geq t_0$ and $i\in I_t$, with probability $1-\delta$, we have
\begin{equation*}
    \sum_{s=t_i}^t \ell(\hat Y_s,Y_s)\leq \hat L_{t,i} + \bar\ell \sqrt{\frac{t}{2}\ln\frac{1}{\delta}} \leq L_{t,i} +(3/2+\bar \ell^2)  \sqrt {t \ln t} + \bar\ell \sqrt{\frac{t}{2}\ln\frac{1}{\delta}}.
\end{equation*}
Therefore, since $|I_t|\leq 1+\ln t$, by union bound with probability $1-\frac{1}{t^2}$ we obtain that for all $i\in I_t$,
\begin{equation*}
    \sum_{s=t_i}^t \ell(\hat Y_s,Y_s) \leq L_{t,i} + (3/2+\bar \ell^2)  \sqrt {t \ln t} + \bar\ell\sqrt{\frac{t}{2}\ln (1+\ln t)}+ \bar\ell \sqrt{t\ln t}\leq (2+\bar\ell+\bar\ell^2)\sqrt {t\ln t},
\end{equation*}
for all $t\geq t_1$ where $t_1\geq t_0$ is a fixed constant. The Borel-Cantelli lemma implies that almost surely, there exists $\hat t\geq 0$ such that
\begin{equation*}
     \forall t\geq \hat t, \forall i\in I_t,\quad \sum_{s=t_i}^t \ell(\hat Y_s,Y_s) \leq  L_{t,i} +(2+\bar\ell+\bar\ell^2)\sqrt {t\ln t}.
\end{equation*}
This ends the proof of the lemma.

\section{Proofs of \cref{sec:alternative}}\label{appB}
\subsection{Proof of \cref{thm:negative_optimistic}}

We start by checking that with the defined loss $(\Nbb,\ell)$ is indeed a metric space $(\Nbb,\ell)$. We only have to check that the triangular inequality is satisfied, the other properties of a metric being directly satisfied. By construction, the loss has values in $\{0,\frac{1}{2},1\}$. Now let $i,j,k\in \Nbb$. The triangular inequality $\ell(i,j)\leq \ell(i,k)+\ell(k,j)$ is directly satisfied if two of these indices are equal. Therefore, we can suppose that they are all distinct and as a result $\ell(i,j),\ell(i,k),\ell(k,j) \in\{\frac{1}{2},1\}.$ Therefore
\begin{equation*}
    \ell(i,j)\leq 1 \leq \ell(i,k)+\ell(k,j),
\end{equation*}
which ends the proof that $\ell$ is a metric.

Now let $(\Xcal,\Bcal)$ be a separable metrizable Borel space. Let $\Xbb\notin\cs$. We aim to show that universal online learning under adversarial responses is not achievable under $\Xbb$. Because $\Xbb\notin\cs$, there exists a sequence of decreasing measurable sets $\{A_i\}_{i\geq 1}$ with $A_i \downarrow \emptyset$ such that $\Ebb[\hat\mu_\Xbb(A_i)]$ does not converge to $0$ for $i\to\infty$. In particular, there exist $\epsilon>0$ and an increasing subsequence $(i_l)_{l\geq 1}$ such that $\Ebb[\hat \mu_\Xbb(A_{i_l})]\geq \epsilon$ for all $l\geq 1$. We now denote $B_l:=A_{i_l}\setminus A_{i_{l+1}}$ for any $l\geq 1$. Then $\{B_l\}_{l\geq 1}$ forms a sequence of disjoint measurable sets such that
\begin{equation*}
    \Ebb\left[\hat \mu_\Xbb \left(\bigcup_{l'\geq l} B_{l'}\right)\right] \geq \epsilon,\quad l\geq 1.
\end{equation*}
Therefore, for any $l\geq 1$ because $\Ebb\left[\hat \mu_\Xbb \left(\bigcup_{l'\geq l} B_{l'}\right)\right] \leq \Pbb\left[\hat \mu_\Xbb \left(\bigcup_{l'\geq l} B_{l'}\right)\geq \frac{\epsilon}{2}\right] + \frac{\epsilon}{2}$ we obtain
\begin{equation*}
    \Pbb\left[\hat \mu_\Xbb \left(\bigcup_{l'\geq l} B_{l'}\right)\geq \frac{\epsilon}{2}\right] \geq \frac{\epsilon}{2}.
\end{equation*}
Now because $\hat\mu$ is increasing we obtain
\begin{align*}
    \Pbb\left[\hat \mu_\Xbb \left(\bigcup_{l'\geq l} B_{l'}\right)\geq \frac{\epsilon}{2},\forall l\geq 1 \right] &= \lim_{L\to\infty} \Pbb\left[\hat \mu_\Xbb \left(\bigcup_{l'\geq l} B_{l'}\right)\geq \frac{\epsilon}{2},1\leq l\leq L \right] \\
    &= \lim_{L\to\infty}\Pbb\left[\hat \mu_\Xbb \left(\bigcup_{l'\geq L} B_{l'}\right)\geq \frac{\epsilon}{2}\right]\geq \frac{\epsilon}{2}.
\end{align*}
We will denote by $\Acal$ this event in which for all $l\geq 1$, we have $\hat \mu_\Xbb \left(\bigcup_{l'\geq l} B_{l'}\right)\geq \frac{\epsilon}{2}$. Under the event $\Acal$, for any $l,t^0\geq 1$, there always exists $t^1>t^0$ such that $\frac{1}{t^1}\sum_{t=1}^{t^1}\1_{\bigcup_{l'\geq l}B_{l'}}(X_t) \geq \frac{3\epsilon}{8}.$ We construct a sequence of times $(t_p)_{p\geq 1}$ and indices $(l_p)_{p\geq 1}$, $(u_p)_{p\geq 1}$ by induction as follows. We first pose $u_0=t_0=0$. Now assume that for $p\geq 1$, the time $t_{p-1}$ and index $u_{p-1}$ are defined. We first construct an index $l_p>u_{p-1}$ such that
\begin{equation*}
    \Pbb\left[\Xbb_{\leq t_{p-1}}\cap \left(\bigcup_{l\geq l_p} B_l\right)\neq \emptyset\right]\leq \frac{\epsilon}{2^{p+3}}.
\end{equation*}
We will denote by $\Ecal_p$ the complementary of this event. Note that finding such index $l_p$ is possible because the considered events $\{\Xbb_{\leq t_{p-1}}\cap \left(\bigcup_{l'\geq l} B_{l'}\right)\neq \emptyset\}$ are decreasing as $l>u_{p-1}$ increases and we have $\bigcap_{l>u_{p-1}}\left\{ \Xbb_{\leq t_{p-1}}\cap \left(\bigcup_{l'\geq l} B_{l'} \right)\neq\emptyset \right\} = \left\{ \Xbb_{\leq t_{p-1}} \cap \left(\bigcap_{l>u_{p-1}}\bigcup_{l'\geq l}B_{l'}\right)\neq\emptyset \right\}=\emptyset.$ We then construct $t_p>t_{p-1}$ such that
\begin{equation*}
    \Pbb\left[\Acal^c \cup \bigcup_{t_{p-1}<t\leq t_p}\left\{\frac{1}{t}\sum_{u=1}^t \1_{\bigcup_{l\geq  l_p}B_l}(X_u) \geq \frac{3\epsilon}{8}\right\}\right]\geq 1-\frac{\epsilon}{2^{p+4}}.
\end{equation*}
This is also possible because $\Acal\subset \bigcup_{t>\frac{8}{\epsilon}t_{p-1}}\left\{\frac{1}{t}\sum_{u=1}^t \1_{\bigcup_{l\geq l_p}B_l}(X_u) \geq \frac{3\epsilon}{8}\right\}$. Last, we can now construct $u_p\geq l_p$ such that 
\begin{equation*}
    \Pbb\left[\Acal^c \cup \bigcup_{t_{p-1}<t\leq t_p}\left\{\frac{1}{t}\sum_{u=1}^t \1_{\bigcup_{l_p\leq l\leq u_p}B_l}(X_u) \geq \frac{\epsilon}{4}\right\}\right]\geq 1-\frac{\epsilon}{2^{p+3}},
\end{equation*}
which is possible using similar arguments as above. We denote $\Fcal_p$ this event. This ends the recursive construction of times $t_p$ and indices $l_p$ for all $p\geq 1$. Note that by construction, $\Pbb[\Ecal_p^c],\Pbb[\Fcal_p^c]\leq \frac{\epsilon}{2^{p+3}}$. Hence, by union bound, the event $\Acal\cap\bigcap_{p\geq 1}(\Ecal_p\cap \Fcal_p)$ has probability $\Pbb[\Acal\cap\bigcap_{p\geq 1}(\Ecal_p \cap \Fcal_p)]\geq \Pbb[\Acal]-\frac{\epsilon}{4}\geq \frac{\epsilon}{4}$. To simplify the rest of the proof, we denote $\tilde B_p = \bigcup_{l_p\leq l\leq u_p} B_l$ for any $p\geq 1$. Also, for any $t_1\leq t_2$, we  denote by
\begin{equation*}
    N_p(t_1,t_2) = \sum_{t=t_1}^{t_2}\1_{\tilde B_p}(X_t)
\end{equation*}
the number of times that set $\tilde B_p$ has been visited between times $t_1$ and $t_2$.

We now fix a learning rule $f_\cdot$ and construct a process $\Ybb$ for which consistency will not be achieved on the event $\Acal\cap\bigcap_{p\geq 1}(\Ecal_p\cap\Fcal_p)$. Precisely, we first construct a family of processes $\Ybb^b$ indexed by a sequence of binary digits $ b=(b_i)_{i\geq 1}$. The process $\Ybb^b$ is defined such that for any $p\geq 1$, and for all $t_{p-1}<t\leq t_p$,
\begin{equation*}
    Y_t^b := \begin{cases}
        n_{t_p} + 4u_p(t) +2b_{i(p,u_p(t))}+b_{i(p,u_p(t))+1} &\text{if } X_t\in \tilde B_p,\\
        n_{t_{p'}}+4t_{p'} +\{b_{i(p',t_{p'}-1)}\ldots b_{i(p',1)} b_{i(p',0)} \}_2 &\text{if } X_t\in \tilde B_{p'}, p'<p,  \\
        0 &\text{otherwise},
    \end{cases}
\end{equation*}
where we denoted $u_p(t)=N_p(t_{p-1}+1,t-1)$ and posed for any $p\geq 1$ and $u\geq 1$:
\begin{equation*}
    i(p,u) = 2\sum_{p'<p} t_{p'} + 2u.
\end{equation*}
Note in particular that conditionally on $\Xbb$, $\Ybb^b$ is deterministic: it does not depends on the random predictions of the learning rule. Because we always have $N_p(t_{p-1}+1,t-1)\leq t_p$ for any $t\leq t_p$, the process is designed so that we have $Y^b_t\in I_{t_p}$ if $X_t\in \tilde B_p$ and $t_{p-1}<t\leq t_p$. Further, for $t_{p-1}<t\leq t_p$, if $X_t\in \bigcup_{p'<p}\tilde B_{p'}$ then $Y^b_t\in J_{t_{p'}}$.
We now consider an i.i.d. Bernoulli $\Bcal(\frac{1}{2})$ sequence of random bits $\mb b$ independent from the process $\Xbb$---and any learning rule predictions. We analyze the responses of the learning rule for responses $\Ybb^{\mb b}$. We first fix a realization $\mb x$ of the process $\Xbb$, which falls in the event $\Acal\cap\bigcap_{p\geq 1}(\Ecal_p\cap\Fcal_p)$. For any $p\geq 1$ we define $\Tcal_p:=\{t_{p-1}<t\leq t_p:\;x_t\in \tilde B_p\}$. For simplicity of notation, for any $t\in \Tcal_p$ we denote $i(t)=i(p,u_p(t))$. We will also denote $\hat Y_t:=f_t({\mb x}_{<t},\Ybb^{\mb b}_{<t},x_t)$. Last, denote by $r_t$ the possible randomness used by the learning rule $f_t$ at time $t$. For any $t\in \Tcal_p$, we have
\begin{align*}
    \Ebb_{\mb b,\mb r}&\ell(\hat Y_t,Y^{\mb b}_t)
    =  \Ebb_{\{b_{i(p',u')},b_{i(p',u')+1},\; p'\leq p,u'\leq t_{p'}\}\cup\{r_{t'},t'\leq t\}} \ell(\hat Y_t,Y^{\mb b}_t)\\
    &= \Ebb \left[ \Ebb_{b_{i(t)},b_{i(t)+1}}\left. \ell(\hat Y_t,Y^{\mb b}_t) \right| b_{i(t')},b_{i(t')+1},t'< t,t'\in\Tcal_{p};\;b_i, i< i(p,0);\;r_{t'},t'\leq t\right]\\
    &=  \Ebb \left[ \Ebb_{b_{i(t)},b_{i(t)+1}}\left. \ell(\hat Y_t,Y^{\mb b}_t) \right|\hat Y_t\right]\\
    &= \Ebb_{\hat Y_t}\left[ \frac{1}{4}\sum_{m=0}^3\ell(\hat Y_t,n_{t_p}+4u_p(t)+m)\right]\\
    &=  \Ebb_{\hat Y_t}\left[ \1_{\hat Y_t\notin \{n_{t_p}+4u_p(t)+m,0\leq m\leq 3\}\cup J_{t_p}} +\frac{3}{4}\1_{\hat Y_t\in \{n_{t_p}+4u_p(t)+m,0\leq m\leq 3\}}+ \frac{3}{4}\1_{\hat Y_t\in J_{t_p}}\right]\\
    &\geq \frac{3}{4}.
\end{align*}
where in the last equality, we used the fact that if $j\in J_{k(t)}$ then by construction $\ell(j,n_{t_p}+4u_p(t))=\ell(j,n_{t_p}+4u_p(t)+1)$, $\ell(j,n_{t_p}+4u_p(t)+2)=\ell(j,n_{t_p}+4u_p(t)+3)$, and $\{\ell(j,n_{t_p}+4u_p(t)),\ell(j,n_{t_p}+4u_p(t)+2)\}=\{\frac{1}{2},1\}$. Summing all equations, we obtain for any $t_{p-1}<T\leq t_p$,
\begin{equation*}
    \Ebb_{\mb b,\mb r}\left[ \sum_{t=1}^{T}\ell(f_t({\mb x}_{<t},\Ybb^{\mb b}_{<t},x_t),Y^{\mb b}_t) \right] \geq \frac{3}{4}\sum_{p'<p}|\Tcal_{p'}| + \frac{3}{4}|\Tcal_p\cap\{t\leq T\}|.
\end{equation*}
This holds for all $p\geq 1$. Let us now compare this loss to the best prediction of a fixed measurable function. Specifically, for any binary sequence $b$, we consider the following function $f^b:\Xcal\to\Nbb$:
\begin{equation*}
    f^b(x) = \begin{cases}
    n_{t_p}+4t_p +\{b_{i(p,t_p-1)}\ldots b_{i(p,1)} b_{i(p,0)} \}_2 &\text{if }x\in \tilde B_p\\
    0 &\text{if }x\notin\bigcup_{p\geq 1} \tilde B_p.
    \end{cases}
\end{equation*}
Now let $t_{p-1}<t\leq t_p$ and $p\geq 1$. If $x_t\in \bigcup_{p'<p}\tilde B_{p'}$ we have $f^{\mb b}(x_t)=Y_t^{\mb b}$, hence $\ell(f^{\mb b}(x_t),Y_t^{\mb b})=0$. Now if $X_t\in\tilde B_p$ by construction we have $\ell(f^{\mb b}(x_t),Y_t^{\mb b})=\frac{1}{2}$. Finally, observe that because the event $\Ecal_{p+1}$ is satisfied by $\mb x$ there does not exist $t_{p-1}<t\leq t_p$ such that $t\in \bigcup_{p'>p}\tilde B_{p'}\subset \bigcup_{l\geq l_{p+1}}B_l$. As a result, we have $\ell(f^{\mb b}(x_t),Y_t^{\mb b})=\frac{1}{2}\1_{t\in\Tcal_p}$ for any $t_{p-1}<t\leq t_p$. Thus, we obtain for any $t_{p-1}<T\leq t_p$,
\begin{equation*}
    \Ebb_{\mb b,\mb r}\left[ \sum_{t=1}^{T}\ell(\hat Y_t,Y^{\mb b}_t) - \ell(f^{\mb b}(X_t),Y^{\mb b}_t)\right] \geq \frac{1}{4}\sum_{p'\leq p}|\Tcal_{p'}| + \frac{1}{4}|\Tcal_p\cap\{t\leq T\}|\geq \frac{1}{4}|\Tcal_p\cap\{t\leq T\}|.
\end{equation*}
Recall that the event $\Fcal_p$ is satisfied by $\mb x$ for any $p\geq 1$. Therefore, there exists a time $t_{p-1}<T_p\leq t_p$ such that $\sum_{t=1}^{T_p} \1_{\tilde B_p}(x_t)\geq \frac{\epsilon T_p}{4}.$ Then, note that because the event $\Ecal_p$ is satisfied, we have $\sum_{t=1}^{t_{p-1}} \1_{\tilde B_p}(x_t)=0$. Therefore, we obtain $|\Tcal_p\cap\{t\leq T_p\}|\geq \frac{\epsilon T_p}{4}$, and as a result,
\begin{equation*}
    \Ebb_{\mb b,\mb r}\left[\frac{1}{T_p}\sum_{t=1}^{T_p}\ell(\hat Y_t,Y^{\mb b}_t) - \ell(f^{\mb b}(X_t),Y^{\mb b}_t)\right] \geq \frac{\epsilon}{16}.
\end{equation*}
Because this holds for any $p\geq 1$ and as $p\to\infty$ we have $T_p\to\infty$, we can now use Fatou lemma which yields
\begin{equation*}
    \Ebb_{\mb b,\mb r}\left[\limsup_{T\to\infty}\frac{1}{T}\sum_{t=1}^{T}\ell(\hat Y_t,Y^{\mb b}_t) - \ell(f^{\mb b}(X_t),Y^{\mb b}_t)\right]\geq \frac{\epsilon}{16}.
\end{equation*}
This holds for any realization in $\Acal\cap\bigcap_{p\geq 1}(\Ecal_p\cap\Fcal_p)$ which we recall has probability at least $\frac{\epsilon}{4}$. Therefore we finally obtain
\begin{equation*}
    \Ebb_{\mb b,\mb r,\Xbb} \left[\limsup_{T\to\infty}\frac{1}{T}\sum_{t=1}^{T}\ell(\hat Y_t,Y^{\mb b}_t) - \ell(f^{\mb b}(X_t),Y^{\mb b}_t)\right]\geq \frac{\epsilon^2}{2^6}.
\end{equation*}
As a result, there exists a specific realization of $\mb b$ which we denote $b$ such that
\begin{equation*}
    \Ebb_{\mb r,\Xbb} \left[\limsup_{T\to\infty}\frac{1}{T}\sum_{t=1}^{T}\ell(\hat Y_t,Y^b_t) - \ell(f^{b}(X_t),Y^b_t)\right]\geq \frac{\epsilon^2}{2^6},
\end{equation*}
which shows that with nonzero probability $   \limsup_{T\to\infty}\frac{1}{T}\sum_{t=1}^{T}\ell(\hat Y_t,Y^b_t) - \ell(f^b(X_t),Y^b_t)> 0.$ This ends the proof of the theorem. As a remark, one can note that the construction of our bad example $\Ybb^b$ is a deterministic function of $\Xbb$: it is independent from the realizations of the randomness used by the learning rule.

\subsection{Proof of \cref{lemma:bandits_Nbb}}

We first construct our online learning algorithm, which is a simple variant of the classical exponential forecaster. We first define a step $\eta:=\sqrt{2\ln t_0/ t_0}$. At time $t=1$ we always predict $0$. For time step $t\geq 2$, we define the set $S_{t-1}:=\{y\in\Nbb, \sum_{u=1}^{t-1} \1_{y= y_u}>0\}$ the set of values which have been visited. Then, we construct weights for all $y\in\Nbb$ such that
\begin{equation*}
    w_{y,t-1} = \begin{cases}
        e^{\eta\sum_{u=1}^{t-1} \1_{y= y_u}}, &y\in S_{t-1}\\
        0 &\text{otherwise},
        \end{cases}
\end{equation*}
and output a randomized prediction independent of the past history such that
\begin{equation*}
    \Pbb(\hat y_t=y) = \frac{w_{y,t-1}}{\sum_{y'\in \Nbb}w_{y',t-1}}.
\end{equation*}
This defines a proper online learning rule. Note that the denominator is well defined since $w_{y,t-1}$ is non-zero only for values in $S_{t-1}$, which contains at most $t-1$ elements. We now define the expected success at time $1\leq t\leq T$ as $\hat s_t:= \frac{w_{y_t,t-1}}{\sum_{y\in\Nbb} w_{y,t-1}}\1_{y_t\in S_t}.$ Note that $\hat s_t=\Ebb[\1_{f_t({\mb y}_{\leq t-1})=y_t}]$. We first show that we have
\begin{equation*}
    \sum_{t=1}^T \hat s_t\geq \max_{y\in \Nbb} \sum_{t=1}^T \1_{y= y_t} - \sqrt T \ln T.
\end{equation*}
To do so, we analyze the quantity $W_t:=\frac{1}{\eta}\ln\left(\sum_{y\in S_t} e^{\eta\sum_{u=1}^t (\1_{y=y_u}-\hat s_u)}\right)$. Let $2\leq t\leq T$. Supposing that $y_t\in S_{t-1}$, i.e., $S_t=S_{t-1}$, we define the operator $\Phi:\mb x\in \Rbb^{|S_{t-1}|}\mapsto \frac{1}{\eta}\ln\left(\sum_{y\in S_{t-1}} e^{\eta x_y}\right)$ and use the Taylor expansion of $\Phi$ to obtain
\begin{align*}
    W_t &= \frac{1}{\eta}\ln\left(\sum_{y\in S_{t-1}} e^{\eta\sum_{u=1}^{t-1} (\1_{y=y_u}-\hat s_u) + \eta(\1_{y=y_t}-\hat s_t)}\right)\\
    &= W_{t-1} + \sum_{y\in S_{t-1}}  (\1_{y=y_t}-\hat s_t)\frac{e^{\eta\sum_{u=1}^{t-1} \1_{y=y_u}}}{\sum_{y'\in S_{t-1}}e^{\eta\sum_{u=1}^{t-1} \1_{y'=y_u} }} \\
    &\quad\quad\quad\quad+ \frac{1}{2} \sum_{y_1,y_2\in S_{t-1}} \left.\frac{\partial^2 \Phi}{\partial x_{y_1}\partial x_{y_2}}\right|_{\xi} (\1_{y_1=y_u}-\hat s_u)(\1_{y_2=y_u}-\hat s_u)\\
    &= W_{t-1} + \frac{1}{2} \sum_{y_1,y_2\in S_{t-1}} \left.\frac{\partial^2 \Phi}{\partial x_{y_1}\partial x_{y_2}}\right|_{\xi} (\1_{y_1=y_t}-\hat s_u)(\1_{y_2=y_t}-\hat s_u)\\
    &\leq  W_{t-1} + \frac{1}{2} \sum_{y\in S_{t-1}} \frac{\eta e^{\eta \xi_y}}{\sum_{y'\in S_{t-1}}e^{\eta \xi_{y'}}} (\1_{y=y_t}-\hat s_u)^2\\
    &\leq W_{t-1} + \frac{\eta}{2},
\end{align*}
for some vector $\xi\in\Rbb^{|S_{t-1}|}$, where in the last inequality we used the fact $|\1_{y=y_t}-\hat s_u|\leq 1$. We now suppose that $y_t\notin S_{t-1}$ and $W_{t-1}\geq 1+\frac{\ln 2+2\ln \frac{1}{\eta}}{\eta}$. In that case, $e^{\eta W_t}=e^{\eta W_{t-1}} + e^{\eta(1-\sum_{u=1}^{t-1}\hat s_u)}.$ Hence, we obtain
\begin{equation*}
    W_t= W_{t-1} + \frac{\ln\left(1+e^{\eta(1-W_{t-1}-\sum_{u=1}^{t-1}\hat s_u)}\right)}{\eta} \leq W_{t-1} + \frac{e^{\eta(1-W_{t-1})}}{\eta}\leq W_{t-1} + \frac{\eta}{2}.
\end{equation*}
Now let $l=\max\{1\}\cup\left\{1\leq t\leq T:W_t < 1+\frac{\ln 2+2\ln \frac{1}{\eta}}{\eta}\right\}$. Note that for any $l< t\leq T$ the above arguments yield $W_t \leq W_{t-1} + \frac{\eta}{2}$. As a result, noting that $W_1\leq 1$, we finally obtain
\begin{equation*}
    W_T \leq W_l + \eta\frac{T-l}{2} \leq 1+\frac{\ln 2+2\ln \frac{1}{\eta}}{\eta} + \eta\frac{T}{2} \leq 1+ \ln 2\sqrt{\frac{t_0}{2\ln t_0}} + \sqrt{\frac{\ln t_0}{2t_0}} (t_0 + T). 
\end{equation*}
Therefore, for any $y\in S_T$, we have
\begin{equation*}
    \sum_{t=1}^T (\1_{y=y_t}-\hat s_t) \leq W_T\leq 1+\ln 2\sqrt{\frac{t_0}{2\ln t_0}} + \sqrt{\frac{\ln t_0}{2t_0}} (t_0 + T). 
\end{equation*}
In particular, this shows that
\begin{equation*}
    \sum_{t=1}^T \hat s_t \geq \max_{y\in S_T}\sum_{t=1}^T\1_{y=y_t} - 1- \ln 2\sqrt{\frac{t_0}{2\ln t_0}} - \sqrt{\frac{\ln t_0}{2t_0}} (t_0 + T). 
\end{equation*}
Now note that if $y\notin S_T$, then $ \sum_{t=1}^T \1_{y=y_t}=0$, which yields $\max_{y\in S_T}\sum_{t=1}^T\1_{y=y_t} = \max_{y\in \Nbb}\sum_{t=1}^T\1_{y=y_t}$. For the sake of conciseness, we will now denote by $\hat y_t$ the prediction of the online learning rule at time $t$. We observe that the variables $\1_{\hat y_t=y_t}-\hat s_t$ for $1\leq t\leq T$ form a sequence of martingale differences. Further, $|\1_{\hat y_t=y_t}-\hat s_t|\leq 1$. Therefore, the Hoeffding-Azuma inequality shows that with probability $1-\delta$,
\begin{equation*}
    \sum_{t=1}^T (\1_{\hat y_t=y_t}-\hat s_t) \geq -\sqrt{2T\ln \frac{1}{\delta}}.
\end{equation*}
Putting everything together yields that with probability $1-\delta$,
\begin{align*}
    \sum_{t=1}^T \1_{\hat y_t=y_t} &\geq \sum_{t=1}^T \hat s_t -\sqrt{2T\ln \frac{1}{\delta}}\\
    &\geq \max_{y\in \Nbb} \sum_{t=1}^T \1_{y=y_t} -  1-\ln 2\sqrt{\frac{t_0}{2\ln t_0}} - \sqrt{\frac{\ln t_0}{2t_0}} (t_0 + T) -\sqrt{2T\ln \frac{1}{\delta}}.
\end{align*}
This ends the proof of the lemma.

\subsection{Proof of \cref{thm:countable_classification}}

We use a similar learning rule to the one constructed in \cref{sec:totally_bounded_value_spaces} for totally-bounded spaces. We only make a slight modification of the learning rules $f^\epsilon_\cdot$. Precisely, we pose for $0<\epsilon\leq 1$,
\begin{equation*}
    T_\epsilon := \left\lceil \frac{2^4\cdot 3^2 (1+\ln \frac{1}{\epsilon})}{\epsilon^2}\right\rceil \quad \text{and} \quad \delta_\epsilon:= \frac{\epsilon}{2T_\epsilon}.
\end{equation*}
Then, let $\phi$ be the representative function from the $(1+\delta_\epsilon)$C1NN learning rule. Similarly as for the $\epsilon-$approximation learning rule from \cref{sec:totally_bounded_value_spaces}, we consider the same equivalence relation $\stackrel \phi \sim$ on times to define clusters. The learning rule then performs its prediction based on the values observed on the corresponding cluster using the learning rule from \cref{lemma:bandits_Nbb} using $t_0 = T_\epsilon$. Precisely, let $\eta_\epsilon:=\sqrt{2\ln T_\epsilon/ T_\epsilon}$ and define the weights $w_{y,t}=e^{\eta_\epsilon\sum_{u<t:u\stackrel \phi \sim t} \1(Y_u=y)}$ for all $y\in \tilde S:=\{y'\in\Nbb:\sum_{u<t:u\stackrel \phi \sim t} \1(Y_u=y')>0\}$ and $w_{y,t}=0$ otherwise. The learning rule $f^\epsilon_t(\Xbb_{\leq t-1},\Ybb_{\leq t-1},X_t)$ outputs a random value in $\Nbb$ independent of the past history such that
\begin{equation*}
    \Pbb(\hat Y_t=y) = \frac{w_{y,t}}{\sum_{y'\in \Nbb} w_{y',t} },\quad y\in \Nbb.
\end{equation*}
The final learning rule $f_\cdot$ is then defined similarly as before from the learning rules $f^\epsilon_\cdot$ for $\epsilon>0$. Therefore, \cref{lemma:concatenation_predictors} still holds. Also, using the same notations as in the proof of \cref{thm:optimistic_regression_totally_bounded}, \cref{lemma:bandits_Nbb} implies that for any $t\geq 1$, we can write for any $t\geq 1$ on the cluster $\Ccal(t) = \{u<t:u\stackrel \phi\sim t\}$,
\begin{align*}
     \sum_{u\in \Ccal(t)} &\bar\ell_{01}(\hat Y_u(\epsilon),Y_u) \leq \min_{y\in\Nbb} \sum_{u\in \Ccal(t)} \ell_{01}(y,Y_u) + 1+\ln 2\sqrt{\frac{T_\epsilon}{2\ln T_\epsilon}} + \sqrt{\frac{\ln T_\epsilon}{2T_\epsilon}} (T_\epsilon + |\Ccal(t)|)\\
     &\leq \min_{y\in\Nbb} \sum_{u\in \Ccal(t)} \ell_{01}(y,Y_u) + \left(\frac{1}{T_\epsilon} +\frac{\ln 2}{\sqrt{2 T_\epsilon \ln T_\epsilon}} + \sqrt{\frac{2\ln T_\epsilon}{T_\epsilon}} \right) \max(T_\epsilon ,|\Ccal(t)|)\\
     &\leq \min_{y\in\Nbb} \sum_{u\in \Ccal(t)} \ell_{01}(y,Y_u) + \left(\frac{\epsilon}{3} + \frac{\epsilon}{3} + \frac{\epsilon}{3}\right) \max(T_\epsilon ,|\Ccal(t)|)\\
     &= \min_{y\in\Nbb} \sum_{u\in \Ccal(t)} \ell_{01}(y,Y_u) + \epsilon \max(T_\epsilon ,|\Ccal(t)|)
\end{align*}
Therefore, the same proof of \cref{thm:optimistic_regression_totally_bounded} holds by replacing all $\epsilon-$nets $\Ycal_\epsilon$ directly by $\Nbb$. The martingale argument still holds since the learning rule used is indeed online. This ends the proof of this theorem.

\subsection{Proof of \cref{thm:good_value_spaces}}

We first need the following simple result which intuitively shows that we can assume that the predictions of the learning rule for mean estimation $g^\epsilon_{\leq t_\epsilon}$ are unrelated for $t=1,\ldots,t_\epsilon$.
\begin{lemma}\label{lemma:independent_predictions}
    Let $(\Ycal,\ell)$ satisfying $\ftime$. For any $\eta>0$, there exists a horizon time $T_\eta\geq 1$, an online learning rule $g_{\leq T_\eta}$ such that for any $\mb y:=(y_t)_{t=1}^{T_\eta}$ of values in $\Ycal$ and any value $y\in\Ycal$, we have
\begin{equation*}
    \frac{1}{T_\eta}\Ebb\left[\sum_{t=1}^{T_\eta} \ell(g_t({\mb y}_{\leq t-1}), y_t) - \ell(y,y_t) \right] \leq \eta,
\end{equation*}
    and such that the random variables $g_t({\mb y}_{\leq t-1})$ are independent.
\end{lemma}

\begin{proof}
    Fix $\eta>0$, $T_\eta\geq 1$ and $g_{\leq T_\eta}$ such that this online learning rule satisfies the condition from $\ftime$ for $\eta>0$. We consider the following learning rule $\tilde g_\cdot$. For any $t\geq 1$ and $\mb y\in\Ycal^{t-1}$,
    \begin{equation*}
        \tilde g_t(\mb y_{\leq t-1}) = g^t_t(\mb y_{\leq t-1}),
    \end{equation*}
    where $(g^t_\cdot)$ are i.i.d. samples of the learning rule $g_\cdot$. By construction, we still have that for any sequence $\mb y_{T_\eta}\in\Ycal^{T_\eta}$,
    \begin{equation*}
        \frac{1}{T_\eta}\Ebb\left[\sum_{t=1}^{T_\eta} \ell(\tilde g_t({\mb y}_{\leq t-1}), y_t) - \ell(y,y_t) \right] = \frac{1}{T_\eta}\Ebb\left[\sum_{t=1}^{T_\eta} \ell(g_t({\mb y}_{\leq t-1}), y_t) - \ell(y,y_t) \right] \leq \eta.
    \end{equation*}
    This ends the proof of the lemma.
\end{proof}

From now on, by \cref{lemma:independent_predictions}, we will suppose without loss of generality that the learning rule $g^\epsilon$ has predictions that are independent at each step (conditionally on the observed values). For simplicity, we refer to the prediction of the defined learning rule $f_\cdot$ (resp. $f^\epsilon_\cdot$) at time $t$ as $\hat Y_t$ (resp. $\hat Y_t(\epsilon)$). We now show that is optimistically universal for arbitrary responses. By construction of the learning rule $f_\cdot$, \cref{lemma:concatenation_predictors} still holds. Therefore, we only have to focus on the learning rules $f^\epsilon_\cdot$ and prove that we obtain similar results as before. Let $T\geq 1$ and denote by $\Acal_i:=\{t\leq T:|\{u\leq T:\phi(u)=t\}|=i\}$ the set of times which have exactly $i$ children within horizon $T$ for $i=0,1,2$. Then, we define
\begin{equation*}
    \Bcal_T = \{t\leq T: L_t=0 \text{ and } |\{t<u\leq T:u\stackrel \phi\sim t\}|\geq t_\epsilon\},
\end{equation*}
i.e., times that start a new learning block and such that there are at least $t_\epsilon$ future times falling in their cluster within horizon $T$. Note that the function $\psi$ defines a parent-relation (similarly to $\phi$, but defined for all times $t\geq 1$). To simplify notations, for any $t\in\Bcal_T$, we denote $t^u$ the $\psi-$children of $t$ at generation $u-1$ for $1\leq u\leq t_\epsilon$, i.e., we have $\psi^{u-1}(t^u) = t$ for all $1\leq u\leq t_\epsilon$. In particular $t=t^1$. By construction, blocks have length at most $t_\epsilon$. More precisely, the block started at any $t\in\Bcal_T$ has had time to finish completely, hence has length exactly $t_\epsilon$. By construction of the indices $L_t$, the blocks $\{t^u,1\leq u\leq t_\epsilon\}$, for $t\in \Bcal_T$, are all disjoint. This implies in particular $|\Bcal_T|t_\epsilon \leq T$. We first analyze the predictions along these blocks and for any $t\in\Bcal_T$ and $y\in\Ycal$, we pose $\delta_{t}(y):=\frac{1}{t_\epsilon}\sum_{u=1}^{t_\epsilon}  \left(\ell(\hat Y_{t^u},Y_{t^u}) - \ell(y,Y_{t^u}) -\epsilon\right)$. Now by construction of the learning rule $f^\epsilon_\cdot$, we have
\begin{equation*}
    t_\epsilon \delta_t(y^t) =
    \sum_{u=1}^{t_\epsilon} \left(\ell(g^{\epsilon,t}_u(\{Y_{t^l}\}_{l=1}^{u-1}), Y_{t^u}) - \ell(y^t,Y_{t^u})\right) - \epsilon t_\epsilon.
\end{equation*}
Next, for any $t\leq t_\epsilon$ and sequence $\mb y_{\leq t-1}$ and value $y\in \Ycal$, we write $\bar \ell(g^\epsilon_t(\mb y_{\leq t-1}), y) := \Ebb \left[\ell(g^\epsilon_t(\mb y_{\leq t-1}), y)\right].$ Now by hypothesis on the learning rule $g^\epsilon_{\leq t_\epsilon}$,
\begin{equation}\label{eq:good_mean_estimation}
    \frac{1}{t_\epsilon}\sum_{u=1}^{t_\epsilon}  \bar\ell(\hat Y_{t^u},Y_{t^u}) - \ell(y^t,Y_{t^u})  \leq \epsilon.
\end{equation}
Now consider the following sequence $(\ell(\hat Y_{t^u},Y_{t^u}) -\bar \ell(\hat Y_{t^u},Y_{t^u}))_{t\in\Bcal_T,1\leq u\leq s(t)}$. Because of the definition of the learning rule, which uses i.i.d. copies of the learning rule $g^\epsilon_\cdot$, if we order the former sequence by increasing order of $t^u$, we obtain a sequence of martingale differences. We can continue this sequence by zeros to ensure that it has length exactly $T$. As a result, we obtain a sequence of $T$ martingale differences, which are all bounded by $\bar\ell$ in absolute value. Now, the Azuma-Hoeffding inequality implies that for $\delta>0$, with probability $1-\delta$, we have
\begin{equation*}
    \sum_{t\in\Bcal_T} \sum_{u=1}^{t_\epsilon}\ell(\hat Y_{t^u},Y_{t^u}) \leq \sum_{t\in\Bcal_T} \sum_{u=1}^{t_\epsilon}\bar\ell(\hat Y_{t^u},Y_{t^u}) + \bar\ell\sqrt{2T\ln \frac{1}{\delta}}.
\end{equation*}
Thus, using Eq~\eqref{eq:good_mean_estimation}, with probability at least $1-\delta$,
\begin{equation}\label{eq:replace:lemma}
    \sum_{t\in\Bcal_T}t_\epsilon \delta_t(y^t) \leq \bar\ell\sqrt{2T\ln \frac{1}{\delta}}.
\end{equation}
We also denote $\Tcal = \bigcup_{t\in\Bcal_T} \{t^u,1\leq u\leq t_\epsilon\}$ the union of all blocks within horizon $T$. This set contains all times $t\leq T$ except \emph{bad} times close to the last times of their corresponding cluster $\{u\leq T:u\stackrel \phi\sim t\}$. Precisely, these are times $t$ such that $|\{t<u\leq T:u\stackrel \phi\sim t\}|<  t_\epsilon-L_t$. As a result, there are at most $t_\epsilon$ such times for each cluster. Using the same arguments as in the proof of \cref{thm:optimistic_regression_totally_bounded}, if we consider only clusters of duplicates (i.e., the cluster started for a specific instance which has high number of duplicates), the corresponding \emph{bad} times contribute to a proportion $\leq \frac{t_\epsilon}{T_\epsilon/\epsilon}\leq \epsilon^2$ of times. Now consider clusters that have at least $T_\epsilon$ times. Their \emph{bad} times contribute to a proportion $\leq\frac{t_\epsilon}{T_\epsilon} \leq \epsilon$ of times. Last, we need to account for clusters of size $<T_\epsilon$ which necessarily contain leaves of the tree $\phi$: there are at most $|\Acal_0|$ such clusters. By the Chernoff bound, with probability at least $1-e^{-T\delta_\epsilon/3}$ we have 
\begin{equation*}
    T-|\Tcal| \leq (\epsilon^2 + \epsilon)T + |\Acal_0|t_\epsilon  \leq t_\epsilon + (\epsilon^2 + \epsilon + 2\delta_\epsilon t_\epsilon) T \leq t_\epsilon + 3\epsilon T.
\end{equation*}
By the Borel-Cantelli lemma, because $\sum_{T\geq 1} e^{-T\delta_\epsilon/3}<\infty$, almost surely there exists a time $\hat T$ such that for $T\geq \hat T$ we have $T-|\Tcal| \leq t_\epsilon +  3\epsilon T$. We denote by $\Ecal_\epsilon$ this event. Then, on the event $\Ecal_\epsilon$, for any $T\geq \hat T$ and for any sequence of values $(y^t)_{t\geq 1}$ we have
\begin{align*}
    \sum_{t=1}^{T} \ell(\hat Y_t(\epsilon),Y_t)  &\leq \sum_{t\in\Bcal_T}  \sum_{u=1}^{t_\epsilon} \ell(\hat Y_{t^u},Y_{t^u})  + ( T-|\Tcal| )\bar\ell\\
    &\leq \sum_{t\in\Bcal_T}  \sum_{u=1}^{t_\epsilon} \ell(y^t,Y_{t^u}) + \sum_{t\in\Bcal_T} t_\epsilon \delta_{t}(y^t)+ \epsilon|\Bcal_T| t_\epsilon+ t_\epsilon\bar\ell + 3\epsilon T\\
    &\leq \sum_{t\in\Bcal_T}  \sum_{u=1}^{t_\epsilon} \ell(y^t,Y_{t^u}) + \sum_{t\in\Bcal_T} t_\epsilon\delta_t(y^t)+ t_\epsilon \bar\ell + 4\epsilon T.
\end{align*}
Now let $f:\Xcal\to\Ycal$ be a measurable function to which we compare $f^\epsilon_\cdot$. By \cref{thm:1+deltaC1NN_optimistic}, because $(1+\delta_\epsilon)$C1NN is optimistically universal without noise and $\Xbb\in\soul$, almost surely $ \frac{1}{T}\sum_{t=1}^T \ell(f(X_{\phi(t)}),f(X_t)) \to 0$. We denote by $\Fcal_\epsilon$ this event of probability one. The proof of \cref{thm:optimistic_regression_totally_bounded} shows that on $\Fcal_\epsilon$, for any $0\leq u\leq T_\epsilon-1$ we have
\begin{equation*}
    \frac{1}{T}\sum_{t=1}^T \ell(f(X_{\phi^u(t)}),f(X_t)) \to 0.
\end{equation*}
We let $y^t=f(X_t)$ for all $t\geq 1$. Then, recalling that for any $t\in\Bcal_T$, we have $t = \phi^{u-1}(t^u)$, on the event $\Ecal_\epsilon$, for any $T\geq \hat T$ we have
\begin{align*}
    &\sum_{t=1}^{T} \ell(\hat Y_t(\epsilon),Y_t) \\
    &\leq \sum_{t\in\Bcal_T}  \sum_{u=1}^{t_\epsilon} \left((1+\epsilon)\ell(f(X_{t^u}),Y_{t^u}) + c_\epsilon^\alpha \ell(f(X_t),f(X_{t^u}))\right) + \sum_{t\in\Bcal_T} t_\epsilon \delta_t(y^t)+ t_\epsilon \bar\ell + 4\epsilon T\\
    &\leq \sum_{t=1}^T  \ell(f(X_t),Y_t) +c_\epsilon^\alpha \frac{T_\epsilon}{\epsilon} \sum_{u=0}^{T_\epsilon-1} \sum_{t=1}^T \ell(f(X_{\phi^u(t)}), f(X_t)) + \sum_{t\in\Bcal_T} t_\epsilon \delta_{\varphi(t)}(y^t)+ t_\epsilon\bar\ell + 5\epsilon T,
\end{align*}
where in the first inequality we used \cref{lemma:loss_identity}, and in the second inequality we used the fact that cluster with distinct instance values have at most $\frac{T_\epsilon}{\epsilon}$ duplicates of each instance. Next, using Eq~\eqref{eq:replace:lemma}, with probability $1-\frac{1}{T^2}$, we have
\begin{equation*}
    \sum_{t\in\Bcal_T} t_\epsilon \delta_t(y^t)\leq 2\bar\ell\sqrt{T\ln T}.
\end{equation*}
Because $\sum_{T\geq 1}\frac{1}{T^2}<0$, the Borel-Cantelli lemma implies that on an event $\Gcal_\epsilon$ of probability one, there exists $\hat T_2$ such that for all $T\geq \hat T_2$ the above inequality holds. As a result, on the event $\Ecal_\epsilon\cap\Fcal_\epsilon\cap\Gcal_\epsilon$ we obtain for any $T\geq \max(\hat T,\hat T_2)$ that
\begin{multline*}
    \sum_{t=1}^{T} \ell(\hat Y_t(\epsilon),Y_t) \leq \sum_{t=1}^T  \ell(f(X_t),Y_t) +\frac{ c_\epsilon^\alpha T_\epsilon}{\epsilon} \sum_{u=0}^{T_\epsilon-1} \sum_{t=1}^T \ell(f(X_{\phi^u(t)}), f(X_t))\\
     + 2\bar\ell\sqrt{T\ln T}+ t_\epsilon \bar\ell + 5\epsilon T.
\end{multline*}
where $\frac{1}{T}\sum_{u=0}^{T_\epsilon-1} \sum_{t=1}^T \ell(f(X_{\phi^u(t)}), f(X_t)) \to 0$ because the event $\Fcal_\epsilon$ is met. Therefore, we obtain that on the event $\Ecal_\epsilon\cap\Fcal_\epsilon\cap\Gcal_\epsilon$ of probability one,
\begin{equation*}
    \limsup_{T\to\infty} \frac{1}{T}\sum_{t=1}^{T} \left[\ell(\hat Y_t(\epsilon),Y_t) - \ell(f(X_t),Y_t)\right] \leq 5\epsilon,
\end{equation*}
i.e., almost surely, the learning rule $f^\epsilon_\cdot$ achieves risk at most $5\epsilon$ compared to the fixed function $f$. By union bound, on the event $\bigcap_{i\geq 0}(\Ecal_{\epsilon_i}\cap\Fcal_{\epsilon_i}\cap\Gcal_{\epsilon_i})$ of probability one we have that
\begin{equation*}
    \limsup_{T\to\infty} \frac{1}{T}\sum_{t=1}^{T} \left[\ell(\hat Y_t(\epsilon_i),Y_t) - \ell(f(X_t),Y_t)\right] \leq 5\epsilon_i,\quad \forall i\geq 0.
\end{equation*}
The rest of the proof uses similar arguments as in the proof of \cref{thm:optimistic_regression_totally_bounded}. Precisely, let $\Hcal$ be the almost sure event of \cref{lemma:concatenation_predictors} such that there exists $\hat t$ for which
\begin{equation*}
    \forall t\geq \hat t,\forall i\in I_t,\quad \sum_{s=t_i}^t\ell(\hat Y_t,Y_t) \leq \sum_{s=t_i}^t \ell(\hat Y_t(\epsilon_i),Y_t) + (2+\bar\ell+\bar\ell^2)\sqrt {t\ln t}.
\end{equation*}
In the rest of the proof we will suppose that the event $\Hcal\cap \bigcap_{i\geq 0}( \Ecal_{\epsilon_i}\cap\Fcal_{\epsilon_i}\cap\Gcal_{\epsilon_i})$ of probability one is met. Let $i\geq 0$. For all $t\geq \max(\hat t,t_i)$ we have
\begin{align*}
    \frac{1}{T}\sum_{t=1}^T \ell(\hat Y_t,Y_t)- &\ell(f(X_t),Y_t) \leq \frac{t_i}{T}\bar \ell + \frac{1}{T}\sum_{t=t_i}^T \ell(\hat Y_t,Y_t)- \ell(f(X_t),Y_t)\\
    &\leq \frac{t_i}{T}\bar \ell + \frac{1}{T}\sum_{t=t_i}^T \ell(\hat Y_t(\epsilon_i),Y_t)- \ell(f(X_t),Y_t) + (2+\bar\ell +\bar\ell^2)
    \sqrt{\frac{\ln T}{ T}}\\
    &\leq \frac{1}{T}\sum_{t=1}^T \ell(\hat Y_t(\epsilon_i),Y_t)- \ell(f(X_t),Y_t) + \frac{2t_i}{T}\bar\ell  + (2+\bar\ell +\bar\ell^2)
    \sqrt{\frac{\ln T}{ T}}.
\end{align*}
Therefore we obtain $\limsup_{T\to\infty} \frac{1}{T}\sum_{t=1}^T \ell(\hat Y_t,Y_t)- \ell(f(X_t),Y_t) \leq 5\epsilon_i$. Because this holds for any $i\geq 0$ we finally obtain
\begin{equation*}
    \limsup_{T\to\infty} \frac{1}{T}\sum_{t=1}^T \ell(\hat Y_t,Y_t)- \ell(f(X_t),Y_t) \leq 0.
\end{equation*}
As a result, $f_\cdot$ is universally consistent for adversarial responses under all $\soul$ processes. Hence, $\solar=\soul$ and $f_\cdot$ is in fact optimistically universal. This ends the proof of the theorem.

\subsection{Proof of \cref{lemma:equivalent_conditions}}

We first note that with the same horizon time $T_\eta$, we have that $\ftime$ implies Property 2. We now show that Property 2 implies $\ftime$. Let $(\Ycal,\ell)$ satisfying Property 2. We now fix $\eta>0$ and let $T,g_{\leq \tau}$ such that for any $\mb y:=(y_t)_{t=1}^{T}$ of values in $\Ycal$ and any value $y\in\Ycal$, we have
\begin{equation*}
    \Ebb\left[\frac{1}{\tau}\sum_{t=1}^{\tau} \left(\ell(g_t({\mb y}_{\leq t-1}), y_t) - \ell(y,y_t)\right) \right] \leq \eta.
\end{equation*}
We now construct a random time $1\leq \tilde \tau \leq T$ such that $\Pbb[\tilde \tau=t] = \frac{\Pbb[\tau=t]}{t\Ebb [1/\tau]}$ for all $1\leq t\leq T$. This indeed defines a proper random variable because $\sum_{t=1}^T \frac{\Pbb[\tau=t]}{t\Ebb [1/\tau]}= 1.$ Let $Supp(\tau):=\{1\leq t\leq T:\Pbb[\tau=t]>0\}$ be the support of $\tau$. For any $t\in Supp(\tau)$, we denote by $g^t_{\leq t}$ the learning rule obtained by conditioning $g_{\leq \tau}$ on the event $\{\tau=t\}$, i.e., $g^t_{\leq t}=g_{\leq \tau}|\tau=t$. More precisely, recall that $\tau$ only uses the randomness of $g_t$. It is not an online random time. Hence, a practical way to simulate $g^t_{\leq t}$ for all $t\in Supp(\tau)$ is to first draw an i.i.d. sequence of learning rules $(g_{i,\leq \tau_i})_{i\geq 1}$. Then, for each $t\in Supp(\tau)$ we select the randomness which first satisfies $\tau=t$. Specifically, we define the time $i_t = \min\{i: \tau_i=t\}$ for all $t\in Supp(\tau)$. With probability one, these times are finite for all $t\in Supp(\tau)$. Denote this event $\Ecal$. Then, letting $\bar y\in\Ycal$ be an arbitrary fixed value, for all $1\leq t\leq T$ we pose
\begin{equation*}
    g^t_{\leq t} = \begin{cases}
        g_{i_t,\leq t} &\text{if }\Ecal \text{ is met},\\
        {\bar y}_{\leq t} &\text{otherwise},
    \end{cases}
    \quad t\in Supp(\tau)\quad \text{ and }\quad g^t_{\leq t} = {\bar y}_{\leq t},\quad t\notin Supp(\tau).
\end{equation*}
where ${\bar y}_{\leq t} $ denotes the learning rules which always outputs value $\bar y$ for all steps $u\leq t$. Intuitively, $g^t_{\leq t}$ has the same distribution as $g_{\leq \tau}$ conditioned on the event $\{\tau=t\}$. We are now ready to define a new learning rule $\tilde g_{\leq \tilde \tau}$, by $\tilde g_{\leq \tilde \tau} := g^{\tilde \tau}_{\leq \tilde \tau}.$ Noting that for any $t\notin Supp(\tau)$ we have $\Pbb[\tilde\tau=t]=0$, we can write
\begin{align*}
    \Ebb&\left[\sum_{t=1}^{\tau} \left(\ell(\tilde g_t({\mb y}_{\leq t-1}), y_t) - \ell(y,y_t)\right) -\eta \tau\right]\\
    &= \sum_{t=1}^T \Pbb[\tilde\tau=t] \Ebb\left[\left.\sum_{u=1}^{t} \left(\ell(\tilde g_u({\mb y}_{\leq u-1}), y_u) - \ell(y,y_u)\right) - \eta t\right| \tilde \tau = t\right]\\
    &= \sum_{t\in Supp(\tau)} \Pbb[\tilde\tau=t] \Ebb\left[\left.\sum_{u=1}^{t} \left(\ell(\tilde g_u({\mb y}_{\leq u-1}), y_u) - \ell(y,y_u)\right) - \eta t\right| \tilde \tau = t,\Ecal\right]\\
    &=\frac{1}{\Ebb[1/\tau]}\sum_{t\in Supp(\tau)} \Pbb[\tau=t] \Ebb\left[\left.\frac{1}{t}\sum_{u=1}^{t} \left(\ell(g_{i_t,u}({\mb y}_{\leq u-1}), y_u) - \ell(y,y_u) \right) -\eta \right|\tilde\tau =t,\Ecal\right]\\
    &=\frac{1}{\Ebb[1/\tau]}\sum_{t\in Supp(\tau)} \Pbb[\tau=t] \Ebb\left[\frac{1}{t}\sum_{u=1}^{t} \left(\ell(g_{i_t,u}({\mb y}_{\leq u-1}), y_u) - \ell(y,y_u) \right) -\eta \right]\\
    &=\frac{1}{\Ebb[1/\tau]}\sum_{t\in Supp(\tau)} \Pbb[\tau=t] \Ebb\left[\left. \frac{1}{t}\sum_{u=1}^{t} \left(\ell(g_u({\mb y}_{\leq u-1}), y_u) - \ell(y,y_u) \right) -\eta \right| \tau = t\right]\\
    &=\frac{1}{\Ebb[1/\tau]} \Ebb\left[\frac{1}{\tau}\sum_{t=1}^{\tau} \left(\ell(g_t({\mb y}_{\leq t-1}), y_t) - \ell(y,y_t)\right) -\eta \right]\leq 0.
\end{align*}
where in the second and fourth equality we used the fact that $\Pbb[\Ecal]=1$.
As a result, there exists a learning rule $\tilde g_{\leq \tilde\tau}$ such that $1\leq \tilde\tau\leq T_\eta$, and for any $\mb y_{\leq T_\eta}\in\Ycal^{T_\eta}$ and $y\in\Ycal$ one has
\begin{equation*}
    \Ebb\left[\sum_{t=1}^{\tilde \tau} \left(\ell(\tilde g_t({\mb y}_{\leq t-1}), y_t) - \ell(y,y_t)\right) -\eta \tilde \tau\right] \leq 0.
\end{equation*}

We now pose $T_\eta' = \lceil T_\eta/\eta\rceil$ and draw an i.i.d. sequence of learning rules $(\tilde g^i_{\leq \tilde \tau_i})_{i\geq 1}$. Denote $\theta_i = \sum_{j<i}\tilde\tau_i$ with the convention $\theta_1=0$. We are now ready to define a learning rule $h_{\leq T_\eta'}$ as follows. For any $1\leq t\leq T_\eta'$ and $\mb y_{\leq t}\in\Ycal^t$,
\begin{equation*}
    h_t(\mb y_{\leq t-1}) = \tilde g^i_{\leq t-\theta_i}((y_{t'})_{\theta_i<t'\leq t-1}),\qquad \theta_i<t\leq \theta_{i+1}, i\geq 1.
\end{equation*}
In other words, the learning rule performs independent learning rules $\tilde g_{\leq \tilde \tau}$ and when the time horizon $\tilde \tau$ is reached, we re-initialize the learning rule with a new randomness. Now let $\mb y_{\leq T_\eta'}\in\Ycal^{T_\eta'}$ and $y\in\Ycal$. We denote by $\hat i=\max\{i\geq 1,\theta_i\leq t\}$, the index of the last learning rule which had time to finish completely. Then, because $\tilde \tau_{\hat i}\leq T_\eta$,
\begin{align*}
    \Ebb&\left[ \sum_{t=1}^{T_\eta'} (\ell(h_t(\mb y_{\leq t-1}),y_t) - \ell(y,y_t))-2\eta T_\eta'\right] \\
    &\leq \Ebb\left[\sum_{i\leq \hat i}\sum_{t=1}^{\tilde \tau_i}( \ell(\tilde g^i_{t-\theta_i}(\mb y_{\theta_i<\cdot\leq t-1}),y_t)-\ell(y,y_t)) - \eta T_\eta'\right] -\eta T_\eta' + T_\eta\\
    &\leq \Ebb\left[\sum_{i\leq \hat i}\left(\sum_{t=1}^{\tilde \tau_i}( \ell(\tilde g^i_{t-\theta_i}(\mb y_{\theta_i<\cdot\leq t-1}),y_t)-\ell(y,y_t)) - \eta \tilde \tau_i\right) \right].
\end{align*}
We now analyze the last term. First, note that by construction, the sequence \begin{equation*}
    \left\{S_j:=\sum_{j\leq i} \left(\sum_{t=1}^{\tilde \tau_j}( \ell(\tilde g^j_{t-\theta_j}(\mb y_{\theta_j<\cdot\leq t-1}),y_t)-\ell(y,y_t)) - \eta \tilde \tau_j\right) \right\}_{j\geq 1}
\end{equation*}
is a super-martingale. Now, note that $\hat i\leq 1+T_\eta'$ since for all $i$, $\theta_i=\sum_{j<i}\tau_i\geq i-1$. As a result, $\hat i$ is bounded, is a stopping time for the considered filtration (after finishing period $\hat i$ we stop if and only we exceed time $T_\eta'$) and we can apply Doob's optimal sampling theorem to obtain $\Ebb[S_{\hat i}]\leq 0.$ Thus, combining the above equations gives
\begin{equation*}
    \frac{1}{T_\eta'}\Ebb\left[ \sum_{t=1}^{T_\eta'} (\ell(h_t(\mb y_{\leq t-1}),y_t) - \ell(y,y_t))\right] \leq 2\eta.
\end{equation*}
Because this holds for all $\eta>0$, $\ftime$ is satisfied. This ends the proof of the lemma.

\subsection{Proof of \cref{thm:bad_value_spaces}}

We first prove that adversarial regression for processes outside of $\cs$ is not achievable. Precisely, we show that for any $\Xbb\notin\cs$, for any online learning rule $f_\cdot$, there exists a process $\Ybb$ on $\Ycal$, a measurable function $f^*:\Xcal\to\Ycal$ and $\delta>0$ such that with non-zero probability $\Lcal_{(\Xbb,\Ybb)}(f_\cdot,f^*)>\delta$.

Because $\ftime$ is not satisfied by $(\Ycal,\ell)$, by \cref{lemma:equivalent_conditions}, Property 2 is not satisfied either. Hence, we can fix $\eta>0$ such that for any horizon $T\geq 1$ and any online learning rule $g_{\leq \tau}$ with $1\leq \tau\leq T$, there exist a sequence $\mb y:=(y_t)_{t=1}^T$ of values in $\Ycal$ and a value $y$ such that
\begin{equation*}
    \Ebb\left[\frac{1}{\tau}\sum_{t=1}^{\tau} \left(\ell(g_t({\mb y}_{\leq t-1}), y_t) - \ell(y,y_t)\right)\right] >\eta,
\end{equation*}
as in the assumption of the space $(\Ycal,\ell)$. Let $\Xbb\notin\cs$. The proof of \cref{thm:negative_optimistic} shows that there exist $0<\epsilon<1$, a sequence of disjoint measurable sets $\{B_p\}_{p\geq 1}$ and a sequence of times $(t_p)_{p\geq 0}$ with $t_0=0$ and such that with $\mu:=\max(1,\frac{8\bar\ell}{\epsilon\eta})$, for any $p\geq 1$, $t_p>\mu t_{p-1}$, and defining the events 
\begin{align*}
    \Ecal_p = \left\{\Xbb_{\leq t_{p-1}}\cap\left(\bigcup_{p'\geq p}B_p\right)=\emptyset \right\}
    \text{ and } \Fcal_p:= \bigcup_{\mu t_{p-1}<t\leq t_p} \left\{\frac{1}{t}\sum_{u=1}^t \1_{B_p}(X_u)\geq \frac{\epsilon}{4}\right\},
\end{align*}
we have $\Pbb[\bigcap_{p\geq 1}(\Ecal_p\cap\Fcal_p)]\geq \frac{\epsilon}{4}$. We now fix a learning rule $f_\cdot$ and construct a ``bad'' process $\Ybb$ recursively. Fix $\bar y\in\Ycal$ an arbitrary value. We start by defining the random variables $N_p(t) = \sum_{u=t_{p-1}+1}^{t}\1_{B_p}(X_u)$ for any $p\geq 1$. We now construct (deterministic) values $y_p$ and sequences $(y_p^u)_{u=1}^{t_p}$ for all $p\geq 1$, of values in $\Ycal$. Suppose we have already constructed the values $y_q$ as well as the sequences $(y_q^u)_{u=1}^{t_q}$ for all $q<p$. We will now construct $y_p$ and $(y_p^u)_{u=1}^{t_p}$. Assuming that the event $\Ecal_p\cap \Fcal_p$ is met, there exists $\mu t_{p-1}<t\leq t_p$ such that
\begin{equation*}
    N_p(t) = \sum_{u=t_{p-1}+1}^t\1_{B_p}(X_u) = \sum_{u=1}^t\1_{B_p}(X_u) \geq \frac{\epsilon}{4}t,
\end{equation*}
where in the first equality we used the fact that on $\Ecal_p$, the process $\Xbb_{\leq t_{p-1}}$ does not visit $B_p$. In the rest of the construction, we will denote
\begin{equation*}
    T_p = \begin{cases}
        \min \{\mu t_{p-1}<t\leq t_p:N_p(t)\geq \frac{\epsilon}{4} t\} &\text{if } \Ecal_p\cap\Fcal_p \text{ is met}.\\
        t_p &\text{otherwise}.
    \end{cases}
\end{equation*}
Now consider the process $\Ybb_{t\leq t_{p-1}}(\Xbb)$ defined as follows. For any $1\leq q<p$ we pose
\begin{equation*}
    Y_t(\Xbb) = \begin{cases}
        y_q^{N_q(t)} &\text{if } t\leq T_q \text{ and } X_t\in B_q, \\
        y_q &\text{if } t> T_q \text{ and } X_t\in B_q, \\
        y_{q'} &\text{if }X_t\in B_{q'},\;q'<q,\\
        \bar y &\text{otherwise},
    \end{cases}\quad \quad t_{q-1}<t\leq t_q.
\end{equation*}
Similarly, for $M\geq 1$ and given any sequence $\{\tilde y_i\}_{i = 1}^{M}$, we define the following process $\Ybb_{t_{p-1}<u\leq t_p}\left(\Xbb,\{\tilde y_i\}_{i=1}^{M}\right)$ by
\begin{equation*}
    Y_u\left(\Xbb,\{\tilde y_i\}_{i =q 1}^{M}\right) = \begin{cases}
        \tilde y_{\min(N_p(u),M)} &\text{if }X_t\in B_p, \\
        y_q &\text{if } X_t\in B_{q},\;q<p,\\
        \bar y &\text{otherwise}.
    \end{cases}
\end{equation*}
We now construct a learning rule $g^p_\cdot$. First, we define the event $\Bcal:=\bigcap_{p\geq 1}(\Ecal_p\cap\Fcal_p)$. We will denote by $\tilde \Xbb=\Xbb|\Bcal$ a sampling of the process $\Xbb$ on the event $\Bcal$ which has probability at least $\frac{\epsilon}{4}$. For instance we draw i.i.d. samplings following the same distribution as $\Xbb$ then select the process which first falls into $\Bcal$. We are now ready to define a learning rule $(g^p_u)_{u\leq \tau}$ where $\tau$ is a random time. To do so, we first draw a sample $\tilde \Xbb$ which is now fixed for the learning rule $g^p_\cdot$. We define the stopping time as $\tau=N_p(T_p)$. Finally, for all $1\leq u\leq \tau$, and any sequence of values $\mb{\tilde y}_{\leq u-1}$, we pose
\begin{equation*}
    g^p_u(\mb {\tilde y}_{\leq u-1}) = 
            f_{T_p(u)}\left(\tilde \Xbb_{\leq T_p(u)-1},\left\{\Ybb_{\leq t_{p-1}}( \tilde \Xbb),\Ybb_{t_{p-1}<u\leq T_p(u)-1} \left(\tilde \Xbb,\{\tilde y_i \}_{i=1}^{u-1}\right) \right\} , \tilde X_{T_p(u)}\right),
\end{equation*}
where we used the notation $T_p(u):=\min\{t_{p-1}<t'\leq t_p:N_p(t)=u\}$ for the time of the $u-$th visit of $B_p$, which exists because $u\leq \tau=N_p(T_p)\leq N_p(t_p)$ since the event $\Bcal$ is satisfied by $\tilde\Xbb$. Note that the prediction of the rule $g_\cdot^p$ is random because of the dependence on $\tilde \Xbb$. Also, observe that the random time $\tau$ is bounded by $1\leq \tau\leq T_p\leq t_p$. Therefore, by hypothesis on the value space $(\Ycal,\ell)$, there exists a sequence $\{y^u_p\}_{u=1}^{t_p}$ and a value $y_p\in\Ycal$ such that
\begin{equation*}
    \Ebb\left[\frac{1}{\tau}\sum_{u=1}^{\tau} \left(\ell(g_u^p({\mb y_p}^{\leq u-1}), y_p^u) - \ell(y_p,y_p^u)\right)\right] \geq  \eta.
\end{equation*}
This ends the recursive construction of the values $y_p$ and the sequences $(y^u_p)_{u=1}^{t_p}$ for all $p\geq 1$. We are now ready to define the process $\Ybb(\Xbb)$, using a similar construction as before. For any $p\geq 1$ we define
\begin{equation*}
    Y_t(\Xbb) = \begin{cases}
        y_p^{N_p(t)} &\text{if } t\leq T_p \text{ and } X_t\in B_p, \\
        y_p &\text{if } t> T_p \text{ and } X_t\in B_p, \\
        y_q &\text{if } X_t\in B_q,\;q<p,\\
        \bar y &\text{otherwise},
    \end{cases}\quad \quad t_{p-1}<t\leq t_p.
\end{equation*}
We also define a function $f^*:\Xcal\to\Ycal$ by
\begin{equation*}
    f^*(x)=\begin{cases}
        y_p &\text{if }x\in B_p,\\
        \bar y &\text{otherwise}.
    \end{cases}
\end{equation*}
This function is simple hence measurable. From now, we will suppose that the event $\Bcal$ is met. For simplicity, we will denote by $\hat Y_t:=f_t(\Xbb_{\leq t-1},\Ybb_{\leq t-1},X_t)$ the prediction of the learning rule at time $t$. For any $p\geq 1$, because $\Ecal_p\cap\Fcal_p$ is met, for all $1\leq u\leq N_p(T_p)$, we have $t_{p-1}<T_p(u)\leq T_p$, and $X_{T_p(u)}\in B_p$. Hence, by construction, we have $\hat Y_{T_q(u)}=y^u_q$ and we can write
\begin{align*}
    \sum_{t=1}^{T_p} \ell(\hat Y_t,Y_t) &\geq \sum_{t=t_{p-1}+1}^{T_p} \ell(\hat Y_t,Y_t)\\
    &\geq \sum_{u=1}^{N_p(T_p)} \ell(\hat Y_{T_p(u)},Y_{T_p(u)}) \\
    &= \sum_{u=1}^{\tau} \ell( f_{T_p(u)}\left(\Xbb_{\leq T_p(u)-1},\Ybb_{\leq T_p(u)-1}, X_{T_p(u)}\right),y_p^u).
\end{align*}
Now note that because the construction was similar to the construction of $g_\cdot^p$, we have $\Ybb_{\leq T_p(u)-1} = \left\{ \Ybb_{\leq t_{p-1}}(\Xbb), \Ybb_{t_{p-1}<t\leq T_p(u)-1}\left(\Xbb,\{y^i_p\}_{i=1}^{u-1} \right) \right\}$, i.e., $\hat Y_{T_p(u)}$ coincides with the prediction $g_u^p(\{y^i_p\}_{i=1}^{u-1})$ provided that $g_u^p$ precisely used the realization $\Xbb$. Hence, conditioned on $\Bcal$ for all $u\leq \tau_p$, $\hat Y_{T_p(u)}$ has the same distribution as $g_u^p(\mb{y_p}^{\leq u-1})$. Therefore we obtain
\begin{align*}
    \Ebb\left[\left. \frac{1}{\tau}\sum_{t=1}^{T_p} \ell(\hat Y_t,Y_t) - \frac{1}{\tau}\sum_{u=1}^{\tau}\ell(y_p,y_p^u) \right| \Bcal\right]
    &\geq \Ebb\left[\left. \frac{1}{\tau}\sum_{u=1}^{\tau} \left(\ell(g_u^p(\hat Y_{T_p(u)},y_p^u) -\ell(y_p,y_p^u) \right) \right| \Bcal\right]\\
    &= \Ebb\left[ \frac{1}{\tau}\sum_{u=1}^{\tau} \left(\ell(g_u^p(\mb{y_p}^{\leq u-1}),y_p^u) -\ell(y_p,y_p^u) \right) \right]\\
    &\geq \eta.
\end{align*}
We now turn to the loss obtained by the simple function $f^*$. By construction, assuming that the event $\Bcal$ is met, we have
\begin{equation*}
    \sum_{t=1}^{T_p} \ell(f^*(X_t),Y_t) \leq \bar\ell t_{p-1} + \sum_{u=1}^{N_p(T_p)}\ell(f^*(X_{T_p(u)}),y_p^u) =\bar\ell t_{p-1} +\sum_{u=1}^{\tau}\ell(y_p,y_p^u).
\end{equation*}
Recalling that $T_p>\mu t_{p-1}\geq \frac{8\bar\ell}{\epsilon\eta}t_{p-1}$ and noting that $\tau=N_p(T_p)\geq \frac{\epsilon}{4}T_p$, we obtain
\begin{align*}
    \Ebb&\left[\left. \sup_{t_{p-1}<T\leq t_p}\frac{1}{T}\sum_{t=1}^T (\ell(\hat Y_t,Y_t) - \ell(f(X_t),Y_t)) \right|\Bcal\right] \\
    &\geq \Ebb\left[\left. \frac{\tau}{T_p}\frac{1}{\tau}\left(\sum_{t=1}^T \ell(\hat Y_t,Y_t) -\sum_{u=1}^{\tau}\ell(y_p,y_p^u) \right)- \bar\ell \frac{t_{p-1}}{T_p}  \right|\Bcal\right]\\
    &\geq \frac{\epsilon}{4} \Ebb\left[\left. \frac{1}{\tau}\sum_{t=1}^{T_p} \ell(\hat Y_t,Y_t) - \frac{1}{\tau}\sum_{u=1}^{\tau}\ell(y_p,y_p^u) \right| \Bcal\right] -  \frac{\epsilon\eta}{8}\\
    &\geq \frac{\epsilon\eta}{8}.
\end{align*}
Because this holds for any $p\geq 1$, Fatou lemma yields
\begin{align*}
    \Ebb&\left[\limsup_{T\to\infty}\frac{1}{T}\sum_{t=1}^T \ell(\hat Y_t,Y_t) - \ell(f(X_t),Y_t)\right] \\
    &\geq \Ebb\left[\left. \limsup_{T\to\infty}\frac{1}{T}\sum_{t=1}^T (\ell(\hat Y_t,Y_t) - \ell(f(X_t),Y_t)) \right|\Bcal\right] \Pbb[\Bcal]\\
    &\geq \frac{\epsilon^2\eta}{32}.
\end{align*}
Hence, we do note have almost surely $\limsup_{T\to\infty}\frac{1}{T}\sum_{t=1}^T \ell(\hat Y_t,Y_t) - \ell(f(X_t),Y_t)\leq 0$. This shows that $\Xbb\notin\solar$, which in turn implies $\solar\subset\cs$. This ends the proof that $\solar\subset\cs$. The proof that $\cs\subset\solar$ and the construction of an optimistically universal learning rule for adversarial regression is deferred to \cref{sec:unbounded_loss_moment_constraint} where we give a stronger result which also holds for unbounded losses. Note that generalizing \cref{thm:hanneke_2022} to adversarial responses already shows that $\cs\subset\solar$ and provides an optimistically universal learning rule when the loss $\ell$ is a metric $\alpha=1$.

\section{Proofs of \cref{sec:mean_estimation}}

\subsection{Proof of \cref{thm:mean_estimation}}

We first show that there exists $t_1\geq 1$ such that for any $t\geq t_1$, with high probability, for all $i\in I_t$,
\begin{equation*}
    \sum_{s=t_i}^t\ell(\hat Y_s,Y_s)\leq  L_{t,i} +3\ln^2 t\sqrt t.
\end{equation*}
For any $t\geq 0$, note that we have $\hat \ell_t=\Ebb[\ell(\hat Y_t,Y_t)\mid \Ybb_{\leq t}]$. We define the instantaneous regret $r_{t,i} = \hat\ell_t - \ell(y^i,Y_t)$. We now define $w'_{t-1,i}:=e^{\eta_{t-1}(\hat L_{t-1,i}-L_{t-1,i})}$ and pose $W_{t-1} = \sum_{i\in I_t}w_{t-1,i}$ and $W'_{t-1} = \sum_{i\in I_{t-1}} w'_{t-1,i}$, i.e., which induces the most regret. We also denote the index $k_t\in I_t$ such that $\hat L_{t,k_t}- L_{t,k_t} = \max_{i\in I_t} \hat L_{t,i} - L_{t,i}$. We first note that for any $i,j\in I_t$, we have $\ell(y^i,Y_t)-\ell(y^j,Y_t)\leq \ell(y^i,y^0)+\ell(y^0,y^j)\leq 2\ln t$. Therefore, we also have $|r_{t,i}|\leq 2\ln t$. Hence, we can apply Hoeffding's lemma to obtain
\begin{equation*}
    \frac{1}{\eta_t}\ln \frac{W'_t}{W_{t-1}} = \frac{1}{\eta_t} \ln \sum_{i\in I_t} \frac{w_{t-1,i}}{W_{t-1}}e^{\eta_t r_{t,i}} \leq \frac{1}{\eta_t}\left( \eta_t\sum_{i\in I_t} r_{t,i} \frac{w_{t-1,i}}{W_{t-1}} + \frac{\eta_t^2 (4\ln t)^2}{8}\right) = 2\eta_t \ln^2 t.
\end{equation*}
The same computations as in the proof of \cref{lemma:concatenation_predictors} then show that
\begin{multline}
\label{eq:to_sum_mean_estimation}
     \frac{1}{\eta_t}\ln \frac{w_{t-1,k_{t-1}}}{W_{t-1}}- \frac{1}{\eta_{t+1}}\ln \frac{w_{t,k_t}}{W_t} \leq 2\left(\frac{1}{\eta_{t+1}}-\frac{1}{\eta_t}\right) \ln (1+\ln (t+1)) + \frac{|I_{t+1}|-|I_t|}{\eta_t \sum_{i\in I_t} w_{t,i}} \\
     + (\hat L_{t-1,k_{t-1}}- L_{t-1,k_{t-1}}) - (\hat L_{t,k_t}- L_{t,k_t}) + 2\eta_t \ln^2 t.
\end{multline}
First suppose that we have $\sum_{i\in I_t}w_{t,i}\leq 1$. Similarly to \cref{lemma:concatenation_predictors}, we obtain $\hat L_{t,k_t}-L_{t,k_t} \leq 0$. Otherwise, let $t'=\min \{1\leq s\leq t:\forall s\leq s'\leq t,\sum_{i\in I_{s'}} w_{s',i}\geq 1\}$. We sum equation~\eqref{eq:to_sum_mean_estimation} for $s=t',\ldots, t$ which gives
\begin{multline*}
     \frac{1}{\eta_1}\ln \frac{w_{t'-1,k_{t'-1}}}{W_{t'-1}}- \frac{1}{\eta_{t+1}}\ln \frac{w_{t,k_t}}{W_t}  \leq \frac{2}{\eta_{t+1}} \ln (1+\ln (t+1))+ \frac{|I_{t+1}|}{\eta_t}\\
     + (\hat L_{t'-1,k_{t'-1}}- L_{t'-1,k_{t'-1}}) - (\hat L_{t,k_t}- L_{t,k_t}) + 2\sum_{s=t'}^t\eta_s \ln^2 s.
\end{multline*}
Similarly as in \cref{lemma:concatenation_predictors}, we have $\frac{w_{t,k_t}}{W_t}\leq 1$, $\frac{w_{t'-1,k_{t'-1}}}{W_{t'-1}}\geq \frac{1}{1+\ln t}$ and $\hat L_{t'-1,k_{t'-1}}- L_{t'-1,k_{t'-1}}\leq 0$. Finally, using the fact that $\sum_{s=1}^t \frac{1}{\sqrt s}\leq 2\sqrt t$, we obtain
\begin{equation*}
    \hat L_{t,k_t}- L_{t,k_t} \leq  \ln(1+\ln (t+1))(4+8\sqrt{t+1}) +4(1+\ln (t+1))\sqrt t + \ln^2 t\sqrt t \leq 2 \ln^2 t \sqrt t,
\end{equation*}
for all $t\geq t_0$ where $t_0$ is a fixed constant, and as a result, for all $t\geq t_0$ and $i\in I_t$, we have $\hat L_{t,i} -L_{t,i} \leq 2\ln^2 t \sqrt t.$

Now note that $|\ell(\hat Y_t,Y_t)-\Ebb[\ell(\hat Y_t,Y_t)\mid\Ybb_{\leq t}]|\leq 2\ln t$ because for all $i\in I_t$, we have $\ell(y^i,y^0)\leq \ln t$. Hence, we can apply Hoeffding-Azuma inequality to the variables $\ell(\hat Y_t,Y_t)-\hat \ell_t$ that form a sequence of differences of a martingale, which yields
\begin{equation*}
    \Pbb\left[\sum_{s=t_i}^t \ell(\hat Y_s,Y_s)>\hat L_{t,i} + u\right] \leq e^{ -\frac{u^2}{8t\ln^2 t}}.
\end{equation*}
Hence, for $t\geq t_0$ and $i\in I_t$, with probability $1-\delta$, we have
\begin{equation*}
    \sum_{s=t_i}^t \ell(\hat Y_s,Y_s)\leq \hat L_{t,i} +\ln t \sqrt{2t\ln\frac{1}{\delta}} \leq L_{t,i} + 2\ln^2 t\sqrt t + \ln t \sqrt{2t\ln\frac{1}{\delta}}.
\end{equation*}
Therefore, since $|I_t|\leq 1+\ln t$, by union bound with probability $1-\frac{1}{t^2}$ we obtain that for all $i\in I_t$,
\begin{equation*}
    \sum_{s=t_i}^t \ell(\hat Y_s,Y_s) \leq L_{t,i} + 2\ln^2 t\sqrt t + \ln t \sqrt{2t\ln(1+\ln t)}+ \ln t\sqrt{4t\ln t}\leq 3\ln^2 t \sqrt t
\end{equation*}
for all $t\geq t_1$ where $t_1\geq t_0$ is a fixed constant. Now because $\sum_{t\geq 1}\frac{1}{t^2}<\infty$, the Borel-Cantelli lemma implies that almost surely, there exists $\hat t\geq 0$ such that
\begin{equation*}
     \forall t\geq \hat t, \forall i\in I_t,\quad \sum_{s=t_i}^t \ell(\hat Y_s,Y_s) \leq L_{t,i} +3\ln^2 t\sqrt t.
\end{equation*}
We denote by $\Acal$ this event. Now let $y\in\Ycal$, $\epsilon>0$ and consider $i\geq 0$ such that $\ell(y^i,y)<\epsilon$. On the event $\Acal$, we have for all $t\geq \max(\hat t,t_i)$,
\begin{equation*}
    \quad \sum_{s=t_i}^t \ell(\hat Y_s,Y_s) \leq  \sum_{s=t_i}^t \ell(y^i,Y_s) + 3\ln^2 t\sqrt t \leq \sum_{s=t_i}^t \ell(y,Y_s) + \epsilon t + 3\ln^2 t\sqrt t.
\end{equation*}
Therefore, $\limsup_{t\to\infty} \frac{1}{t}\sum_{s=1}^t \left( \ell(\hat Y_s,Y_s)-\ell(y,Y_s) \right) \leq \epsilon$ on $\Acal$. Because this holds for any $\epsilon>0$ we finally obtain $ \limsup_{t\to\infty} \frac{1}{t}\sum_{s=1}^t \left( \ell(\hat Y_s,Y_s)-\ell(y,Y_s) \right) \leq 0$ on the event $\Acal$ of probability one, which holds for all $y\in \Ycal$ simultaneously. This ends the proof of the theorem.

\subsection{Proof of \cref{cor:universal_regression_unbounded}}

We denote by $g_\cdot$ the learning rule on values $\Ycal$ for mean estimation described in \cref{thm:mean_estimation}. Because processes in $\Xbb\in\fs$ visit only finite number of different instance points in $\Xcal$ almost surely, we can simply perform the learning rule $g_\cdot$ on each sub-process $\Ybb_{\{t:X_t=x\}}$ separately for any $x\in\Xcal$. Note that the learning rule $g_\cdot$ does not explicitely re-use past randomness for its prediction. Hence, we will not specify that the randomness used for all learning rules---for each $x$ visited by $\Xbb$---should be independent. Let us formally describe our learning rule. Consider a sequence $\mb x_{\leq t-1}$ of instances in $\Xcal$ and $\mb y_{\leq t-1}$ of values in $\Ycal$. We denote by $S_{t-1}=\{x:\mb x_{\leq t-1}\cap\{x\}\neq\emptyset\}$ the support of $\mb x_{\leq t-1}$. Further, for any $x\in S_{t-1}$, we denote $N_{t-1}(x)=\sum_{u\leq t-1}\1_{x_u=x}$ the number of times that the specific instance $x$ was visited by the sequence $\mb x_{\leq t-1}$. Last, for any $x\in S_{t-1}$, we denote $\mb y^x_{\leq N(x)}$ the values $\mb y_{\{u\leq t: X_u=x\}}$ obtained when the instance was precisely $x$ in the sequence $\mb x_{\leq t-1}$, ordered by increasing time $u$. We are now ready to define our learning rule $f_t$ at time $t$. Given a new instance point $x_t$, we pose
\begin{equation*}
    f_t(\mb x_{\leq t-1},\mb y_{\leq t-1},x_t) = \begin{cases}
        g_{N_{t-1}(x)+1}(\mb y^x_{\leq N_{t-1}(x)}) &\text{if }x_t\in S_{t-1},\\
        g_1(\emptyset) &\text{otherwise}.
    \end{cases}
\end{equation*}
Recall that for any $u\geq 1$, $g_u$ uses some randomness. The only subtlety is that at each iteration $t\geq 1$ of the learning rule $f_\cdot$, the randomness used by the subroutine call to $g_\cdot$ should be independent from the past history. We now show that $f_\cdot$ is universally consistent for adversarial regression under all processes $\Xbb\in \fs$.

Let $\Xbb\in \fs$. For simplicity, we will denote by $\hat Y_t$ the prediction of the learning rule $f_\cdot$ at time $t$. We denote $S=\{x:\{x\}\cap\Xbb\neq\emptyset\}$ the random support of $\Xbb$. By hypothesis, we have $|S|<\infty$ with probability one. Denote by $\Ecal$ this event. We now consider a specific realization $\mb x$ of $\Xbb$ falling in the event $\Ecal$. Then, $S$ is a fixed set. We also denote $\tilde S:=\{x\in S:\lim_{t\to\infty}N_t(x) = \infty\}$ the instances which are visited an infinite number of times by the sequence $\mb x$. Now, we can write for any function $f :\Xcal\to\Ycal$,
\begin{align*}
    \sum_{t=1}^T &\left(\ell(\hat Y_t,Y_t) -\ell(f(x_t),Y_t)\right)= \sum_{x\in S} \sum_{u=1}^{N_t(x)} \left(\ell(g_u(\Ybb^x_{\leq u-1}),Y_u^x) - \ell(f(x),Y_u)\right)\\
    &\leq \sum_{s\in S\setminus \tilde S} \bar\ell |\{t\geq 1:x_t=x\}| + \sum_{s\in \tilde S} \sum_{u=1}^{N_t(x)}\left(\ell(g_u(\Ybb^x_{\leq u-1}),Y_u^x) - \ell(f(x),Y_u)\right).
\end{align*}
Now, because the randomness in $g_\cdot$ was taken independently from the past at each iteration, we can apply directly \cref{thm:mean_estimation}. For all $x\in\tilde S$, with probability one, for all $y^x\in \Ycal$,
\begin{equation*}
    \limsup_{t'\to\infty}\frac{1}{t'}\sum_{u=1}^{t'}\left(\ell(g_u(\Ybb^x_{\leq u-1}),Y_u^x) - \ell(y^x,Y_u)\right) \leq 0.
\end{equation*}
We denote by $\Ecal_x$ this event. Then, on the event $\bigcap_{x\in\tilde S}\Ecal_x$ of probability one, we have for any measurable function $f:\Xcal\to\Ycal$,
\begin{align*}
    \limsup_{T\to\infty} \frac{1}{T}&\left(\ell(\hat Y_t,Y_t) -\ell(f(x_t),Y_t)\right) \\
    &\leq \sum_{s\in\tilde S}  \limsup_{T\to\infty} \frac{1}{T} \sum_{u=1}^{N_t(x)}\left(\ell(g_u(\Ybb^x_{\leq u-1}),Y_u^x) - \ell(f(x),Y_u)\right)\\
    &\leq \sum_{s\in\tilde S}  \limsup_{T\to\infty} \frac{1}{N_t(x)} \sum_{u=1}^{N_t(x)}\left(\ell(g_u(\Ybb^x_{\leq u-1}),Y_u^x) - \ell(f(x),Y_u)\right)\leq 0.
\end{align*}
As a result, averaging on realisations of $\Xbb$, we obtain that with probability one, we have that $\Lcal_{(\Xbb,\Ybb)}(f_\cdot,f)\leq 0$ for all measurable functions $f:\Xcal\to\Ycal$. Note that this is stronger than the notion of universal consistency which we defined in \cref{sec:formal_setup}, where we ask that for all measurable function $f:\Xcal\to\Ycal$, we have almost surely $\Lcal_{(\Xbb,\Ybb)}(f_\cdot,f)\leq 0$. In particular, this shows that $\fs\subset \solaru$. As result $\solaru=\fs$ and $f_\cdot$ is optimistically universal. This ends the proof of the result.

\subsection{Proof of \cref{thm:empty_solar}}

We first show that mean-estimation is not achievable. To do so, let $f_\cdot$ be a learning rule. For simplicity, we will denote by $\hat Y_t$ its prediction at step $t$. We aim to construct a process $\Ybb$ on $\Rbb$ and a value $y^*\in\Rbb$ such that with non-zero probability we have
\begin{equation*}
    \limsup_{T\to\infty} \frac{1}{T}\sum_{t=1}^T \ell(f_t(\Ybb_{\leq t-1}),Y_t) - \ell(y^*,Y_t) > 0.
\end{equation*}
We now pose $\beta:=\frac{2\alpha}{\alpha-1}>2$. For any sequence $\mb b:=(b_t)_{t\geq 1}$ in $\{-1,1\}$, we consider the following process $\Ybb^{\mb b}$ such that for any $t\geq 1$ we have $Y_t^{\mb b} = 2^{\beta^t} b_t.$ Let $\mb B:= (B_t)_{t\geq 1}$ be an i.i.d. sequence of Rademacher random variables, i.e., such that $B_1=1$ (resp. $B_1=-1$) with probability $\frac{1}{2}$. We consider the random variables $e_t:= \1_{\hat Y_t\cdot Y_t \leq 0}$ which intuitively correspond to flags for large mistakes of the learning rule $f_\cdot$ at time $t$. Because $f_\cdot$ is an online learning rule, we have
\begin{equation*}
    \Ebb[e_t\mid \Ybb_{\leq t-1}] = \Ebb_{\hat Y_t}\left[\Ebb_{B_t}[\1_{\hat Y_t\cdot Y_t\leq 0} \mid \hat Y_t]\right] = \Ebb_{\hat Y_t}\left[\1_{\hat Y_t=0} + \frac{1}{2}\1_{\hat Y_t\neq 0}\right] \geq \frac{1}{2}.
\end{equation*}
where the expectation $\Ebb_{\hat Y_t}$ refers to the expectation on the randomness of the rule $f_t$. As a result, the random variables $e_t-\frac{1}{2}$ form a sequence of differences of a sub-martingale bounded by $\frac{1}{2}$ in absolute value. By the Azuma-Hoeffding inequality, we obtain $\Pbb\left[\sum_{t=1}^T e_t \leq \frac{T}{4}\right] \leq e^{-T/8}.$ Because $\sum_{t\geq 1} e^{-t/8}<\infty$, the Borel-Cantelli lemma implies that on an event $\Ecal$ of probability one, we have $\limsup_{T\to\infty} \frac{1}{T}\sum_{t=1}^T e_t \geq \frac{1}{4}$. As a result, there exists a specific realization $\mb b$ of $\mb B$ such that on an event $\tilde \Ecal$ of probability one, we have $\limsup_{T\to\infty} \frac{1}{T}\sum_{t=1}^T e_t \geq \frac{1}{4}$. Note that the sequence $\Ybb^{\mb b}$ is now deterministic. Then, writing $e_t=e_t\1_{Y_t>0} + e_t\1_{Y_t<0}$, we obtain
\begin{equation*}
    \limsup_{T\to\infty} \frac{1}{T}\sum_{t=1}^T e_t \1_{Y_t>0}+ \limsup_{T\to\infty}   \frac{1}{T}\sum_{t=1}^T e_t \1_{Y_t<0} \geq \frac{1}{4}.
\end{equation*}
Without loss of generality, we can suppose that $\limsup_{T\to\infty} \frac{1}{T}\sum_{t=1}^T \1_{\hat Y_t\cdot Y_t\leq 0} \1_{Y_t>0} \geq \frac{1}{8}$. We now pose $y^*=1$. In the other case, we pose $y^*=-1$. We now compute for any $T\geq 1$ such that $\hat Y_t\cdot Y_t\leq 0$ and $Y_t>0$,
\begin{align*}
    \frac{1}{T}\sum_{t=1}^T \left(\ell(f_t(\Ybb_{\leq t-1}),Y_t) - \ell(y^*,Y_t)\right) &\geq \frac{ \ell(0,2^{\beta^T}) -  \ell(1,2^{\beta^T})}{T} - \frac{1}{T}\sum_{t=1}^{T-1} \ell(1,-2^{\beta^t}).\\
    &= \frac{\alpha}{T} 2^{(\alpha-1)\beta^T} + O\left(\frac{1}{T} 2^{(\alpha-2)\beta^T}\right) - 2^{\alpha(1+\beta^{T-1})}\\
    &= \frac{\alpha}{T} 2^{2\alpha\beta^{T-1}}(1+o(1)).
\end{align*}
Because, by construction $\limsup_{T\to\infty} \frac{1}{T}\sum_{t=1}^T \1_{\hat Y_t\cdot Y_t\leq 0} \1_{Y_t>0} \geq \frac{1}{8}$, we obtain
\begin{equation*}
    \limsup \frac{1}{T}\sum_{t=1}^T \left(\ell(f_t(\Ybb_{\leq t-1}),Y_t) - \ell(y^*,Y_t)\right) = \infty,
\end{equation*}
on the event $\tilde E$ of probability one. This end the proof that mean-estimation is not achievable. Because mean-estimation is the easiest regression setting, this directly implies $\solaru=\emptyset$. Formally, let $\Xbb$ a process on $\Xcal$. and $f_\cdot$ a learning rule for regression. We consider the same processes $\Ybb^{\mb B}$ where $\mb B$ is i.i.d. Rademacher and independent from $\Xbb$. The same proof shows that there exists a realization $\mb b$ for which we have almost surely $\Lcal_{(\Xbb,\Ybb)}(f_\cdot,f^*:=y^*) = \infty$, where $f^*=y^*$ denotes the constant function equal to $y^*$ where $y^*\in\Rbb$ is the value constructed as above. Hence, $\Xbb\notin\solaru$, and as a result, $\solaru=\emptyset.$

\subsection{Proof of \cref{prop:alternative_unbounded}}

Suppose that there exists an online learning rule $g_\cdot$ for mean-estimation. In the proof of \cref{cor:universal_regression_unbounded}, instead of using the learning rule for mean-estimation for metric losses introduced in \cref{thm:mean_estimation}, we can use the learning rule $g_\cdot$ to construct the learning rule $f_\cdot$ for adversarial regression on $\fs$ instance processes, which simply performs $f_\cdot$ separately on each subprocess $\Ybb_{t:X_t=x}$ with the same instance $x\in\Xcal$ for all visited $x\in\Xcal$ in the process $\Xbb$. The same proof shows that because almost surely $\Xbb$ visits a finite number of different instances, $f_\cdot$ is universally consistent under any process $\Xbb\in\fs$. Hence, $\fs\subset\solaru$. Because $\solaru\subset\soul=\fs$, we obtain directly $\solaru=\fs$ and $f_\cdot$ is optimistically universal.

On the other hand, if mean-estimation with adversarial responses is not achievable, we can use similar arguments as for the proof of \cref{thm:empty_solar}. Let $f_\cdot$ a learning rule for regression, and consider the following learning rule $g_\cdot$ for mean-estimation. We first draw a process $\tilde \Xbb$ with same distribution as $\Xbb$. Then, we pose
\begin{equation*}
    g_t(\mb y_{\leq t-1}):=f_t(\tilde \Xbb_{\leq t-1},\mb y_{\leq t-1},\tilde X_t).
\end{equation*}
Then, because mean-estimation is not achievable, there exists an adversarial process $\Ybb$ on $(\Ycal,\ell)$ such that with non-zero probability,
\begin{equation*}
    \limsup \frac{1}{T}\sum_{t=1}^T \left(\ell(g_t(\Ybb_{\leq t-1}),Y_t) - \ell(y^*,Y_t)\right) > 0.
\end{equation*}
Then, we obtain that with non-zero probability, $\Lcal_{(\tilde \Xbb,\Ybb)}>0$. Hence, $f_\cdot$ is not universally consistent. Note that the ``bad'' process $\Ybb$ is not correlated with $\tilde \Xbb$ in this construction.

\section{Proofs of \cref{sec:unbounded_loss_moment_constraint}}

\subsection{Proof of \cref{thm:negative_first_order_moment}}

Let $(x^k)_{k\geq 0}$ a sequence of distinct points of $\Xcal$. Now fix a value $y_0\in\Ycal$ and construct a sequence of values $y^1_k,y^2_k$ for $k\geq 1$ such that $\ell(y^1_k,y^2_k)\geq c_\ell 2^{k+1}$. Because $\ell(y^1_k,y^2_k)\leq c_\ell \ell(y_0,y^1_k)+c_\ell \ell(y_0,y^2_k)$, there exists $i_k\in\{1,2\}$ such that $\ell(y_0,y^{i_k}_k)\geq 2^k$. For simplicity, we will now write $y_k:=y^{i_k}_k$ for all $k\geq 1$. We define
\begin{equation*}
    t_k = \left\lfloor \sum_{l=1}^k \ell(y_0,y_l)\right\rfloor.
\end{equation*}
This forms an increasing sequence of times because $t_{k+1}-t_k\geq \ell(y_0,y_{k+1})\geq 1$. Consider the deterministic process $\Xbb$ that visits $x^k$ at time $t_k$ and $x^0$ otherwise, i.e., such that
\begin{equation*}
    X_t=\begin{cases}
    x^k &\text{if }t=t_k,\\
    x^0 &\text{otherwise}.
    \end{cases}
\end{equation*}
The process $\Xbb$ visits $\Xcal\setminus\{x^0\}$ a sublinear number of times. Hence we have for any measurable set $A$:
\begin{equation*}
    \lim_{T\to\infty}\frac{1}{T}\sum_{t=1}^T\1_A(X_t) = \begin{cases}
        1 &\text{if }x^0\in A\\
        0 &\text{otherwise}.
    \end{cases}
\end{equation*}
As a result, $\Xbb\in\crf$. We will now show that universal learning under $\Xbb$ with the first moment condition on the responses is not achievable. For any sequence $b:=(b_k)_{k\geq 1}$ of binary variables $b_k\in\{0,1\}$, we define the function $f^*_b:\Xcal\to\Ycal$ such that
\begin{equation*}
    f^*_b(x^k)=\begin{cases}
    y_0 &\text{if }b_k=0,\\
    y_k &\text{otherwise},
    \end{cases}\quad k\geq 0\quad \text{and }\quad f^*_b(x)=y_0 \text{ if }x\notin\{x_k,k\geq 0\}.
\end{equation*}
These functions are simple, hence measurable. We will first show that for any binary sequence $b$, the function $f^*_b$ satisfies the moment condition on the target functions. Indeed, we note that for any $T\geq t_1$, with $k:=\max\{l\geq 1:t_l\leq T\}$, we have
\begin{equation*}
    \frac{1}{T}\sum_{t=1}^T\ell(y_0,f^*_b(X_t))\leq \frac{1}{T}\sum_{l=1}^k \ell(y_0,y_k)\leq \frac{t_k+1}{T}\leq \frac{T+1}{T}.
\end{equation*}
Therefore, $\limsup_{T\to\infty} \frac{1}{T}\sum_{t=1}^T\ell(y_0,f^*_b(X_t))\leq 1.$
We now consider any online learning rule $f_\cdot$. Let $B=(B_k)_{k\geq 1}$ be an i.i.d. sequence of Bernouilli variables independent from the learning rule randomness. For any $k\geq 1$, denoting by $\hat Y_{t_k}:=f_{t_k}(\Xbb_{\leq t_k-1},f^*_B(\Xbb_{\leq t_k-1}),X_{t_k})$ we have
\begin{equation*}
    \Ebb_{B_k} \ell(\hat Y_{t_k},f^*_B(X_{t_k})) = \frac{\ell(\hat Y_{t_k},y_0) + \ell(\hat Y_{t_k},y_k)}{2} \geq \frac{1}{2c_\ell} \ell(y_0,y_k).
\end{equation*}
In particular, taking the expectation over both $B$ and the learning rule, we obtain
\begin{equation*}
    \Ebb\left[ \frac{1}{t_k}\sum_{t=1}^{t_k} \ell(f_t(\Xbb_{\leq t-1},f^*_B(\Xbb_{\leq t-1}),X_t),f^*_B(X_t)) \right]\geq \frac{1}{2c_\ell t_k} \sum_{l=1}^k\ell(y_0,y_k)\geq \frac{1}{2c_\ell}.
\end{equation*}
As a result, using Fatou's lemma we obtain
\begin{align*}
    \Ebb&\left[ \limsup_{T\to\infty} \frac{1}{T}\sum_{t=1}^T \ell(f_t(\Xbb_{\leq t-1},f^*_B(\Xbb_{\leq t-1}),X_t),f^*_B(X_t))\right] \\
    &\geq \limsup_{T\to\infty}\Ebb \left[ \frac{1}{T}\sum_{t=1}^T \ell(f_t(\Xbb_{\leq t-1},f^*_B(\Xbb_{\leq t-1}),X_t),f^*_B(X_t))\right] \\
    &\geq \frac{1}{2c_\ell}.
\end{align*}
Therefore, the learning rule $f_\cdot$ is not consistent under $\Xbb$ for all target functions of the form $f^*_b$ for some sequence of binary variables $b$. Indeed, otherwise for all binary sequence $b=(b_k)_{k\geq 1}$, we would have $ \Ebb_\Xbb \left[ \limsup_{T\to\infty} \frac{1}{T}\sum_{t=1}^T \ell(f_t(\Xbb_{\leq t-1},f^*_b(\Xbb_{\leq t-1}),X_t),f^*_b(X_t))\right] =0$ and as a result
\begin{equation*}
    \Ebb_B\Ebb_\Xbb \left[ \limsup_{T\to\infty} \frac{1}{T}\sum_{t=1}^T \ell(f_t(\Xbb_{\leq t-1},f^*_B(\Xbb_{\leq t-1}),X_t),f^*_B(X_t)) \right] =0.
\end{equation*}
This ends the proof of the theorem.

\subsection{Proof of \cref{lemma:empirically_integrable}}

It suffices to prove that empirical integrability implies the latter property. We pose $\epsilon_i=2^{-i}$ for any $i\geq 0$. By definition, there exists an event $\Ecal_i$ of probability one such that on $\Ecal_i$ we have
\begin{equation*}
    \exists M_i\geq 0,\quad \limsup_{T\to\infty}\frac{1}{T}\sum_{t=1}^T\ell(y_0,Y_t)\1_{\ell(y_0,Y_t)\geq M_i}\leq \epsilon_i.
\end{equation*}
As a result, on $\bigcap_{i\geq 0}\Ecal_i$ of probability one, we obtain
\begin{equation*}
    \forall \epsilon>0, \exists M:=M_{\lceil \log_2 \frac{1}{\epsilon}\rceil}\geq 0,\quad \limsup_{T\to\infty}\frac{1}{T}\sum_{t=1}^T\ell(y_0,Y_t)\1_{\ell(y_0,Y_t)\geq M}\leq \epsilon.
\end{equation*}
This ends the proof of the lemma.

\subsection{Proof of \cref{thm:noiseless_unbounded}}

Let $\Xbb\in\soul$ and $f^*:\Xcal\to\Ycal$ such that $f^*(\Xbb)$ is empirically integrable. By \cref{lemma:empirically_integrable}, there exists some value $y_0\in\Ycal$ such that on an event $\Acal$ of probability one, for all $\epsilon>0$ there exists $M_\epsilon\geq 0$ such that
\begin{equation*}
    \limsup_{T\to\infty}\frac{1}{T}\sum_{t=1}^T\ell(y_0,f^*(X_t))\1_{\ell(y_0,f^*(X_t))\geq M_\epsilon}\leq \epsilon.
\end{equation*}
For any $M\geq 1$ we define the function $f^*_M$ by
\begin{equation*}
    f^*_M(x)=\begin{cases}
    f^*(x) &\text{if } \ell(y_0,f^*(x))\leq M,\\
    y_0 &\text{otherwise}.
    \end{cases}
\end{equation*}
We know that 2C1NN is optimistically universal in the noiseless setting for bounded losses. Therefore, restricting the study to the output space $(B_\ell(y_0,M),\ell)$ we obtain that 2C1NN is consistent for $f^*_M$ under $\Xbb$, i.e.
\begin{equation*}
    \limsup_{T\to\infty} \frac{1}{T}\sum_{t=1}^T \ell(2C1NN_t(\Xbb_{t-1},f^*_M(\Xbb_{\leq t-1}),X_t),f^*_M(X_t)) = 0\quad (a.s.).
\end{equation*}
For any $t\geq 1$, we denote $\phi(t)$ the representative used by the 2C1NN learning rule. We denote $\Ecal_M$ the above event such that $\limsup_{T\to\infty}\frac{1}{T}\sum_{t=1}^T \ell(f^*_M(X_{\phi(t)}),f^*_M(X_t))=0$. We now write for any $T\geq 1$ and $M\geq 1$,
\begin{multline*}
     \frac{1}{T}\sum_{t=1}^T \ell(f^*(X_{\phi(t)}),f^*(X_t))\leq \frac{c_\ell^2}{T}\sum_{t=1}^T \ell(f^*_M(X_{\phi(t)}),f^*_M(X_t)) +
     \frac{c_\ell^2}{T}\sum_{t=1}^T \ell(f^*(X_t),f^*_M(X_t)) \\
     +\frac{c_\ell}{T}\sum_{t=1}^T \ell(f^*(X_{\phi(t)}),f^*_M(X_{\phi(t)})).
\end{multline*}
We now note that by construction of the 2C1NN learning rule,
\begin{align*}
    \frac{1}{T}\sum_{t=1}^T \ell(f^*(X_{\phi(t)}),f^*_M(X_{\phi(t)})) &= \frac{1}{T}\sum_{u=1}^T \ell(f^*(X_u),f^*_M(X_u)) |\{u<t\leq T:\phi(t)=u\}|\\
    &\leq \frac{2}{T}\sum_{t=1}^T \ell(f^*(X_t),f^*_M(X_t)).
\end{align*}
Hence, we obtain
\begin{multline*}
    \frac{1}{T}\sum_{t=1}^T \ell(f^*(X_{\phi(t)}),f^*(X_t))\leq \frac{c_\ell^2}{T}\sum_{t=1}^T \ell(f^*_M(X_{\phi(t)}),f^*_M(X_t))\\ +
     \frac{c_\ell(2+c_\ell)}{T}\sum_{t=1}^T \ell(y_0,f^*(X_t))\1_{\ell(y_0,f^*(X_t))>M}.
\end{multline*}
As a result, on the event $\Acal\cap\bigcap_{M\geq 1}\Ecal_M$ of probability one, for any $M\geq 1$, we obtain
\begin{multline*}
    \limsup_{T\to\infty}\frac{1}{T}\sum_{t=1}^T\ell(f^*(X_{\phi(t)}),f^*(X_t)) \\
    \leq c_\ell(2+c_\ell) \limsup_{T\to\infty}\frac{1}{T}\sum_{t=1}^T \ell(y_0,f^*(X_t))\1_{\ell(y_0,f^*(X_t))\geq M}.
\end{multline*}
In particular, if $\epsilon>0$ we can apply this result with $M:=\lceil M_\epsilon\rceil$, which in turn shows that we have $\limsup_{T\to\infty}\frac{1}{T}\sum_{t=1}^T\ell(f^*(X_{\phi(t)}),f^*(X_t)) \leq c_\ell(2+c_\ell)\epsilon$. Because this holds for any $\epsilon>0$ we finally obtain that on the event  $\Acal\cap\bigcap_{M\geq 1}\Ecal_M$ we have
\begin{equation*}
    \limsup_{T\to\infty}\frac{1}{T}\sum_{t=1}^T\ell(f^*(X_\phi(t)),f^*(X_t)) =0.
\end{equation*}
This ends the proof of the theorem.

\subsection{Proof of \cref{thm:CS_regression_unbounded}}

We first define the learning rule. Using Lemma 23 of \cite{hanneke2021learning}, let $\Tcal\subset\Bcal$ a countable set such that for all $\Xbb\in\cs,A\subset\Bcal$ we have
\begin{equation*}
    \inf_{G\in\Tcal} \Ebb[\hat \mu_\Xbb (G \bigtriangleup A)]=0.
\end{equation*}
Now let $(y^i)_{i\geq 0}$ be a dense sequence in $\Ycal$. For any $k\geq 0$, any indices $l_1,\ldots, l_k\in \Nbb$ and any sets $A_1,\ldots,A_k \in \Tcal$, we define the function $f_{\{l_1,\ldots,l_k\},\{A_1,\ldots,A_k\}}:\Xcal\to\Ycal$ as \begin{equation*}
    f_{\{l_1,\ldots,l_k\},\{A_1,\ldots,A_k\}}(x)  = y^{\max\{0\leq j\leq k:\; x\in A_j\}}
\end{equation*}
where $A_0=\Xcal$. These functions are simple hence measurable. Because the set of such functions is countable, we enumerate these functions as $f^0,f^1\ldots$ Without loss of generality, we suppose that $f^0=y^0$. For any $i\geq 0$, we denote $k^i\geq 0$, $\{l_1^i,\ldots,l_{k^i}^i\}$ and $\{A_1^i,\ldots,A_{k^i}^i\}$ such that $f^i$ was defined as $f^i:=f_{\{l_1^i,\ldots,l_k^i\},\{A_1^i,\ldots,A_k^i\}}$. We now define a sequence of sets $(I_t)_{t\geq 1}$ of indices and a sequence of sets $(\Fcal_t)_{t\geq 1}$ of measurable functions by
\begin{equation*}
    I_t := \{i\leq \ln t: \ell(y^{l_p^i},y^0)\leq 2^{-\alpha+1}\ln t,\; \forall 1\leq p\leq k^i\}\quad \text{and}\quad \Fcal_t:=\{f^i:i\in I_t\}.
\end{equation*}
Then, clearly $I_t$ is finite and $\bigcup_{t\geq 1}I_t = \Nbb$. For any $i\geq 0$, we define $t_i = \min\{t:i\in I_t\}$. We are now ready to construct our learning rule. Let $\eta_t=\frac{1}{\ln t \sqrt t}$. Fix any sequences $(x_t)_{t\geq 1}$ in $\Xcal$ and $(y_t)_{t\geq 1}$ in $\Ycal$. At step $t\geq 1$, after observing the values $x_i$ for $1\leq i\leq t$ and $y_i$ for $1\leq i\leq t-1$, we define for any $i\in I_t$ the loss $L_{t-1,i}:= \sum_{s=t_i}^{t-1}\ell(f^i(x_s),y_s).$ 
For any $M\geq 1$ we define the function $\phi_M:\Ycal\to\Ycal$ such that
\begin{equation*}
    \phi_M(y) = \begin{cases} 
    y &\text{if } \ell(y,y^0)< M,\\
    y^0 &\text{otherwise}.
    \end{cases}
\end{equation*}
We now construct construct some weights $w_{t,i}$ for $t\geq 1$ and $i\in I_t$ recursively in the following way. Note that $I_1=\{0\}$. Therefore, we pose $w_{0,0}=1$. Now let $t\geq 2$ and suppose that $w_{s-1,i}$ have been constructed for all $1\leq s\leq t-1$. We define
\begin{equation*}
    \hat \ell_s:=\frac{\sum_{j\in I_s} w_{s-1,j}\ell(f^j(x_s),\phi_{2^{-\alpha+1}\ln s}(y_s))}{\sum_{j\in I_s} w_{s-1,j}}
\end{equation*}
and for any $i\in I_t$ we note $\hat L_{t-1,i} := \sum_{s=t_i}^{t-1} \hat \ell_s$. In particular, if $t_i=t$ we have $\hat L_{t-1,i}=L_{t-1,i}=0$. The weights at time $t$ are constructed as $w_{t-1,i}:= e^{\eta_t(\hat L_{t-1,i}-L_{t-1,i})}$ for any $i\in I_t$. Last, let $\{\hat i_t\}_{t\geq 1}$ a sequence of independent random $\Nbb-$valued variables such that
\begin{equation*}
    \Pbb(\hat i_t = i) = \frac{w_{t-1,i}}{\sum_{j\in I_{t}}w_{t-1,j}},\quad i\in I_t.
\end{equation*}
Finally, the prediction is defined as $\hat y_t:=f^{\hat i_t}(x_t)$. The learning rule is summarized in \cref{alg:C1_processes}. 

\begin{algorithm}[tb]
\caption{A learning rule for adversarial empirically integrable responses under $\cs$ processes.}\label{alg:C1_processes}
\hrule height\algoheightrule\kern3pt\relax
\KwIn{Historical samples $(X_t,Y_t)_{t<T}$ and new input point $X_T$}
\KwOut{Predictions $\hat Y_t$ for $t\leq T$}
Construct the sequence of measurable functions $\{f^i,i\geq 0\}$ with $f^i=f_{\{l_1^i,\ldots,l_k^i\},\{A_1^i,\ldots,A_k^i\}}$\\
$I_t:=\{i\leq \ln t,\ell(y^{l_p^i},y^0) \leq 2^{-\alpha+1}\ln t,\forall 1\leq p\leq k^i\},\Fcal_t:=\{f^i,i\in I_t\},\eta_t:=\frac{1}{\ln t\sqrt t},t\geq 1$\\
$t_i=\min\{t:i\in I_t\},i\geq 0$\\
$w_{0,0}:=1,\quad \hat Y_1=y^0(=f^0(X_0))$ \tcp*[f]{Initialisation}\\
\For{$t=2,\ldots, T$}{
    $L_{t-1,i} = \sum_{s=t_i}^{t-1} \ell(f^i(X_s),\phi_{2^{-\alpha+1}\ln t}(Y_s)),\quad 
    \hat L_{t-1,i} = \sum_{s=t_i}^{t-1} \hat \ell_s,\quad i\in I_t$\\
    $w_{t-1,i} := \exp(\eta_t(\hat L_{t-1,i}-L_{t-1,i})),\quad i\in I_t$\\
    $p_t(i) = \frac{w_{t-1,i}}{\sum_{j\in I_t} w_{t-1,j}},\quad i\in I_t$\\
    $\hat i_t \sim p_t(\cdot)$ \tcp*[f]{Function selection}\\
    $\hat Y_t = f^{\hat i_t}(X_t)$\\
    $\hat \ell_t:=\frac{\sum_{j\in I_t} w_{t-1,j}\ell(f^j(X_s),\phi_{2^{-\alpha+1}\ln t}(Y_t)}{\sum_{j\in I_t} w_{t-1,j}}$
}
\hrule height\algoheightrule\kern3pt\relax
\end{algorithm}

For simplicity, we will refer to the predictions of the learning rule as $(\hat Y_t)_{t\geq 1}$. Now consider a process $(\Xbb,\Ybb)$ with $\Xbb\in \cs$ and such that $\Ybb$ is empirically integrable. By \cref{lemma:empirically_integrable}, there is $y_0\in\Ycal$ such that on an event $\Acal$ of probability one, for any $\epsilon>0$, there exists $M_\epsilon \geq 0$ with $\limsup_{T\to\infty} \frac{1}{T}\sum_{t=1}^T \ell(y_0,Y_t)\1_{\ell(y_0,Y_t)\geq M_\epsilon} \leq \epsilon$. We will now denote $\tilde \Ybb$ the process defined by $\tilde Y_t = \phi_{2^{-\alpha+1}\ln t}(Y_t)$ for all $t\geq 1$. Then, for any $i\in I_t$, note that using \cref{lemma:loss_identity} we have
\begin{equation*}
    0\leq \ell(f^i(x_t),\tilde Y_t) \leq 2^{\alpha-1}\left(\ell(f^i(x_t), y^0)+ \ell(y^0,\tilde Y_t)\right)\leq 2\ln t,
\end{equation*}
by construction of the set $I_t$. As a result, for any $i,j\in I_t$, we obtain $|\ell(f^i(x_t),\tilde Y_t^M)-\ell(f^j(x_t)-\tilde Y_t^M)| \leq 2\ln t$. Hence, we can use the same proof as for \cref{thm:mean_estimation} and show that almost surely, there exists $\hat t\geq 1$ such that
\begin{equation*}
     \forall t\geq \hat t,\forall i\in I_t, \quad \sum_{s=t_i}^t \ell(\hat Y_s,\tilde Y_s^M) \leq L_{t,i} + 3 \ln^2 t\sqrt t .
\end{equation*}
We denote by $\Bcal$ this event. Now let $f:\Xcal\to\Ycal$ to which we compare the predictions of our learning rule. For any $M\geq 1$, the function $\phi_M\circ f$ is measurable and has values in the ball $B_\ell(y_0,M)$ where the loss is bounded by $2^\alpha M$. Hence, by Lemma 24 from \cite{hanneke2021learning} because $\Xbb\in\Ccal_1$ we have
\begin{equation*}
    \inf_{i\geq 0} \Ebb\left[\hat\mu_\Xbb (\ell(\phi_M\circ f(\cdot),f^i(\cdot)))\right] = 0.
\end{equation*}
Now for any $k\geq 0$, let $i_k\geq 0$ such that $\Ebb\left[\hat\mu_\Xbb (\ell(\phi_M\circ f(\cdot),f^{i_k}(\cdot)))\right] < 2^{-2k}$. By Markov inequality, we have
\begin{equation*}
    \Pbb\left[ \hat\mu_\Xbb (\ell(\phi_M\circ f(\cdot),f^i(\cdot))) \right] < 2^{-k}] \geq 1-2^{-k}.
\end{equation*}
Because $\sum_k 2^{-k}<\infty$, the Borel-Cantelli lemma implies that almost surely there exists $\hat k$ such that for any $k\geq \hat k$, the above inequality is met. We denote $\Ecal_M$ this event. On the event $\Bcal\cap \Ecal_M$ of probability one, for $k\geq \hat k$ and any $T\geq \max(t_{i_k},\hat t)$ we have for any $\epsilon>0$,
\begin{align*}
    &\frac{1}{T}\sum_{t=1}^T \left(\ell(\hat Y_t,\tilde Y_t)-\ell( \phi_M\circ f(X_t),\tilde Y_t)\right)\\
    &= \frac{1}{T}\sum_{t=1}^T \ell(\hat Y_t,\tilde Y_t)-\ell( f^{i_k}(X_t),\tilde Y_t) +  \frac{1}{T}\sum_{t=1}^T \ell(f^{i_k}(X_t),\tilde Y_t)-\ell( \phi_M\circ f(X_t),\tilde Y_t)\\
    &\leq \frac{1}{T}\sum_{t=1}^{t_{i_k}-1}\ell(\hat Y_t,\tilde Y_t) + \frac{1}{T}\left(\sum_{t=t_{i_k}}^T \ell(\hat Y_t,\tilde Y_t) - L_{T,i_k} \right)  + \frac{\epsilon}{T}\sum_{t=1}^T \ell( \phi_M\circ f(X_t),\tilde Y_t)\\
    &\quad\quad\quad\quad+  \frac{c_\epsilon^\alpha}{T}\sum_{t=1}^T \ell(f^{i_k}(X_t), \phi_M\circ f(X_t))\\
    &\leq \frac{2\ln t_{i_k}}{T}+  \frac{3\ln^2 T}{\sqrt T} + \epsilon 2^{\alpha-1}M + \epsilon 2^{\alpha-1}\frac{1}{T}\sum_{t=1}^T \ell(y^0,\tilde Y_t) + \frac{c_\epsilon^\alpha}{T}\sum_{t=1}^T \ell(f^{i_k}(X_t), \phi_M\circ f(X_t))\\
    &\leq \frac{2\ln t_{i_k}}{T}+  \frac{3\ln^2 T}{\sqrt T} + \epsilon 2^{\alpha-1}M + \epsilon 2^{\alpha-1}\frac{1}{T}\sum_{t=1}^T \ell(y^0,Y_t)  +  \frac{c_\epsilon^\alpha}{T}\sum_{t=1}^T \ell(f^{i_k}(X_t), \phi_M\circ f(X_t)),
\end{align*}
where in the last inequality we used the inequality $\ell(y^0,\tilde Y_t)\leq \ell(y^0,Y_t)$ by construction of $\tilde Y_t = \phi_{2^{-\alpha+1}\ln t}(Y_t)$. Now on the event $\Acal$, we have
\begin{align*}
    Z_1:=\limsup_{T\to\infty}\frac{1}{T}&\sum_{t=1}^T \ell(y^0,Y_t) \leq 2^{\alpha-1}\ell(y_0,y^0)+ 2^{\alpha-1}\limsup_{T\to\infty} \frac{1}{T}\sum_{t=1}^T \ell(y_0,Y_t)\\
    &\leq 2^{\alpha-1}\ell(y_0,y^0)+ 2^{\alpha-1}\left(M_1+\limsup_{T\to\infty} \frac{1}{T}\sum_{t=1}^T \ell(y_0,Y_t)\1_{\ell(y_0,Y_t)\geq M_1}\right)\\
    &\leq 2^{\alpha-1}\ell(y_0,y^0)+ 2^{\alpha-1}(M_1+1) < \infty.
\end{align*}
Thus, on the event $\Acal\cap\Bcal\cap \Ecal_M$, for any $k\geq \hat k$ we have for any $\epsilon>0$,
\begin{equation*}
    \limsup_{T} \frac{1}{T}\sum_{t=1}^T\ell(\hat Y_t,\tilde Y_t)-\ell( \phi_M\circ f(X_t),\tilde Y_t)) \leq \epsilon 2^{\alpha-1}M + \epsilon 2^{\alpha-1}Z_1 + \frac{c_\epsilon^\alpha}{2^k}.
\end{equation*}
Let $\delta>0$. Now taking $\epsilon = \frac{1}{2^{\alpha}(M+Z_1)}$, we obtain that on the event $\Acal\cap\Bcal\cap \Ecal_M$, for any $k\geq \hat k$, we have $\limsup_{T} \frac{1}{T}\sum_{t=1}^T\ell(\hat Y_t,\tilde Y_t)-\ell( \phi_M\circ f(X_t),\tilde Y_t)) \leq \delta + \frac{c_\epsilon^\alpha}{2^k}.$ This yields $\limsup_{T\to\infty} \frac{1}{T}\sum_{t=1}^T \ell(\hat Y_t,\tilde Y_t)-\ell( \phi_M\circ f(X_t),\tilde Y_t)) \leq \delta.$ Because this holds for any $\delta>0$ we obtain $\limsup_{T\to\infty} \frac{1}{T}\sum_{t=1}^T \ell(\hat Y_t,\tilde Y_t)-\ell( \phi_M\circ f(X_t),\tilde Y_t)) \leq 0.$ Finally, on the event $\Acal\cap\Bcal\cap \bigcap_{M= 1}^{\infty}\Ecal_M$ of probability one, we have
\begin{equation*}
    \limsup_{T\to\infty} \frac{1}{T}\sum_{t=1}^T \left(\ell(\hat Y_t,\tilde Y_t)-\ell( \phi_M\circ f(X_t),\tilde Y_t)\right) \leq 0,\quad \forall M\geq 1,
\end{equation*}
where $M$ is an integer. We now observe that on the event $\Acal$, the same guarantee for $y_0$ also holds for $y^0$. Indeed, let $\epsilon$. For $\tilde M_\epsilon:=2^{\alpha-1}(M_{2^{-\alpha}\epsilon} +\ell(y^0,y_0)) + \ell(y_0,y^0)$ we have
\begin{align*}
    \frac{1}{T}\sum_{T=1}^T \ell(y^0,Y_t) &\1_{\ell(y^0,Y_t)\geq \tilde M_\epsilon}\\
    &\leq 2^{\alpha-1} \ell(y^0,y_0)\frac{1}{T} \sum_{t=1}^T \1_{\ell(y^0,Y_t)\geq \tilde M_\epsilon} + 2^{\alpha-1} \frac{1}{T}\sum_{T=1}^T \ell(y_0,Y_t) \1_{\ell(y^0,Y_t)\geq \tilde M_\epsilon}\\
    &\leq 2^{\alpha-1} \ell(y^0,y_0)\frac{1}{T} \sum_{t=1}^T \1_{\ell(y_0,Y_t)\geq 2^{-\alpha+1}M - \ell(y_0,y^0)} \\
    &\quad\quad\quad+ 2^{\alpha-1} \frac{1}{T}\sum_{T=1}^T \ell(y_0,Y_t) \1_{\ell(y_0,Y_t)\geq 2^{-\alpha+1}M - \ell(y_0,y^0)}\\
    &\leq 2^{\alpha}  \frac{1}{T}\sum_{t=1}^T \ell(y_0,Y_t)\1_{\ell(y_0,Y_t)\geq M_{2^{-\alpha}\epsilon}}
\end{align*}
Hence, we obtain $\limsup_{T\to\infty}\frac{1}{T}\sum_{T=1}^T \ell(y^0,Y_t) \1_{\ell(y^0,Y_t)\geq \tilde M_\epsilon} \leq \epsilon$. We now write
\begin{align*}
    \frac{1}{T}&\sum_{t=1}^T \ell( \phi_M\circ f(X_t),\tilde Y_t)-\ell( f(X_t),Y_t)\\
    &\leq \frac{1}{T}\sum_{t=1}^T \left(\ell(y^0, Y_t)-\ell( f(X_t),Y_t)\right)\1_{\ell(f(X_t),y^0)\geq M}\1_{ \ell(Y_t,y^0)\leq \ln t} \\
    &\quad\quad\quad+ \frac{1}{T}\sum_{t=1}^T \left(\ell(f(X_t), y^0)-\ell( f(X_t),Y_t)\right)\1_{\ell(f(X_t),y^0)\leq M}\1_{\ell(Y_t,y^0)\geq 2^{-\alpha+1} \ln t}\\
    &\leq \frac{1}{T}\sum_{t=1}^T \left(2\ell(y^0, Y_t)-2^{-\alpha+1}\ell( f(X_t),y^0)\right)\1_{\ell(f(X_t),y^0)\geq M}\\
    &\quad\quad\quad+ 
    \frac{1}{T}\sum_{t=1}^T \left(2\ell(f(X_t), y^0)-2^{-\alpha+1}\ell( y^0,Y_t)\right)\1_{\ell(f(X_t),y^0)\leq M}\1_{\ell(Y_t,y^0)\geq 2^{-\alpha+1}\ln t}\\
    &\leq \frac{2}{T}\sum_{t=1}^T \ell(y^0, Y_t)\1_{\ell(Y_t,y^0)\geq 2^{-\alpha}M} + 
    \frac{2M e^{2^{2\alpha-1} M}}{T}.
\end{align*}
As a result, on the event $\Acal\cap\Bcal\cap \bigcap_{M= 1}^{\infty}\Ecal_M$, for any $M\geq 1$,
\begin{equation*}
    \limsup_{T\to\infty}\frac{1}{T} \sum_{t=1}^T \ell( \phi_M\circ f(X_t),\tilde Y_t)-\ell( f(X_t),Y_t) \leq  2 \limsup_{T\to\infty} \frac{1}{T}\sum_{t=1}^T \ell(y^0, Y_t)\1_{\ell(Y_t,y^0)\geq 2^{-\alpha}M} . 
\end{equation*}
Last, we compute
\begin{align*}
    \frac{1}{T}\sum_{t=1}^T \ell(\hat Y_t,Y_t)-\ell(\hat Y_t,\tilde Y_t) &= \frac{1}{T}\sum_{t=1}^T \left(\ell(\hat Y_t,Y_t)-\ell(\hat Y_t,y^0)\right)\1_{\ell(Y_t,y^0)\geq 2^{-\alpha+1} \ln t}\\
    &\leq \frac{1}{T}\sum_{t=1}^T \left(2^{\alpha-1}\ell(\hat Y_t,y^0) + 2^{\alpha-1}\ell(Y_t,y^0)\right)\1_{\ell(Y_t,y^0)\geq 2^{-\alpha+1}\ln t}\\
    &\leq \frac{1}{T}\sum_{t=1}^T \left(\ln t + 2^{\alpha-1}\ell(Y_t,y^0)\right)\1_{\ell(Y_t,y^0)\geq 2^{-\alpha+1}\ln t}\\
    &\leq \frac{2^\alpha}{T}\sum_{t=1}^T  \ell(Y_t,y^0)\1_{\ell(Y_t,y^0)\geq 2^{-\alpha+1}\ln t}.
\end{align*}
Note that for any $\epsilon>0$, we have on the event $\Acal$ that for any $M\geq 1$,
\begin{align*}
    \limsup_{T\to\infty} \frac{1}{T}\sum_{t=1}^T  \ell(Y_t,y^0)\1_{\ell(Y_t,y^0)\geq 2^{-\alpha+1}\ln t} &\leq \limsup_{T\to\infty} \frac{1}{T}\sum_{t\geq e^{2^{\alpha-1}M}}^T  \ell(Y_t,y^0)\1_{\ell(Y_t,y^0)\geq M}\\
    &= \limsup_{T\to\infty} \frac{1}{T}\sum_{t=1}^T  \ell(Y_t,y^0)\1_{\ell(Y_t,y^0)\geq M}. 
\end{align*}
Hence, because this holds for any $M\geq 1$, if $\epsilon>0$ we can apply this to the integer $M:=\lceil \tilde M_\epsilon\rceil$ which yields $\limsup_{T\to\infty} \frac{1}{T}\sum_{t=1}^T  \ell(Y_t,y^0)\1_{\ell(Y_t,y^0)\geq 2^{-\alpha+1}\ln t}\leq \epsilon$. This holds for any $\epsilon>0$. Hence we obtain on the event $\Acal$ that $\limsup_{T\to\infty} \frac{1}{T}\sum_{t=1}^T  \ell(Y_t,y^0)\1_{\ell(Y_t,y^0)\geq 2^{-\alpha+1}\ln t} \leq 0$, which implies that $\limsup_{T\to\infty} \frac{1}{T}\sum_{t=1}^T \ell(\hat Y_t,Y_t)-\ell(\hat Y_t,\tilde Y_t)\leq 0$. Putting everything together, we obtain on $\Acal\cap\Bcal\cap\bigcap_{M= 1}^{\infty}\Ecal_M$ that for any $M\geq 1$,
\begin{align*}
    &\limsup_{T\to\infty} \frac{1}{T}\sum_{t=1}^T \ell(\hat Y_t,Y_t)-\ell(f(X_t), Y_t) \leq \limsup_{T\to\infty} \frac{1}{T}\sum_{t=1}^T \ell(\hat Y_t,Y_t)-\ell(\hat Y_t,\tilde Y_t)\\
    &\quad\quad +\limsup_{T\to\infty} \frac{1}{T}\sum_{t=1}^T \ell(\hat Y_t,\tilde Y_t)-\ell( \phi_M\circ f(X_t),\tilde Y_t)\\
    &\quad\quad + \limsup_{T\to\infty}\frac{1}{T} \sum_{t=1}^T \ell( \phi_M\circ f(X_t),\tilde Y_t)-\ell( f(X_t),Y_t)\\
    &\leq 2 \limsup_{T\to\infty} \frac{1}{T}\sum_{t=1}^T \ell(y^0, Y_t)\1_{\ell(Y_t,y^0)\geq 2^{-\alpha}M}.
\end{align*}
Because this holds for all $M\geq 1$, we can again apply this result to $M:=\lceil \tilde M_\epsilon \rceil$ which yields the result $\limsup_{T\to\infty} \frac{1}{T}\sum_{t=1}^T \ell(\hat Y_t,Y_t)-\ell(f(X_t), Y_t)\leq \epsilon$. This holds for any $\epsilon>0$. Therefore, we finally obtain on the event $\Acal\cap\Bcal\cap\bigcap_{M= 1}^{\infty}\Ecal_M$ of probability one, one has $\limsup_{T\to\infty} \frac{1}{T}\sum_{t=1}^T \ell(\hat Y_t,Y_t)-\ell(f(X_t), Y_t) \leq 0.$ This ends the proof that \cref{alg:C1_processes} is universally consistent under $\cs$ processes for adversarial empirically integrable responses. Now because there exists a ball $B_\ell(y,r)$ of $(\Ycal,\ell)$ that does not satisfy $\ftime$, from \cref{thm:bad_value_spaces}, universal learning with responses restricted on this ball cannot be achieved for processes $\Xbb\notin\cs$. However, these responses are empirically integrable because they are bounded. Hence, $\cs$ is still necessary for universal learning with adversarial empirically integrable responses. Thus $\solar=\cs$ and the provided learning rule is optimistically universal. This ends the proof of the theorem.

\subsection{Proof of \cref{thm:SOUL_regression_unbounded}}

Fix $(\Xcal,\rho_\Xcal)$ and a value space $(\Ycal,\ell)$ such that any ball satisfies $\ftime$ We now construct our learning rule. Let $\bar y\in\Ycal$ be an arbitrary value. For any $M\geq 1$, because $B_\ell(\bar y,M)$ is bounded and satisfies $\ftime$, there exists an optimistically universal learning rule $f_\cdot^M$ for value space $(B_\ell(y_0,M),\ell)$. For any $M\geq 1$, we define the function $\phi_M:\Ycal\to\Ycal$ defined by restricting the space to the ball $B_\ell(\bar y,M)$ as follows
\begin{equation*}
    \phi_M(y):=\begin{cases}
        y &\text{if }\ell(y,\bar y)< M\\
        \bar y &\text{otherwise}.
    \end{cases}
\end{equation*}
For simplicity, we will denote by $\hat Y_t^M:=f^M_t(\Xbb_{\leq t-1},\phi_M(\Ybb)_{\leq t-1},X_t)$ the prediction of $f_\cdot^M$ at time $t$ for the responses which are restricted to the ball $B_\ell(\bar y,M)$. We now combine these predictors using online learning into a final learning rule $f_\cdot$. Specifically, we define $I_t:=\{0\leq M\leq 2^{-\alpha+1}\ln t\}$ for all $t\geq 1$. We also denote $t_M=\lceil e^{2^{\alpha-1}M}\rceil$ for $M\geq 0$ and pose $\eta_t=\frac{1}{4\sqrt t}$. For any $M\in I_t$, we define 
\begin{equation*}
    L_{t-1,M}:=\sum_{s=t_M}^{t-1}\ell(\hat Y_s^M,\phi_{2^{-\alpha+1}\ln s}(Y_s)).
\end{equation*}
For simplicity, we will denote by $\tilde\Ybb$ the process defined by $\tilde Y_t = \phi_{2^{-\alpha+1}\ln t}(Y_t)$ for all $t\geq 1$. We now construct recursive weights as $w_{0,0}=1$ and for $t\geq 2$ we pose for all $1\leq s\leq t-1$
\begin{equation*}
    \hat l_s:= \frac{\sum_{M\in I_s}w_{s-1,M}\ell(\hat Y_s^M,\tilde Y_s)}{\sum_{M\in I_s} w_{s-1,M}}.
\end{equation*}
Now for any $M\in I_t$ we note $\hat L_{t-1,M}:=\sum_{s=t_M}^{t-1} \hat \ell_s$, and pose $w_{t-1,M}:=e^{\eta_t(\hat L_{t-1,M}-L_{t-1,M})}$. We then choose a random index $\hat M_t$ independent from the past history such that
\begin{equation*}
    \Pbb(\hat M_t = M):= \frac{w_{t-1,M}}{\sum_{M'\in I_t}w_{t-1,M'}},\quad M\in I_t.
\end{equation*}
The output the learning rule is $f_t(\Xbb_{\leq t-1},\Ybb_{\leq t-1},X_t):= \hat Y_t^{\hat M_t}$. For simplicity, we will denote by $\hat Y_t:=f_t(\Xbb_{\leq t-1},\Ybb_{\leq t-1},X_t)$ the prediction of $f_\cdot$ at time $t$. This ends the construction of our learning rule which is summarized in \cref{alg:SMV_EI}.

\begin{algorithm}[tb]
\caption{A learning rule for adversarial empirically integrable responses under $\smv$ processes for value spaces $(\Ycal,\ell)$ such that any ball satisfies $\ftime$.}\label{alg:SMV_EI}
\hrule height\algoheightrule\kern3pt\relax
\KwIn{Historical samples $(X_t,Y_t)_{t<T}$ and new input point $X_T$\\
\quad\quad\quad Optimistically universal learning rule $f^M_\cdot$ for value space $B_\ell(y_0,M),\ell)$, where $y_0\in\Ycal$ fixed.}
\KwOut{Predictions $\hat Y_t$ for $t\leq T$}
$I_t:=\{0\leq M\leq 2^{-\alpha+1}\ln t\} ,\eta_t:=\frac{1}{4\sqrt t},t\geq 1$\\
$t_M=\lceil e^{2^{\alpha-1}M}\rceil,M\geq 0$\\
$w_{0,0}:=1,\quad \hat Y_1=y^0(=f^0(X_0))$ \tcp*[f]{Initialisation}\\
\For{$t=2,\ldots, T$}{
    $L_{t-1,M} = \sum_{s=t_M}^{t-1} \ell(f^M_s(\Xbb_{\leq s-1},\phi_M(\Ybb)_{\leq s-1},X_s),\phi_{2^{-\alpha+1}\ln s} (Y_s)),\quad 
    \hat L_{t-1,M} = \sum_{s=t_M}^{t-1} \hat \ell_s,\quad M\in I_t$\\
    $w_{t-1,M} := \exp(\eta_t(\hat L_{t-1,M}-L_{t-1,M})),\quad M\in I_t$\\
    $p_t(M) = \frac{w_{t-1,M}}{\sum_{M'\in I_t} w_{t-1,M'}},\quad M\in I_t$\\
    $\hat M_t \sim p_t(\cdot)$ \tcp*[f]{Model selection}\\
    $\hat Y_t = f^{\hat M_t}_t(\Xbb_{\leq t-1},\phi_M(\Ybb)_{\leq t-1},X_t)$\\
    $\hat \ell_t:=\frac{\sum_{j\in I_t} w_{t-1,j}\ell(f^M_t(\Xbb_{\leq t-1},\phi_M(\Ybb)_{\leq t-1},X_t),\phi_{2^{-\alpha+1}\ln t}(Y_t)}{\sum_{j\in I_t} w_{t-1,j}}$
}
\hrule height\algoheightrule\kern3pt\relax
\end{algorithm}

Now let $(\Xbb,\Ybb)$ be such that $\Xbb\in\soul$ and $\Ybb$ empirically integrable. By \cref{lemma:empirically_integrable}, there exists some value $y_0\in\Ycal$ such that on an event $\Acal$ of probability one, we have for any $\epsilon$, a threshold $M_\epsilon\geq 0$ with $\limsup_{T\to\infty}\frac{1}{T}\sum_{t=1}^T\ell(y_0,Y_t)\1_{\ell(y_0,Y_t)\geq M_\epsilon} \leq \epsilon.$ We fix a measurable function $f:\Xcal\to\Ycal$. Also, for any $t\geq 1$ and $M\in I_t$ we have $0\leq \ell(\hat Y^M_t,\tilde Y_t) \leq 2^{\alpha-1}\ell(\hat Y^M_t,\bar y) + 2^{\alpha-1}\ell(\tilde Y_t,\bar y) \leq 2\ln t$. As a result, for any $M, M'\in I_t$ we have $|\ell( \hat Y_t^M,\tilde Y_t)-\ell(\hat Y_t^{M'},\tilde Y_t)|\leq 2\ln t$. Because $|I_t|\leq 1+\ln t$ for all $t\geq 1$, the same proof as \cref{thm:mean_estimation} shows that on an event $\Bcal$ of probability one, there exists $\hat t\geq 0$ such that
\begin{equation*}
    \forall t\geq \hat t,\forall M\in I_t,\quad \sum_{s=t_M}^t \ell(\hat Y_t,\tilde Y_t)\leq \sum_{s=t_M}^t \ell(\hat Y_t^M,\tilde Y_t) + 3\ln^2 t\sqrt t.
\end{equation*}
Further, we know that $f_\cdot^M$ is Bayes optimistically universal for value space $(B_\ell(\bar y,M),\ell)$. In particular, because $\Xbb\in\soul$ and $\phi_M\circ f:\Xcal\to B_\ell(\bar y,M)$, we have
\begin{equation*}
    \limsup_{T\to\infty} \frac{1}{T}\sum_{t=1}^T \ell(\hat Y^M_t,\phi_M(Y_t)) - \ell(\phi_M\circ f(X_t),\phi_M(Y_t)) \leq 0\quad (a.s.).
\end{equation*}
For simplicity, we introduce $\delta_T^M:=  \frac{1}{T}\sum_{t=1}^T \ell(\hat Y^M_t,\phi_M(Y_t)) - \ell(\phi_M\circ f(X_t),\phi_M(Y_t))$ and define $\Ecal_M$ as the event of probability one where the above inequality is satisfied, i.e., $\limsup_{T\to\infty} \delta_T^M\leq 0$. Because we always have 
$\ell(\hat Y_t,\bar y)\leq 2^{-\alpha+1} \ln t$, we can write
\begin{align*}
    \frac{1}{T}\sum_{t=1}^T \ell(\hat Y_t,Y_t) -\ell(\hat Y_t,\tilde Y_t)&= \frac{1}{T}\sum_{t=1}^T \left( \ell(\hat Y_t,Y_t) -\ell(\hat Y_t,\bar y) \right)\1_{\ell(Y_t,\bar y)\geq 2^{-\alpha+1}\ln t}\\
    &\leq \frac{1}{T}\sum_{t=1}^T \left( 2^{\alpha-1}\ell(\hat Y_t,\bar y) +2^{\alpha-1}\ell(Y_t,\bar y) \right)\1_{\ell(Y_t,\bar y)\geq 2^{-\alpha+1}\ln t}\\
    &\leq \frac{2^\alpha}{T}\sum_{t=1}^T \ell(Y_t,\bar y) \1_{\ell(Y_t,\bar y)\geq 2^{-\alpha+1}\ln t}.
\end{align*}
The proof of \cref{thm:CS_regression_unbounded}  shows that on the event $\Acal$,
\begin{equation*}
    \limsup_{T\to\infty} \frac{1}{T}\sum_{t=1}^T \ell(Y_t,\bar y) \1_{\ell(Y_t,\bar y)\geq 2^{-\alpha+1}\ln t}\leq 0,
\end{equation*}
which implies $\limsup_{T\to\infty} \frac{1}{T}\sum_{t=1}^T \ell(\hat Y_t,Y_t) -\ell(\hat Y_t,\tilde Y_t) \leq 0$. Now let $M\geq 1$. We write
\begin{align*}
    \frac{1}{T}\sum_{t=1}^T & \ell(\hat Y^M_t,\tilde Y_t) - \ell(\hat Y^M_t,\phi_M(Y_t)) \\
    &\leq \frac{1}{T}\sum_{t=1}^{t_M-1} \ell(\hat Y^M_t,\tilde Y_t) +  \frac{1}{T}\sum_{t=t_M}^T \left(\ell(\hat Y^M_t,Y_t) - \ell(\hat Y^M_t,\bar y) \right)\1_{M\leq \ell(Y_t,\bar y)<2^{-\alpha+1}\ln t}\\
    &\leq \frac{e^{2^{\alpha-1}M}2^{\alpha}M}{T} + \frac{1}{T}\sum_{t=1}^T \left(2^{\alpha-1}\ell(\hat Y_t^M,\bar y) + 2^{\alpha-1}\ell(Y_t,\bar y)\right)\1_{ \ell(Y_t,\bar y)\geq M}\\
    &\leq \frac{e^{2^{\alpha-1}M}2^{\alpha}M}{T} + \frac{2^{\alpha}}{T}\sum_{t=1}^T \ell(Y_t,\bar y)\1_{ \ell(Y_t,\bar y)\geq M}.
\end{align*}
Hence, on the event $\Acal$, we obtain \begin{equation*}
    \limsup_{T\to\infty}\frac{1}{T}\sum_{t=1}^T \ell(\hat Y^M_t,\tilde Y_t) - \ell(\hat Y^M_t,\phi_M(Y_t)) \leq 2^{\alpha} \limsup_{T\to\infty} \frac{1}{T}\sum_{t=1}^T \ell(Y_t,\bar y)\1_{ \ell(Y_t,\bar y)\geq M}.
\end{equation*}
Finally, we compute
\begin{align*}
    &\frac{1}{T}\sum_{t=1}^T  \ell(\phi_M\circ f(X_t),\phi_M(Y_t)) - \ell(f(X_t),Y_t)\\
    &\leq \frac{1}{T}\sum_{t=1}^T \left(\ell(\bar y,Y_t) - \ell(f(X_t),Y_t) \right) \1_{\ell(f(X_t),\bar y)\geq M}\1_{\ell(Y_t,\bar y)\leq M}\\
    &\quad\quad\quad+  \frac{1}{T}\sum_{t=1}^T \left(\ell( f(X_t),\bar y) - \ell(f(X_t),Y_t) \right) \1_{\ell(f(X_t),\bar y)\leq M}\1_{\ell(Y_t,\bar y)\geq M}\\
    &\leq \frac{1}{T}\sum_{t=1}^T \ell(\bar y,Y_t) \1_{\ell(Y_t,\bar y)\geq 2^{-\alpha}M} +  \frac{M}{T}\sum_{t=1}^T \1_{\ell(Y_t,\bar y)\geq M} \\
    &\quad\quad\quad\quad   + \frac{1}{T}\sum_{t=1}^T \left(\ell(\bar y,Y_t) - \ell(f(X_t),Y_t) \right) \1_{\ell(f(X_t),\bar y)\geq M}\1_{\ell(Y_t,\bar y)\leq  2^{-\alpha}M}  \\
    &\leq \frac{1}{T}\sum_{t=1}^T \ell(\bar y,Y_t) \1_{\ell(Y_t,\bar y)\geq 2^{-\alpha}M} +  \frac{1}{T}\sum_{t=1}^T \ell(Y_t,\bar y)\1_{\ell(Y_t,\bar y)\geq M} \\
    &\quad\quad\quad\quad + \frac{1}{T}\sum_{t=1}^T \left(2\ell(\bar y,Y_t) - 2^{-\alpha+1}\ell(f(X_t),\bar y) \right) \1_{\ell(f(X_t),\bar y)\geq M}\1_{\ell(Y_t,\bar y)\leq  2^{-\alpha}M} \\
    &\leq \frac{1}{T}\sum_{t=1}^T \ell(\bar y,Y_t) \1_{\ell(Y_t,\bar y)\geq 2^{-\alpha}M} + \frac{1}{T}\sum_{t=1}^T \ell(Y_t,\bar y)\1_{\ell(Y_t,\bar y)\geq M} .
\end{align*}
We now put all these estimates together. On the event $\Acal\cap\Bcal\cap\bigcap_{M=1}^\infty \Ecal_M$, for any $M\geq 1$ and $t\geq \max(\hat t,t_M)$ we can write
\begin{align*}
    &\frac{1}{T}\sum_{t=1}^T \ell(\hat Y_t,Y_t) - \ell( f(X_t),Y_t) \leq \frac{1}{T}\sum_{t=1}^T \left(\ell(\hat Y_t,Y_t) - \ell(\hat Y_t,\tilde Y_t)\right) \\
    &+ \frac{1}{T}\sum_{t=1}^T \left(\ell(\hat Y_t,\tilde Y_t) - \ell(\hat Y^M_t,\tilde Y_t)\right) + \frac{1}{T}\sum_{t=1}^T  \left(\ell(\hat Y^M_t,\tilde Y_t) - \ell(\hat Y^M_t,\phi_M(Y_t))\right) + \delta_T^M\\
    &+ \frac{1}{T}\sum_{t=1}^T  \left(\ell(\phi_M\circ f(X_t),\phi_M(Y_t)) - \ell(f(X_t),Y_t)\right)\\
    &\leq \frac{1}{T}\sum_{t=1}^T \left(\ell(\hat Y_t,Y_t) - \ell(\hat Y_t,\tilde Y_t)\right) +\frac{3\ln^2 T}{\sqrt T}  + \frac{1}{T}\sum_{t=1}^T \left(\ell(\hat Y^M_t,\tilde Y_t) - \ell(\hat Y^M_t,\phi_M(Y_t))\right)\\
    &\quad\quad\quad + \delta_T^M + \frac{1}{T}\sum_{t=1}^T  \left(\ell(\phi_M\circ f(X_t),\phi_M(Y_t)) - \ell(f(X_t),Y_t)\right).
\end{align*}
Thus, we obtain on the event $\Acal\cap\Bcal\cap\bigcap_{M=1}^\infty \Ecal_M$, for any $M\geq 1$, 
\begin{multline*}
    \limsup_{T\to\infty} \frac{1}{T}\sum_{t=1}^T \ell(\hat Y_t,Y_t) - \ell( f(X_t),Y_t) \leq \limsup_{T\to\infty}\frac{1}{T}\sum_{t=1}^T \ell(\bar y,Y_t) \1_{\ell(Y_t,\bar y)\geq 2^{-\alpha}M}\\
    +(1+2^{\alpha})\limsup_{T\to\infty} \frac{1}{T}\sum_{t=1}^T \ell(Y_t,\bar y)\1_{\ell(Y_t,\bar y)\geq M} 
\end{multline*}
On the event $\Acal$, the same arguments as in the proof of \cref{thm:CS_regression_unbounded} show that we have same guarantees for $y_0$ as for $\bar y$, i.e., for any $\epsilon>0$, there exists $\tilde M_\epsilon$ such that $\limsup_{T\to\infty}\frac{1}{T}\sum_{t=1}^T \ell(Y_t,\bar y)\1_{\ell(Y_t,\bar y)\geq \tilde M_\epsilon} \leq \epsilon$. Therefore, for any $\epsilon>0$, we can apply the above equation to $M:=\lceil 2^\alpha M_\epsilon + M_{2^{-\alpha-1}\epsilon}\rceil$ to obtain
\begin{equation*}
    \limsup_{T\to\infty} \frac{1}{T}\sum_{t=1}^T \ell(\hat Y_t,Y_t) - \ell( f(X_t),Y_t) \leq \epsilon + \frac{1+2^{\alpha}}{2^{\alpha+1}} \leq 2\epsilon.
\end{equation*}
Because this holds for all $\epsilon>0$, we can in finally get 
\begin{equation*}
    \limsup_{T\to\infty} \frac{1}{T}\sum_{t=1}^T \left(\ell(\hat Y_t,Y_t) - \ell(f(X_t),Y_t)\right) \leq 0,
\end{equation*}
on the event $\Acal\cap\Ecal\cap\bigcap_{M\geq 1}\Fcal_M$ of probability one. This ends the proof of the theorem.

\end{appendix}

\end{document}